\documentclass[sn-vancouver,Numbered]{sn-jnl}

\usepackage{graphicx}%
\usepackage{multirow}%
\usepackage{amsmath,amssymb,amsfonts}%
\usepackage{amsthm}%
\usepackage{mathrsfs}%
\usepackage[title]{appendix}%
\usepackage{xcolor}%
\usepackage{textcomp}%
\usepackage{manyfoot}%
\usepackage{booktabs}%
\usepackage{algorithm}%
\usepackage{algorithmicx}%
\usepackage{algpseudocode}%
\usepackage{listings}%

\theoremstyle{thmstyleone}%
\newtheorem{theorem}{Theorem}
\newtheorem{proposition}[theorem]{Proposition}%

\theoremstyle{thmstyletwo}%
\newtheorem{remark}{Remark}%
\newtheorem{lemma}{Lemma}
\newtheorem{corollary}{Corollary}

\theoremstyle{thmstylethree}%
\newtheorem{definition}{Definition}%

\newcommand{\vardbtilde}[1]{\tilde{\raisebox{0pt}[0.85\height]{$\tilde{#1}$}}}

\newif\ifCOM
\COMfalse

\begin{document}

\title[Non-asymptotic spectral bounds]{Non-asymptotic spectral bounds on the $\varepsilon$-entropy of kernel classes}


\author{\fnm{Rustem} \sur{Takhanov}}\email{rustem.takhanov@nu.edu.kz}

\affil{\orgname{Nazarbayev University}, \orgaddress{\street{Kabanbay batyr ave.}, \city{Astana}, \postcode{010000}, \country{Kazakhstan}}}


\abstract{
Let $K: \boldsymbol{\Omega}\times \boldsymbol{\Omega}$ be a continuous Mercer kernel defined on a compact subset of ${\mathbb R}^n$ and $\mathcal{H}_K$ be the reproducing kernel Hilbert space (RKHS) associated with $K$. Given a finite measure $\nu$ on $\boldsymbol{\Omega}$, we investigate upper and lower bounds on the $\varepsilon$-entropy of the unit ball of $\mathcal{H}_K$ in the space $L_p(\nu)$. This topic is an important direction in the modern statistical theory of kernel-based methods. 

We prove sharp upper and lower bounds for $p\in [1,+\infty]$. For $p\in [1,2]$, the upper bounds are determined solely by the eigenvalue behaviour of the corresponding integral operator $\phi\to \int_{\boldsymbol{\Omega}} K(\cdot,{\mathbf y})\phi({\mathbf y})d\nu({\mathbf y})$. In constrast, for $p>2$, the bounds additionally depend on the convergence rate of the truncated Mercer series to the kernel $K$ in the $L_p(\nu)$-norm.

We discuss a number of consequences of our bounds and show that they are substantially tighter than previous bounds for general kernels. 
Furthermore, for specific cases, such as zonal kernels and the Gaussian kernel on a box, our bounds are asymptotically tight as $\varepsilon\to +0$.
}




\maketitle

\section{Introduction and main results}
Kernel methods, such as kernel ridge regression (KRR), support vector machines (SVM), kernel density estimation, moment matching networks, etc, play a central role in modern statistics and machine learning~\cite{scholkopf2018learning,Evgeniou,Sriperumbudur,Krikamol,Steffen}. The mathematics underlying the ``kernel trick'' relies on the construction of the so-called reproducing kernel Hilbert spaces (RKHS), making the structure and properties of RKHS a key focus of theoretical research~\cite{Cucker2001OnTM}.
Recently, approaches like the Neural Network Gaussian Process (NNGP) and Neural Tangent Kernel (NTK) have emerged, reframing the study of neural network approximation and generalization as an analysis of different types of kernels for KRR~\cite{Jacot,DuSimon,pmlr-v119-huang20l}. This perspective provides promising insights into the spectral bias of neural networks~\cite{pmlr-v97-rahaman19a,ijcai2021-304}. In this context, exploring the optimization space of KRR has become an essential task.

Given a Mercer kernel $K$, it is well-known that the KRR learning process operates within a ball in the RKHS induced by $K$. The capacity of this ball is naturally measured by its covering numbers or, equivalently, by its $\varepsilon$-entropy in the space of continuous functions equipped with the supremum norm.
This measure of capacity plays a role in statistical learning theory for real-valued function classes analogous to the role of VC-dimension for $\{0,1\}$-valued function classes~\cite{HAUSSLER199278}.


\subsection{$\varepsilon$-entropy of a ball in $\mathcal{H}_{K}$}
Let $\boldsymbol{\Omega}\subseteq {\mathbb R}^n$ be a compact set and $K: \boldsymbol{\Omega}\times \boldsymbol{\Omega}\to {\mathbb R}$ be a continuous Mercer kernel. The space of continuous functions on $\boldsymbol{\Omega}$ equipped with the supremum norm is denoted by $C(\boldsymbol{\Omega})$. Let $\nu$ be a finite, nondegenerate Borel measure on $\boldsymbol{\Omega}$, and $L_p(\nu), p\geq 1$ be the completion of the space of real-valued functions $f$ on $\boldsymbol{\Omega}$ for which $|f|^p$ is integrable w.r.t. $\nu$, i.e. $\int |f|^pd\nu<\infty$. The operator ${\rm O}_{K}: L_2(\nu)\to L_2(\nu)$ is defined by ${\rm O}_{K}[\phi]({\mathbf x}) = \int_{\boldsymbol{\Omega}}K({\mathbf x}, {\mathbf y})\phi({\mathbf y})d \nu({\mathbf y})$.
By Mercer's theorem, there exists an orthonormal basis in $L_2(\nu)$ consisting of eigenvectors of ${\rm O}_{K}$.
Let $\{\lambda_i\}_{i=1}^\infty$ denote the positive eigenvalues (counted with multiplicities), with corresponding eigenvectors $\{\phi_i\}_{i=1}^\infty$, and let $\sigma_i = \sqrt{\lambda_i}$. Let us denote
\begin{equation*}
\begin{split}
&C_K = \sqrt{\max\limits_{{\mathbf x}\in \boldsymbol{\Omega}} K({\mathbf x},{\mathbf x})}\\
&C_{K,p} = \|\sqrt{ K({\mathbf x},{\mathbf x})}\|_{L_p(\nu)} = \big(\int_{\boldsymbol{\Omega}} K({\mathbf x},{\mathbf x})^{p/2} d\nu({\mathbf x})\big)^{1/p}.
\end{split}
\end{equation*}
Using Mercer's theorem, the reproducing kernel Hilbert space induced by $K$ can be characterized as follows.

\begin{proposition}[\cite{Cucker2001OnTM,cucker_zhou_2007}] \label{smale} Let $\nu$ be a Borel, nondegenerate measure on $\boldsymbol{\Omega}$. Let $\{\lambda_i\}_{i=1}^\infty$ be the set of all positive eigenvalues of ${\rm O}_K$ (counting multiplicities) with corresponding orthogonal unit eigenvectors $\{\phi_i\}_{i=1}^\infty$. Then, the RKHS of $K$, denoted by $\mathcal{H}_{K}$, can be expressed as
$${\rm O}^{1/2}_{K}[L_2(\nu)] = \{\sum_{i=1}^\infty a_i\phi_i \mid \sum_{i=1}^\infty\frac{a_i^2}{\lambda_i}<+\infty\}\subseteq C(\boldsymbol{\Omega}),$$ with the inner product
$\langle \sum_{i=1}^\infty a_i\phi_i , \sum_{i=1}^\infty b_i\phi_i \rangle_{\mathcal{H}_{K}} = \sum_{i=1}^\infty \frac{a_ib_i}{\lambda_i}$. For any $f\in \mathcal{H}_{K}$,
\begin{equation*}
\begin{split}
\|f\|_{L_\infty(\nu)}\leq C_K\|f\|_{\mathcal{H}_{K}},\\
\|f\|_{L_2(\nu)}\leq \sqrt{\lambda_1}\|f\|_{\mathcal{H}_{K}}.
\end{split}
\end{equation*}
\end{proposition}

A unit ball centered at zero in a Banach space $X$ is denoted by $B_X = \{f\in X\mid \|f\|_{X} \leq 1\}$.
Let us define notions of covering numbers and $\varepsilon$-entropy.
\begin{definition} Let $X$ be a Banach space and $A\subseteq X$. A subset $C$ of $X$ is called an external $\varepsilon$-covering of $A$ if $A\subseteq \bigcup_{c\in C}(c+\varepsilon B_X)$. If additionally, $C\subseteq A$, then the external $\varepsilon$-covering $C$ of $A$ is called an internal $\varepsilon$-covering. The minimum cardinality of $C$ over all external $\varepsilon$-coverings of $A$ is denoted by ${\mathcal N}^{\rm ext}(\varepsilon, A, X)$. Analogously, ${\mathcal N}^{\rm int}(\varepsilon, A, X)$ is defined. The quantity ${\mathcal H}(\varepsilon, A, X) = \log {\mathcal N}^{\rm ext}(\varepsilon, A, X)$ is called the $\varepsilon$-entropy of $A$. Sometimes, we also use notations ${\mathcal H}^{\rm ext}(\varepsilon, A, X) = \log {\mathcal N}^{\rm ext}(\varepsilon, A, X)$ and ${\mathcal H}^{\rm int}(\varepsilon, A, X) = \log {\mathcal N}^{\rm int}(\varepsilon, A, X)$.
\end{definition}
Internal and external covering numbers are related in the following way (see~\cite{Shalev}):
$$
{\mathcal N}^{\rm int}(2\varepsilon, A, X)\leq {\mathcal N}^{\rm ext}(\varepsilon, A, X)\leq {\mathcal N}^{\rm int}(\varepsilon, A, X)
$$

This paper is dedicated to the study of ${\mathcal H}(\varepsilon, r B_{\mathcal{H}_{K}}, L_p(\nu))$, $p\geq 1$ and ${\mathcal H}(\varepsilon, r B_{\mathcal{H}_{K}}, L_\infty(\nu)) = {\mathcal H}(\varepsilon, r B_{\mathcal{H}_{K}}, C(\boldsymbol{\Omega}))$. The norm homogeneity implies $${\mathcal H}(\varepsilon, r B_{\mathcal{H}_{K}}, L_p(\nu)) = {\mathcal H}(\frac{\varepsilon}{r}, B_{\mathcal{H}_{K}}, L_p(\nu)).$$
Therefore, all our results will be formulated for ${\mathcal H}(\varepsilon, B_{\mathcal{H}_{K}}, L_p(\nu))$. Three extreme cases, the ball $B_{\mathcal{H}_{K}}$ as a subset of $L_1(\nu)$, $L_2(\nu)$ and $C(\boldsymbol{\Omega})$, are the most interesting both from perspective of applications and due to the following inequalities for $1\leq q\leq 2\leq p$:
$$
{\mathcal H}(\varepsilon, B_{\mathcal{H}_{K}}, L_q(\nu))\leq {\mathcal H}(\varepsilon, B_{\mathcal{H}_{K}}, L_2(\nu))\leq {\mathcal H}(\varepsilon, B_{\mathcal{H}_{K}}, L_p(\nu))\leq {\mathcal H}(\varepsilon, B_{\mathcal{H}_{K}}, C(\boldsymbol{\Omega})),
$$
if $\nu$ is a probability measure.

Literature dedicated to covering numbers of $B_{\mathcal{H}_{K}}$ in $C(\boldsymbol{\Omega})$ includes  monographs~\cite{cucker_zhou_2007} and~\cite{Andreas}.
The earliest bounds on the $\varepsilon$-entropy of a ball in RKHS were presented in~\cite{Cucker2001OnTM} and further study of this subject was carried out in~\cite{Capacity}. Special kernel cases, such as translation invariant and analytical kernels, were explored in~\cite{CoveringNumber}. Covering numbers of RKHS for the Gaussian kernel were extensively analyzed in~\cite{KUHN2011489,SteinwartIngo}.

{\bf Notations.} The set $\{1,...,k\}$ is denoted by $[k]$. When $f,g:U\to\mathbb{R}_+$ are such that $f(x)\leq Cg(x)$ for some universal constant $C>0$, then we write $f\lesssim g$. Given $f:\,\mathbb{N}\to\mathbb{R}$ and $g:\,\mathbb{N}\to\mathbb{R}_+$ (or, alternatively, given two functions of a small argument $f:\,\mathbb{R}_+\to\mathbb{R}, g:\,\mathbb{R}_+\to\mathbb{R}_+$), we write $f(x)\ll g(x)$  if there exist constants $c_1, c_2\in\mathbb{R}_+$ such that for all natural numbers $x>c_2$ (correspondingly, all positive reals $x < c_2$) we have $|f(x)|\leq c_1 g(x)$. We write $f\asymp g$  if $f\ll g$ and $g\ll f$. Sometimes, we write $f=\mathcal{O}(g)$, instead of $f\ll g$.

Let $S\subseteq {\mathbb N}$ be a finite set. For any $F\subseteq \{f: {\mathbb R}^{n}\to {\mathbb R}\}$, let us denote $F[S] = F\cap {\rm span}(\{\phi_s\mid s\in S\})$. Let $\mu_{{\mathbb S}^{n-1}}$ be the rotation-invariant probability measure over an $n-1$-sphere ${\mathbb S}^{n-1} = \{{\mathbf x}\in {\mathbb R}^n\mid \|{\mathbf x}\|_2=1\}$ where $\|{\mathbf x}\|_2$ denotes the canonical Euclidean norm.

\subsection{Main results}
In this paper, we present an approach to estimating covering numbers based on an analysis of eigenvalues associated with a Mercer kernel. Earlier bounds on entropy numbers based on eigenvalues were obtained in~\cite{Williamson,MendelsonS}. Our major result, given below, is based on the dual Sudakov inequality~\cite{be8abe3dbdc3, 10.1007/BF02392835} and the bound of the $\varepsilon$-entropy of ellipsoids in Euclidean space that belongs to Dumer, Pinsker, and Prelov~\cite{Dumer}.

Given the singular values of ${\rm O}_K$, let us introduce the following functions:
\begin{equation}
\mathcal{E}(\varepsilon, \{\sigma_i\}_{i=1}^N) = \sum_{i\in [N], \sigma_i> \varepsilon}\log (\frac{\sigma_i}{\varepsilon})
\end{equation}
and
\begin{equation}
m_\varepsilon = |\{i\in {\mathbb N} \mid \sigma_i >  \varepsilon \}|
\end{equation}
where $N\in {\mathbb N}\cup \{+\infty\}$. As was shown in~\cite{Dumer}, the function $\mathcal{E}(\varepsilon, \{\sigma_i\}_{i=1}^N)$ is a pretty accurate approximation of the  $\varepsilon$-entropy of an ellipsoid in Euclidean space whose semi-axes are of length $\sigma_1, \cdots, \sigma_N$. The key intuition behind our bounds is that $\mathcal{H}(\varepsilon, B_{\mathcal{H}_K}, L_p(\nu))$ can also be bounded in terms of $\mathcal{E}(\varepsilon, \{\sigma_i\}_{i=1}^N)$ and $m_\varepsilon$. Given upper and lower bounds on singular values $\{\sigma_i\}$, both functions are straightforward to estimate. Below are typical growth rates of these functions as $\varepsilon\to+0$ (details are provided in Section~\ref{growth-rates} of Appendix).

\begin{center}
\begin{tabular}{ |p{3cm}|p{3cm}|p{3cm}|p{3cm}|  }
\hline
\multicolumn{3}{|c|}{Behaviour of key functions for small $\varepsilon$} \\
\hline
Decay rate & $\mathcal{E}(\varepsilon, \{\sigma_i\}_{i=1}^\infty)$  & $m_\varepsilon$   \\
\hline
$\sigma_i = \frac{C}{i^{\gamma}(\log (i+1))^\zeta}$ & $\zeta m_\varepsilon\log\log m_\varepsilon$ & $\frac{\gamma^{\zeta/\gamma}(\frac{C}{\varepsilon})^{1/\gamma}}{(\log (\frac{C}{\varepsilon}))^{\zeta/\gamma}}$  \\
\hline
$\sigma_i = Ce^{-\gamma i^r}$ & $\frac{\gamma r (\big(\frac{\log(C/\varepsilon)}{\gamma}\big)^{1/r}+1)^r}{r+1}$ &  $\big(\frac{\log(C/\varepsilon)}{\gamma}\big)^{1/r}$ \\
\hline
\end{tabular}
\end{center}

For an upper bound on $\mathcal{H}(\varepsilon, B_{\mathcal{H}_K}, L_p(\nu))$ for $p>2$, we need another key function, that is
\begin{equation}
C_{K,p}(N) = \|\sqrt{K({\mathbf x},{\mathbf x})-\sum_{i=1}^N \lambda_i\phi_i({\mathbf x})^2}\|_{L_p(\nu)}.
\end{equation}
Recall that according to Mercer's theorem, $\lim_{N\to +\infty} C_{K,\infty}(N) = 0$. The rate of this convergence was studied in~\cite{ARXIV2205} where it was shown that for differentiable kernels it depends on the decay rate of eigenvalues $\{\lambda_i\}$. The expression $C_{K,p}(N)$ measures how rapidly in the $L_{p/2}(\nu)$-norm the convergence $\sum_{i=1}^N \lambda_i\phi_i({\mathbf x})^2\mathop\to\limits^{N\to +\infty} K({\mathbf x},{\mathbf x})$ holds.

Our main upper bound is presented in the following theorem.
\begin{theorem}\label{Bound-main}  Let $\theta\in (0,1-\frac{1}{\sqrt{2}})$. \\
\begin{itemize}
\item For $1\leq p\leq 2$, we have
\begin{equation}\label{1p2}
\begin{split}
&\mathcal{H}^{\rm ext}(\varepsilon, B_{\mathcal{H}_K}, L_p(\nu))\leq \\
&\mathcal{E}(\nu(\boldsymbol{\Omega})^{1/2-1/p}\varepsilon, \{\sigma_i\}_{i=1}^\infty)+m_{(1-\theta)\nu(\boldsymbol{\Omega})^{1/2-1/p}\varepsilon}\log\frac{3}{\theta(2-\theta)}.
\end{split}
\end{equation}
\item For $p=+\infty$, we have
\begin{equation}\label{p-infty}
\begin{split}
&\mathcal{H}^{\rm ext}(\varepsilon, B_{\mathcal{H}_K}, L_\infty(\nu))\leq \\
&\min\limits_{\substack{\delta\in (0,\varepsilon),\\ N\in {\mathbb N}\cup \{0\}}} \big\{ N\log(\frac{C_K}{\varepsilon-\delta})+c\log(N+1)+\tilde{c}_1\big(\frac{\int_0^{C_{K,+\infty}(N)} \sqrt{\mathcal{H}(t,\tilde{\boldsymbol{\Omega}}, \mathcal{H}_{K})}dt}{\delta}\big)^2\big\}.
\end{split}
\end{equation}
where $\tilde{\boldsymbol{\Omega}} = \{K({\mathbf x}, \cdot)\mid {\mathbf x}\in \boldsymbol{\Omega}\}\subseteq \mathcal{H}_{K}$ and $c,\tilde{c}_1$ are a universal constants. \\
\item For finite $p>2$ we have the following general bound:
\begin{equation}\label{p2}
\begin{split}
&\mathcal{H}^{\rm ext}(\varepsilon, B_{\mathcal{H}_K}, L_p(\nu))\leq \\
&\min\limits_{\substack{\delta\in (0,\varepsilon),\\ N\in {\mathbb N}\cup \{0\}}}\min_{\alpha\in [0,2/p]} \Big\{\mathcal{E}(\varepsilon-\delta, \{C_K^{1-\alpha} \nu(\boldsymbol{\Omega})^{1/p-\alpha/2}N^{\alpha/2}\sigma_i^\alpha\}_{i=1}^N)+\\
&|\{i\in [N]\mid C_K^{1-\alpha} \nu(\boldsymbol{\Omega})^{1/p-\alpha/2}N^{\alpha/2}\sigma_i^\alpha > (1-\theta)(\varepsilon-\delta)\}|\log\frac{3}{\theta(2-\theta)}+\\
&\frac{\tilde{c}p C_{K,p}(N)^{2}}{\delta^2}
\Big\}.
\end{split}
\end{equation}
\end{itemize}
\end{theorem}

Setting $\delta\to \varepsilon-$ and $N=0$ in the inequalities~\eqref{p-infty} and~\eqref{p2} directly gives us the following corollary.
\begin{corollary} For $p=+\infty$, we have
\begin{equation}\label{simple-p-infty}
\begin{split}
&\mathcal{H}^{\rm ext}(\varepsilon, B_{\mathcal{H}_K}, L_\infty(\nu))\lesssim \big(\frac{\int_0^{C_{K}} \sqrt{\mathcal{H}(t,\tilde{\boldsymbol{\Omega}}, \mathcal{H}_{K})}dt}{\varepsilon}\big)^2.
\end{split}
\end{equation}
For finite $p>2$, we have
\begin{equation*}
\begin{split}
&\mathcal{H}^{\rm ext}(\varepsilon, B_{\mathcal{H}_K}, L_p(\nu))\lesssim \frac{p C_{K,p}^{2}}{\varepsilon^2}.
\end{split}
\end{equation*}
\end{corollary}
\begin{remark} The upper bound~\eqref{simple-p-infty} for $\mathcal{H}^{\rm ext}(\varepsilon, B_{\mathcal{H}_K}, L_\infty(\nu))$ behaves like $\mathcal{O}(\frac{1}{\varepsilon^2})$, provided that the integral $\int_0^{C_{K}} \sqrt{\mathcal{H}(t,\tilde{\boldsymbol{\Omega}}, \mathcal{H}_{K})}dt$ is finite. The latter constant is finite if the kernel $K$ is twice continuously differentiable, or more generally, if $K$ satisfies the property~\eqref{property-k} (one can find proofs of those facts in Remark~\ref{non-exotic} and Section~\ref{smooth-kern} of Appendix).
This significantly improves on previous worst-case bounds for general types of kernels. For example, it was shown in~\cite{cucker_zhou_2007} that for $s$ times continuously differentiable kernels $K$, we have $\mathcal{H}^{\rm ext}(\varepsilon, B_{\mathcal{H}_K}, L_\infty(\nu))=\mathcal{O}(\frac{1}{\varepsilon^{2n/s}})$ (which is larger than $\mathcal{O}(\frac{1}{\varepsilon^2})$ if $s<n$).
\end{remark}

\begin{remark}
In machine learning applications, $\nu$ is typically a probabilistic distribution supported on $\boldsymbol{\Omega}$, interpreted as the unsupervised data distribution. When $\boldsymbol{\Omega}$ is a low-dimensional manifold (or, close to one) embedded in a high-dimensional ambient space ${\mathbb R}^n$, kernel-based methods are known to be susceptible to the intrinsic dimension of $\boldsymbol{\Omega}$ --- the lower the intrinsic dimension of $\boldsymbol{\Omega}$, the higher the performance of those methods~\cite{10.3150/23-BEJ1685}.

Note that the integral $\int_0^{C_{K}} \sqrt{\mathcal{H}(t,\tilde{\boldsymbol{\Omega}}, \mathcal{H}_{K})}dt$ reflects this intrinsic dimension of $\boldsymbol{\Omega}$. Indeed, the value
$$
\overline{\rm dim}(\boldsymbol{\Omega}) = \limsup_{\varepsilon\to+0}\frac{\mathcal{H}(\varepsilon,\tilde{\boldsymbol{\Omega}}, \mathcal{H}_{K})}{\log \frac{1}{\varepsilon}}
$$
is the upper metric dimension of the set $\boldsymbol{\Omega}$ equipped with the metric $\varrho({\mathbf x},{\mathbf y}) = \|K({\mathbf x}, \cdot)-K({\mathbf y}, \cdot)\|_{\mathcal{H}_K}$ (which is one of definitions of the fractal dimension)~\cite{TikhomirovKolmogorov}. Assuming that $\mathcal{H}(\varepsilon,\tilde{\boldsymbol{\Omega}}, \mathcal{H}_{K})\leq c\overline{\rm dim}(\boldsymbol{\Omega})\log \frac{1}{\varepsilon}$ for $\varepsilon\in (0, \varepsilon_0)$ where $\varepsilon_0>0$, the main part of the integral term in~\eqref{simple-p-infty} satisfies
\begin{equation*}
\begin{split}
\int_0^{\varepsilon_0} \sqrt{\mathcal{H}(t,\tilde{\boldsymbol{\Omega}}, \mathcal{H}_{K})}dt \lesssim \sqrt{c\overline{\rm dim}(\boldsymbol{\Omega})}\varepsilon_0\sqrt{\log \frac{1}{\varepsilon_0}}.
\end{split}
\end{equation*}
This shows that the bound on $\mathcal{H}^{\rm ext}(\varepsilon, B_{\mathcal{H}_K}, L_\infty(\nu))$ is approximately linear in the intrinsic dimension $\overline{\rm dim}(\boldsymbol{\Omega})$.
\end{remark}

\begin{remark}
We show that upper bounds of Theorem~\ref{Bound-main} are asymptotically tight for $\varepsilon\to +0$ in such practically important cases as (a) special zonal kernels (Section~\ref{zonal-section} of Appendix) and (b) the Gaussian kernel on a box (Section~\ref{gaussian-kernel} of Appendix).
\end{remark}

{\bf Overview of the proof of Theorem~\ref{Bound-main}.} This theorem is proved in a sequence of steps in Sections~\ref{ellipsoidal}, \ref{sudakov-based} and~\ref{proof-upper}. As can be seen from its formulation, we distinguish three primary cases, $1\leq p\leq 2$, $p\in [2,+\infty)$ and $p=+\infty$. The core idea across all cases is to view $B_{\mathcal{H}_K}$ as an infinite-dimensional ellipsoid in $L_2(\nu)$, where the semi-axis lengths are given by the singular values $\{\sigma_i\}_{i=1}^\infty$ and the corresponding axis directions by the eigenfunctions $\{\phi_i\}_{i=1}^\infty$. This structure of $B_{\mathcal{H}_K}$ is already clear from Proposition~\ref{smale}.

Since eigenvalues are usually arranged in a (nearly) decreasing order, it is natural to decompose $B_{\mathcal{H}_K}$ as a subset of
$$
\underbrace{\{\sum_{i=1}^N \xi_i \phi_i\mid \sum_{i=1}^N \frac{\xi_i^2}{\sigma_i^2}\leq 1\}}_{\underbrace{\rm main\,\,part}_{\varepsilon-\delta\rm{-covering}}}\oplus \underbrace{\{\sum_{i=N+1}^\infty \xi_i \phi_i\mid \sum_{i=N+1}^\infty \frac{\xi_i^2}{\sigma_i^2}\leq 1\}}_{\underbrace{\rm tail\,\,part}_{\delta\rm{-covering}}},
$$
where $\oplus$ stands for the Minkowski sum of two subsets of $L_2(\nu)$ and $N$ is the number of components in the main part that we have freedom to choose. Given any $\varepsilon>0$, an upper bound on the size of an optimal $\varepsilon$-covering of the latter sum in $L_p(\nu)$ is made by:
\begin{itemize}
\item selecting a ``budget'' parameter $\delta\in (0,\varepsilon)$ and defining an $\varepsilon-\delta$-covering of the main part of $B_{\mathcal{H}_K}$'s decomposition,
\item defining a $\delta$-covering of the tail part of $B_{\mathcal{H}_K}$'s decomposition,
\item bounding the $\epsilon$-covering number of total $B_{\mathcal{H}_K}$ by the product of previous two covering numbers.
\end{itemize}
Since $N, \delta$ are parameters to be chosen, our upper bound for $p\in [2,+\infty]$ involves the minimization over $N, \delta$.

Section~\ref{ellipsoidal} addresses covering numbers of the main part of $B_{\mathcal{H}_K}$. For $p=2$, we observe that the covering number of interest coincides with the covering number of an ellipsoid in a Euclidean space. The latter subject was studied in~\cite{Dumer} where it was shown that the $\varepsilon$-entropy of an ellipsoid in ${\mathbb R}^N$ with semi-axes lengths $\{\sigma_i\}_{i=1}^N$ approximately equals $\mathcal{E}(\varepsilon, \{\sigma_i\}_{i=1}^N)$ (see Proposition~\ref{Dum}). This leads to the conclusion that for $p\in [1,2]$, $\mathcal{E}(\nu(\boldsymbol{\Omega})^{1/2-1/p}\varepsilon, \{\sigma_i\}_{i=1}^N)$ is a good estimate for the $\varepsilon$-entropy of the main part of $B_{\mathcal{H}_K}$.

The situation appears to be trickier for $p>2$. We bound the $L_p(\nu)$-norm by a suitable Euclidean norm (Lemma~\ref{inclusion2}) and reduce the problem again to the case of an ellipsoid in Euclidean space (Lemma~\ref{ellipsoid2} and Theorem~\ref{bound-for-finite}). Our analysis shows that the $\epsilon$-entropy of the main part in $L_p(\nu)$ can be bounded by the $\epsilon$-entropy of the ellipsoid in ${\mathbb R}^N$ with semi-axes lengths given by $\{C_K^{1-\alpha} \nu(\boldsymbol{\Omega})^{1/p-\alpha/2}N^{\alpha/2}\sigma_i^\alpha\}_{i=1}^N$ where $\alpha\in [0, \frac{2}{p}]$ is a free parameter.
As $p$ increases, the unit ball in $L_p(\nu)$ becomes smaller and its geometry starts to depend on the eigenfunctions $\{\phi_i\}$. That is why our bound for the $\epsilon$-entropy of the main part of $B_{\mathcal{H}_K}$ for $p>2$ is less precise than for $p=2$.

Section~\ref{sudakov-based} focuses on the tail part of $B_{\mathcal{H}_K}$. The key tool behind our bounds is a general fact about covering numbers in Banach spaces, the so-called dual Sudakov inequality (see Proposition~\ref{panos}). Based on that inequality, we first show that the $\varepsilon$-entropy of the tail part of $B_{\mathcal{H}_K}$ behaves like $\lesssim \frac{C_p}{\varepsilon^2}$, where the constant $C_p$ is basically the $L_p(\nu)$-norm of the Gaussian Process on $\boldsymbol{\Omega}$ whose covariance function equals $K({\mathbf x}, {\mathbf y})-\sum_{i=1}^N \sigma_i^2\phi_i({\mathbf x})\phi_i({\mathbf y})$, i.e. the residual kernel after truncating the Mercer expansion. This outlook allows us to prove Theorem~\ref{from-sudakov}, bounding the $\varepsilon$-entropy in terms of $C_{K,p}(N)$. It is the point at which the difference between finite $p$ and $p=+\infty$ matters, because for $p=+\infty$, bounding the supremum norm of the Gaussian Process requires Dudley's integral bound, while for finite $p$, simpler methods suffice.

Finally, in Section~\ref{proof-upper} we assemble results of the previous two sections and formulate the major upper bound as a minimum over parameters $N,\delta,\alpha$ of the sum of the $\varepsilon-\delta$-entropy of the main part of $B_{\mathcal{H}_K}$ and the $\delta$-entropy of the tail part of $B_{\mathcal{H}_K}$. Here the difference between cases of $p\in [1,2]$ and $p\in (2,+\infty]$ becomes clear: for the first case, the final bound is obtained by sending $N$ to infinity (i.e. the tail part does not contribute at all), whereas for $p>2$ we have to find an optimal $N$ that balances contributions from the main and tail parts.

{\bf Bounding the decay rate of $C_{K,p}(N), p>2$}. The bound in Theorem~\ref{Bound-main} for the case $1 \leq p \leq 2$ is expressed solely in terms of the behavior of the eigenvalues $\{\lambda_i\}$ (or equivalently, the singular values $\{\sigma_i\}$). However, for the case $p > 2$, the bound depends on the decay rate of $C_{K,p}(N)$, which in turn is influenced by the behavior of the eigenfunctions $\{\phi_i\} $. This distinction is expected, since for $p = +\infty$, the $L_\infty(\nu)$-norm depends on the geometry of the domain $\boldsymbol{\Omega}$ rather than the measure $ \nu$ (which determines the behaviour of $\{\sigma_i\}$). Therefore, when applying these bounds for $p>2$ to specific kernels, additional assumptions about the eigenfunctions are unavoidable. A common assumption in the machine learning context is the uniform boundedness of eigenfunctions~\cite{Zhou,Regularizationkernel,Steinwart2012}, i.e.,
\[
\sup_{i \in \mathbb{N}} \sup_{{\mathbf x} \in \boldsymbol{\Omega}} |\phi_i({\mathbf x})| < +\infty.
\]
To generalize this assumption, we introduce the following definition.
\begin{definition}\label{kushpel-inspired} Let a Mercer kernel $K: \boldsymbol{\Omega}\times \boldsymbol{\Omega}\to {\mathbb R}$ and a nondegenerate Borel measure $\nu$ on $\boldsymbol{\Omega}$ be such that there exist $\varkappa>0$, a sequence of singular values  of ${\rm O}_K$, $\boldsymbol{\sigma} = \{\sigma_i\}_{i\in {\mathbb N}}$, with a corresponding orthonormal system of eigenfunctions  $\boldsymbol{\phi} = \{\phi_i\}_{i\in {\mathbb N}}\subseteq L_2(\nu)$, such that
\begin{equation*}
\begin{split}
\sum_{j=1}^{i} |\phi_j({\mathbf x})|^2\leq \varkappa i,
\end{split}
\end{equation*}
for any $i\in {\mathbb N}$, and the null space of ${\rm O}_K$ satisfies ${\rm ker\, O}_K = {\rm span}(\{\phi_i\}_{i\in {\mathbb N}})^\perp$. We will denote this fact by $(\boldsymbol{\Omega}, \nu, K, \boldsymbol{\sigma}, \boldsymbol{\phi}, \varkappa)\in \mathcal{K}$.
\end{definition}
Under the assumption $(\boldsymbol{\Omega}, \nu, K, \boldsymbol{\sigma}, \boldsymbol{\phi}, \varkappa) \in \mathcal{K}$, and assuming certain regularity conditions on $K$, we derive the following corollary of Theorem~\ref{Bound-main} (proof in Section~\ref{residual}).

\begin{corollary}\label{corol} Let $(\boldsymbol{\Omega}, \nu, K, \boldsymbol{\sigma}, \boldsymbol{\phi}, \varkappa)\in \mathcal{K}$ and $\lim_{i\to +\infty}i\sigma_i^2 =0$. Suppose there are constants $C>0$, $\alpha >0$ such that
$$\sqrt{K({\mathbf x}, {\mathbf x})+K({\mathbf y}, {\mathbf y})-2K({\mathbf x}, {\mathbf y})}\leq C\|{\mathbf x}- {\mathbf y}\|^\alpha,$$
for any ${\mathbf x}, {\mathbf y}\in \boldsymbol{\Omega}$. Let $C_{n,\alpha, C, R} = R^\alpha Ce^{\frac{c \alpha\log(n+1)}{n}}$, where $R=\max_{{\mathbf x}\in \boldsymbol{\Omega}}\|{\mathbf x}\|_2$ and
$c$ is a constant from Roger's bound~\eqref{rogers}.
Then, there is a universal constant $\vardbtilde{c}$, such that
\begin{equation}\label{kushpel-class-bound}
\begin{split}
&\mathcal{H}^{\rm ext}(2\varepsilon, B_{\mathcal{H}_K}, C(\boldsymbol{\Omega}))\leq m_\varepsilon\log (\frac{C_K}{\varepsilon})+c\log(m_\varepsilon+1)+\\
&\frac{\vardbtilde{c} n}{ \alpha }\cdot \frac{\varkappa^2(\sum_{i=m_\varepsilon+1}^{\infty}i|\sigma_i^2-\sigma_{i+1}^2|)^2\log(\frac{C_{n,\alpha, C, R} }{\varkappa\sum_{i=m_\varepsilon+1}^{\infty}i|\sigma_i^2-\sigma_{i+1}^2|})}{\varepsilon^2}.
\end{split}
\end{equation}
provided that $\varkappa\sum_{i=m_\varepsilon+1}^{\infty}i|\sigma_i^2-\sigma_{i+1}^2|\leq \frac{1}{2}C_{n,\alpha, C, R}$.
\end{corollary}

{\bf Lower bounds on the $\varepsilon$-entropy}. The function $\mathcal{E}(\varepsilon, \{\sigma_i\}_{i=1}^N)$ proves even more useful in lower bounds than in upper bounds. Let us now formulate them.
First, let us introduce the constant (which is the $p$-th moment of the standard Gaussian random variable) $$B_{p}\triangleq\frac{2^{p/2}\Gamma(\frac{p+1}{2})}{\sqrt{\pi}}.$$
A major lower bound is given in the following Theorem.
\begin{theorem}\label{lower-kushpel} Let $p'$ be such that $\frac{1}{p}+\frac{1}{p'}=1$. We have
\begin{equation}\label{generic-lower-bound}
\begin{split}
&\mathcal{H}^{\rm int}(\varepsilon, B_{\mathcal{H}_K}, L_p(\nu))\geq \sup_{N\in {\mathbb N}}\sum_{i=1}^N \log(\frac{\sigma_i}{C_{K,p'}\varepsilon})+N\log \langle \{\sigma_i\}_{i\in [N]}\rangle_{-2} , {\rm if\,\,}1\leq p\leq 2\\
&\mathcal{H}^{\rm ext}(\varepsilon, B_{\mathcal{H}_K}, L_p(\nu))\geq \mathcal{E}(\nu(\boldsymbol{\Omega})^{\frac{1}{2}-\frac{1}{p}} B_{p'}^{1/p'}\varepsilon, \{\sigma_i\}_{i=1}^\infty)-0.7, {\rm if\,\,}p\geq 2.\\
\raisetag{2\baselineskip}
\end{split}
\end{equation}
where  $\langle \{\sigma_i\}_{i\in [N]}\rangle_{-2} =(\frac{1}{N}\sum_{i\in [N]}\sigma^{-2}_i )^{-\frac{1}{2}}$ is a generalized mean of $\{\sigma_i\}_{i\in [N]}$.

If $(\boldsymbol{\Omega}, \nu, K, \boldsymbol{\sigma}, \boldsymbol{\phi}, \varkappa)\in \mathcal{K}$, we have
\begin{equation}
\begin{split}
&\mathcal{H}^{\rm int}(\varepsilon, B_{\mathcal{H}_K}, L_p(\nu))\geq \mathcal{E}(\nu(\boldsymbol{\Omega})^{1/p'} B_{p'}^{1/p'}\varkappa^{\frac{1}{2}}\varepsilon, \{\sigma_i\}_{i=1}^\infty)-0.7 , {\rm if\,\,}1\leq p\leq 2\\
&\mathcal{H}^{\rm ext}(\varepsilon, B_{\mathcal{H}_K}, L_p(\nu))\geq \mathcal{E}(\nu(\boldsymbol{\Omega})^{1/p'} B_{p'}^{1/p'}\varkappa^{\frac{1}{2}}\varepsilon, \{\sigma_i\}_{i=1}^\infty)-0.7, {\rm if\,\,}p\geq 2.\\
\end{split}
\end{equation}
\end{theorem}
The latter lower bounds are proved in Section~\ref{lower-section}. The proof is based on choosing the number $N\in {\mathbb N}$ and considering the coordinate mapping $\Psi$ between $B_{\mathcal{H}_K}$ and ${\mathbb R}^N$, i.e. the mapping $f \mathop\to\limits^{\Psi} [\langle f, \phi_i\rangle_{L_2(\nu)}]_{i=1}^N$. Then, a standard volume-based argument gives us that the $\varepsilon$-covering number of $B_{\mathcal{H}_K}$ in $L_p(\nu)$ is at least the ratio of the volume of $\Psi(B_{\mathcal{H}_K})$ in ${\mathbb R}^N$ to the volume of the $\varepsilon$-ball of the Banach space $({\mathbb R}^N, \|\cdot\|_p^\ast)$, where $\|[\xi_i]_{i=1}^N\|_p^\ast\triangleq \|\sum_{i=1}^N \xi_i \phi_i\|_{L_p(\nu)}$. Using a standard technique, based on Urysohn's inequality, we bound the latter ratio from below and obtain the desired lower bounds.

\begin{remark} Given eigenvalues, our lower bounds, except for the case $1\leq p\leq 2$, can be calculated straightforwardly using the function $\mathcal{E}$. For the case $1\leq p\leq 2$, the generic lower bound~\eqref{generic-lower-bound} is weaker than $\mathcal{E}(\varepsilon', \{\sigma_i\}_{i=1}^\infty)$ where $\varepsilon' = \mathcal{O}(\varepsilon)$. Let us illustrate this for $\sigma_i = i^{-\gamma}$, $\gamma>0$, $\varepsilon'=C_{K,p'}\varepsilon$ and $\varepsilon\to 0$. We have
\begin{equation*}
\begin{split}
\sum_{i=1}^N \log(\frac{\sigma_i}{\varepsilon'}) \asymp \int_1^N \log(\frac{x^{-\gamma}}{\varepsilon'})dx \asymp -\gamma N(\log N-1) -N\log \varepsilon',
\end{split}
\end{equation*}
which gives $\mathcal{E}(\varepsilon', \{\sigma_i\}_{i=1}^\infty)\asymp \sum_{i=1}^{(\varepsilon')^{-1/\gamma}} \log(\frac{\sigma_i}{\varepsilon'}) \asymp \gamma (\varepsilon')^{-1/\gamma}$.

Terms of the lower bound~\eqref{generic-lower-bound} behave as
\begin{equation*}
\begin{split}
\sum_{i=1}^N i^{2\gamma} \asymp \frac{N^{2\gamma+1}}{2\gamma+1}\Rightarrow  \langle \{\sigma_i\}_{i\in [N]}\rangle_{-2} \asymp (\frac{N^{2\gamma}}{2\gamma+1})^{-\frac{1}{2}},\\
N\log \langle \{\sigma_i\}_{i\in [N]}\rangle_{-2} \asymp -\gamma N\log N+\frac{N}{2}\log(2\gamma+1).
\end{split}
\end{equation*}
We obtain
\begin{equation*}
\begin{split}
&\sum_{i=1}^N \log(\frac{\sigma_i}{\varepsilon'}) + N\log \langle \{\sigma_i\}_{i\in [N]}\rangle_{-2} \asymp\\
&-2\gamma N\log N+N\log(\frac{1}{\varepsilon'})+\frac{N}{2}\log(2\gamma+1)+\gamma N,
\end{split}
\end{equation*}
and after setting $N \asymp e^{-1}(\varepsilon')^{-1/2\gamma}$, we have
\begin{equation*}
\begin{split}
\mathcal{H}^{\rm int}(\varepsilon, B_{\mathcal{H}_K}, L_p(\nu))\gg (\varepsilon')^{-1/2\gamma}.
\end{split}
\end{equation*}
In other words, the lower bound~\eqref{generic-lower-bound} for $1\leq p
\leq 2$ behaves approximately as the square root of $\mathcal{E}(\varepsilon', \{\sigma_i\}_{i=1}^\infty)$.
\end{remark}

As the next section demonstrates, the ellipsoidal structure of $B_{\mathcal{H}_{K}}$ allows to derive a simple upper bound on the $\varepsilon$-entropy of the main part of $B_{\mathcal{H}_{K}}$ in terms of singular values $\{\sigma_i\}$.

\section{Upper bounds on the $\varepsilon$-entropy of the main part of $B_{{\mathcal H}_K}$}\label{ellipsoidal}

Recall that $\phi_i$ is the $i$th eigenvector of ${\rm O}_K$ and ${\rm O}_K\phi_i = \lambda_i\phi_i = \sigma_i^2\phi_i$. Let us denote the $\varepsilon$-entropy of an ellipsoid $\{{\mathbf x}\in {\mathbb R}^N \mid \sum_{i=1}^N\frac{x^2_i}{\sigma_i^2}\leq 1\}$ in the Euclidean space ${\mathbb R}^N$ by
$\aleph(\varepsilon,\{\sigma_i\}_{i=1}^N)$. 
Also, we denote $\aleph(\varepsilon, \{a_i=1\}_{i=1}^N)$  simply by $\aleph(\varepsilon, N)$.
Roger's bound~\cite{rogers_1963} for the function $\aleph(\varepsilon, N)$ is a classical result in information theory:
\begin{equation}\label{rogers}
\begin{split}
\aleph(\varepsilon, N)\leq N\log (\frac{1}{\varepsilon})+c\log (N+1),
\end{split}
\end{equation}
where $c$ is a universal constant. This bound was refined in a series of papers~\cite{VergerGaugry2005,Dumer2007}. The behaviour of $\aleph(\varepsilon,\{\sigma_i\}_{i=1}^N)$ is also well-known, under very general conditions on $\{\sigma_i\}_{i=1}^N$~\cite{DumerIlya}.

For any $f\in L_2(\nu)$, let us introduce the notation
\begin{equation*}
\begin{split}
R^\phi_N[f]({\mathbf x}) = \sum_{i=1}^N \langle f, \phi_i\rangle_{L_2(\nu)} \phi_i({\mathbf x}).
\end{split}
\end{equation*}
The set $B_{\mathcal{H}_K}[N] = \{R^\phi_N[f]\mid f\in B_{\mathcal{H}_K}\}$ is simply the projection of $B_{\mathcal{H}_K}$ onto the span of $\phi_1, \cdots, \phi_N$. Recall that $C_K = \max_{{\mathbf x}\in \boldsymbol{\Omega}} \sqrt{K({\mathbf x},{\mathbf x})}$.
The main result of this section is the following theorem.

\begin{theorem}\label{bound-for-finite} For $p\geq 1$ and $N\in {\mathbb N}$, the following inequalities hold:
\begin{equation*}
\begin{split}
&\mathcal{H}^{\rm ext}(\varepsilon, B_{\mathcal{H}_K}[N], L_p(\nu))\leq \\ 
&\begin{cases}
			\aleph( \nu(\boldsymbol{\Omega})^{1/2-1/p}\varepsilon, \{\sigma_i\}_{i=1}^N), & \text{if $1\leq p\leq 2$}\\
			\min_{\alpha\in [0,2/p]}\aleph(\frac{\varepsilon}{C_K^{1-\alpha} \nu(\boldsymbol{\Omega})^{1/p-\alpha/2}N^{\alpha/2}},\{\sigma_i^{\alpha}\}_{i=1}^N)& \text{if $2<p<+\infty$}\\
			\aleph(\frac{\varepsilon}{C_K}, N)& \text{if $p=+\infty$}\\
		 \end{cases},
\end{split}
\end{equation*}
and for $p=2$ we have the equality $\mathcal{H}^{\rm ext}(\varepsilon, B_{\mathcal{H}_K}[N], L_2(\nu)) = \aleph( \varepsilon, \{\sigma_i\}_{i=1}^N)$. 
\end{theorem}
The following part of this section is dedicated to its proof. 
\begin{lemma}\label{inclusion} For any $f\in \mathcal{H}_K$, we have $f\in B_{\mathcal{H}_K}$ if and only if the vector $\boldsymbol{\xi} = \big[ \int_{\boldsymbol{\Omega}}f({\mathbf y})\phi_i({\mathbf y})d\nu ({\mathbf y}) \big]_{i=1}^\infty\in {\mathbb R}^{\mathbb N}$ satisfies
$$
\sum_{i=1}^\infty\frac{\xi^2_i}{\sigma_i^2}\leq 1.
$$
\end{lemma}
\begin{proof} Let $\{\phi^0_i\}_{i=1}^\infty$ be an orthonormal basis in ${\rm Ker\,} {\rm O}_K = \{f\in L_2(\nu)\mid {\rm O}_K[f] = 0\}$. 
Since $\{\phi_i\}_{i=1}^\infty\cup \{\phi^0_i\}_{i=1}^\infty$ is an orthonormal basis in $L_2(\nu)$, any $f\in \mathcal{H}_K\subseteq L_2(\nu)$ can be represented as $f = \sum_{i=1}^\infty \xi_i \phi_i$ where $\xi_i = \langle f, \phi_i\rangle_{L_2(\nu)}$ and the convergence is w.r.t. $L_2(\nu)$. 
By Proposition~\ref{smale}, we have 
\begin{equation*}
\begin{split}
\|f\|^2_{\mathcal{H}_K} = \sum_{i=1}^\infty\frac{\xi_i^2}{\sigma_i^2}.
\end{split}
\end{equation*}
Therefore, $f\in B_{\mathcal{H}_K}$ if and only if
\begin{equation*}
\begin{split}
\sum_{i=1}^\infty\frac{\xi_i^2}{\sigma_i^2}\leq 1.
\end{split}
\end{equation*}
Lemma proved.
\end{proof}

Thus, $B_{\mathcal{H}_K}[S]$ equals $\{\sum_{i\in S} \sigma_i \xi_i \phi_{i}({\mathbf x}) \mid \sum_{i\in S}  \xi_{i}^2\leq 1\}$. By construction, $B_{\mathcal{H}_K}[S]\subseteq B_{\mathcal{H}_K}$. Also, it is easy to see that $B_{\mathcal{H}_K}[[N]]=B_{\mathcal{H}_K}[N]$.

The case of $p=2$ is the simplest one, due to the fact that covering numbers of $B_{\mathcal{H}_K}[N]$ in $L_2(\nu)$ are exactly covering numbers of ellipsoids in Euclidean spaces.

\begin{lemma}\label{ellipsoid} $\mathcal{H}^{\rm ext}(\varepsilon, B_{\mathcal{H}_K}[N], L_2(\nu))= \aleph( \varepsilon, \{\sigma_i\}_{i=1}^N)$.
\end{lemma}
\begin{proof} From Lemma~\ref{inclusion}  we obtain $B_{\mathcal{H}_K}[N]=\{\sum_{i=1}^N\xi_i\phi_i\mid \sum_{i=1}^N\frac{\xi^2_i}{\sigma_i^2} \leq 1\}$. From $\|\sum_{i=1}^N\xi_i\phi_i-\sum_{i=1}^N\xi'_i\phi_i\|_{L_2(\nu)} = \|\boldsymbol{\xi}-\boldsymbol{\xi}'\|_2$ we conclude that any $\varepsilon$-covering of the ellipsoid $\{\boldsymbol{\xi}\mid \sum_{i=1}^N\frac{\xi^2_i}{\sigma_i^2} \leq 1\}$ in ${\mathbb R}^N$ corresponds to an $\varepsilon$-covering of $B_{\mathcal{H}_K}[N]$ in $L_2(\nu)$ and vice versa. Lemma proved.
\end{proof}

\begin{lemma}\label{inclusion2} For $\boldsymbol{\xi}, \boldsymbol{\xi}'\in {\mathbb R}^N$ and $p>2$ we have
\begin{equation*}
\begin{split}
&\|\sum_{i=1}^N \sigma_i\xi_i\phi_i({\mathbf x})-\sum_{i=1}^N \sigma_i\xi'_i\phi_i({\mathbf x})\|_{L_p(\nu)}\leq \\
&\min_{\alpha\in [0,2/p]}C_K^{1-\alpha} \nu(\boldsymbol{\Omega})^{1/p-\alpha/2}N^{\alpha/2}\sqrt{\sum_{i=1}^N \sigma_i^{2\alpha}(\xi_i-\xi'_i)^2}.
\end{split}
\end{equation*}
\end{lemma}
\begin{proof} Let $\alpha, \beta\geq 0$ be such that $\alpha+\beta=1$. Using a generalized form of H{\"o}lder's inequality, we have
\begin{equation*}
\begin{split}
&|\sum_{i=1}^N \sigma_i\xi_i\phi_i({\mathbf x})-\sum_{i=1}^N \sigma_i\xi'_i\phi_i({\mathbf x})| \leq \sum_{i=1}^N \sigma_i|\xi_i-\xi'_i|\cdot |\phi_i({\mathbf x})| = \\
&\sum_{i=1}^N \big(\sigma_i^{2\alpha}|\xi_i-\xi'_i|^2\big)^{1/2}\big(\sigma^2_i|\phi_i({\mathbf x})|^2\big)^{\beta/2}(|\phi_i({\mathbf x})|^2\big)^{\alpha/2} \leq \\ 
&\big(\sum_{i=1}^N\sigma_i^{2\alpha}|\xi_i-\xi'_i|^2\big)^{1/2}\big(\sum_{i=1}^N\sigma^2_i|\phi_i({\mathbf x})|^2\big)^{\beta/2}(\sum_{i=1}^N|\phi_i({\mathbf x})|^2\big)^{\alpha/2}.
\end{split}
\end{equation*}
Using $\sum_{i=1}^N\sigma^2_i|\phi_i({\mathbf x})|^2\leq K({\mathbf x},{\mathbf x})\leq C_K^2$, we conclude that the latter is bounded by
\begin{equation*}
\begin{split}
\big(\sum_{i=1}^N\sigma_i^{2\alpha}|\xi_i-\xi'_i|^2\big)^{1/2} C_K^{\beta} (\sum_{i=1}^N|\phi_i({\mathbf x})|^2\big)^{\alpha/2}.
\end{split}
\end{equation*}
Therefore,
\begin{equation*}
\begin{split}
|\sum_{i=1}^N \sigma_i\xi_i\phi_i({\mathbf x})-\sum_{i=1}^N \sigma_i\xi'_i\phi_i({\mathbf x})|^p  \leq 
C_K^{p\beta}\big(\sum_{i=1}^N\sigma_i^{2\alpha}|\xi_i-\xi'_i|^2\big)^{p/2}  (\sum_{i=1}^N|\phi_i({\mathbf x})|^2\big)^{p\alpha/2}.
\end{split}
\end{equation*}
For $\alpha\in [0,\frac{2}{p}]$ the function $x^{p\alpha/2}$ is concave, which by Jensen's inequality implies  ${\mathbb E}[X^{p\alpha/2}]\leq {\mathbb E}[X]^{p\alpha/2}$ for any nonnegative random variable $X$. This leads to $\frac{1}{\nu(\boldsymbol{\Omega})}\int_{\boldsymbol{\Omega}}(\sum_{i=1}^N|\phi_i({\mathbf x})|^2\big)^{p\alpha/2}d\nu({\mathbf x})\leq (\frac{N}{\nu(\boldsymbol{\Omega})})^{p\alpha/2}$. After integrating w.r.t. $\nu$ we obtain
\begin{equation*}
\begin{split}
\|\sum_{i=1}^N \sigma_i\xi_i\phi_i({\mathbf x})-\sum_{i=1}^N \sigma_i\xi'_i\phi_i({\mathbf x})\|_{L_p(\nu)}^p  \leq 
C_K^{p\beta} \big(\sum_{i=1}^N\sigma_i^{2\alpha}|\xi_i-\xi'_i|^2\big)^{p/2}  \nu(\boldsymbol{\Omega})^{1-p\alpha/2}N^{p\alpha/2}.
\end{split}
\end{equation*}
For the latter the statement of Lemma is straightforward.
\end{proof}

\begin{remark}\label{inclusion2-c} Setting $p=+\infty$ and $\alpha=0$ gives
\begin{equation*}
\begin{split}
\|\sum_{i=1}^N \sigma_i\xi_i\phi_i({\mathbf x})-\sum_{i=1}^N \sigma_i\xi'_i\phi_i({\mathbf x})\|_{C(\boldsymbol{\Omega})}\leq 
C_K \sqrt{\sum_{i=1}^N (\xi_i-\xi'_i)^2}.
\end{split}
\end{equation*}
\end{remark}

\begin{lemma}\label{ellipsoid2} For $p>2$ we have 
\begin{equation*}
\begin{split}
\mathcal{H}^{\rm ext}(\varepsilon, B_{\mathcal{H}_K}[N], L_p(\nu))\leq \min_{\alpha\in [0,2/p]}\aleph(\frac{\varepsilon}{C_K^{1-\alpha} \nu(\boldsymbol{\Omega})^{1/p-\alpha/2}N^{\alpha/2}},\{\sigma_i^{\alpha}\}_{i=1}^N).
\end{split}
\end{equation*}
\end{lemma}
\begin{proof} Let us denote by ${\mathcal L}(N, \lambda)$ the set ${\mathbb R}^N$ equipped with the inner product $\langle {\mathbf x}, {\mathbf y} \rangle_{{\mathcal L}(N, \lambda)} = \sum_{i=1}^N \sigma_i^{2\alpha} x_i y_i$.

An affine transformation $\xi_i'=\sigma_i^\alpha \xi_i$ turns a ball in ${\mathcal L}(N, \lambda)$ into a ball in the canonical Euclidean space ${\mathbb R}^N$, i.e. an ellipsoid $\{\boldsymbol{\xi}\mid \sum_{i=1}^N \sigma_i^{2\alpha} \xi^2_i\leq 1\}$ into $\{\boldsymbol{\xi}'\mid \sum_{i=1}^N (\xi'_i)^2\leq 1\}$. Also, it turns $\{\boldsymbol{\xi}\mid \sum_{i=1}^N \xi^2_i\leq 1\}$ into $\{\boldsymbol{\xi}'\mid \sum_{i=1}^N \frac{(\xi'_i)^2}{\sigma_i^{2\alpha}}\leq 1\}$.
From the latter we conclude that the set $\{\boldsymbol{\xi}\mid \sum_{i=1}^N \xi^2_i\leq 1\}$ can be covered by $M = e^{\aleph(\varepsilon,\{\sigma_i^{\alpha}\}_{i=1}^N)}$ $\varepsilon$-balls in  ${\mathcal L}(N, \lambda)$, i.e. there are ${\mathbf x_i}\in {\mathbb R}^N$, $i\in [M]$ such that
$$
\{\boldsymbol{\xi}\mid \sum_{i=1}^N \xi^2_i \leq 1\} \subseteq \bigcup_{i=1}^M ({\mathbf x_i}+{\rm Ell}_\varepsilon),
$$
where ${\rm Ell}_\varepsilon = \{\boldsymbol{\xi}\mid \sum_{i=1}^N \sigma_i^{2\alpha} \xi^2_i\leq \varepsilon\}$. From this inclusion and Lemma~\ref{inclusion} we obtain
\begin{equation*}
\begin{split}
&B_{\mathcal{H}_K}[N]\subseteq \{ \sum_{i=1}^N\sigma_i  \xi_i \phi_i\mid \sum_{i=1}^N \xi^2_i\leq 1\} \subseteq 
 \{ \sum_{i=1}^N \sigma_i  \xi_i \phi_i\mid \boldsymbol{\xi}\in \bigcup_{i=1}^M ({\mathbf x_i}+{\rm Ell}_\varepsilon)\} = \\ 
&\bigcup_{i=1}^M \{ \sum_{i=1}^N\sigma_i \xi_i \phi_i\mid \boldsymbol{\xi}\in ({\mathbf x_i}+{\rm Ell}_\varepsilon)\}.
\end{split}
\end{equation*}
Lemma~\ref{inclusion2} gives us that for any $\alpha\in [0,\frac{2}{p}]$ we have
\begin{equation*}
\begin{split}
&C_K^{1-\alpha} \nu(\boldsymbol{\Omega})^{1/p-\alpha/2}N^{\alpha/2}\varepsilon B_{L_p(\nu)} = \\
&\{f\in L_p(\nu)\mid \|f\|_{L_p(\nu)}\leq C_K^{1-\alpha} \nu(\boldsymbol{\Omega})^{1/p-\alpha/2}N^{\alpha/2}\varepsilon\} \supseteq \\
&\{\sum_{i=1}^N\sigma_i\xi_i \phi_i\mid \|\boldsymbol{\xi}\|_{{\mathcal L}(N, \lambda)}\leq \varepsilon\} = 
\{ \sum_{i=1}^N\sigma_i\xi_i \phi_i\mid \boldsymbol{\xi}\in {\rm Ell}_{\varepsilon}\}.
\end{split}
\end{equation*}
Thus, $\{ \sum_{i=1}^N\sigma_i\xi_i \phi_i\mid \boldsymbol{\xi}\in {\rm Ell}_{\varepsilon}\}\subseteq C_K^{1-\alpha} \nu(\boldsymbol{\Omega})^{1/p-\alpha/2}N^{\alpha/2}\varepsilon B_{L_p(\nu)}$ and we finally obtain:
$$
B_{\mathcal{H}_K}[N]\subseteq \bigcup_{i=1}^M (f_{{\mathbf x}_i}+C_K^{1-\alpha} \nu(\boldsymbol{\Omega})^{1/p-\alpha/2}N^{\alpha/2}\varepsilon B_{L_p(\nu)} ),
$$
where $f_{{\mathbf x}} = \sum_{i=1}^N \sigma_i x_i \phi_i$. Thus, we covered $B_{\mathcal{H}_K}[N]$ by $M$ balls of radius $C_K^{1-\alpha} \nu(\boldsymbol{\Omega})^{1/p-\alpha/2}N^{\alpha/2}\varepsilon$ in  $L_p(\nu)$. Therefore,
$$
\mathcal{H}^{\rm ext}(C_K^{1-\alpha} \nu(\boldsymbol{\Omega})^{1/p-\alpha/2}N^{\alpha/2}\varepsilon, B_{\mathcal{H}_K}[N], L_p(\nu))\leq \aleph(\varepsilon,\{\sigma_i^{\alpha}\}_{i=1}^N).
$$
Now the statement of lemma is straightforward.
\end{proof}

\begin{proof}[Proof of Theorem~\ref{bound-for-finite}]
For $1\leq p\leq 2$ we have the inequality $\|f\|_{L_p(\nu)}\leq \nu(\boldsymbol{\Omega})^{1/p-1/2}\|f\|_{L_2(\nu)}$. Therefore, any $\varepsilon$-net in $L_2(\nu)$ is also an $\nu(\boldsymbol{\Omega})^{1/p-1/2}\varepsilon$-net in $L_p(\nu)$. From this, using Lemma~\ref{ellipsoid}, we directly obtain
\begin{equation*}
\begin{split}
&\mathcal{H}^{\rm ext}(\varepsilon, B_{\mathcal{H}_K}[N], L_p(\nu))\leq \mathcal{H}^{\rm ext}(\nu(\boldsymbol{\Omega})^{1/2-1/p}\varepsilon, B_{\mathcal{H}_K}[N], L_2(\nu))\leq \\ 
&\aleph( \nu(\boldsymbol{\Omega})^{1/2-1/p}\varepsilon, \{\sigma_i\}_{i=1}^N).
\end{split}
\end{equation*}
The inequality for $p>2$ was already proved as Lemma~\ref{ellipsoid2}. When in this inequality we consider $p\to+\infty$ and $\alpha=0$, we obtain the inequality for $p=+\infty$.
\end{proof}

\section{Upper bounds on the $\varepsilon$-entropy of the tail part of $B_{{\mathcal H}_K}$}\label{sudakov-based}
The following result is our main tool in bounding the $\varepsilon$-entropy of the ``tail'' part of $B_{\mathcal{H}_K}$, i.e. of $B_{\mathcal{H}_K}$ projected onto the span of $\phi_{N+1}, \phi_{N+2}, \cdots$. Its proof can be found in~\cite{be8abe3dbdc3, 10.1007/BF02392835}.
\begin{proposition}[Dual Sudakov inequality]\label{panos} For any Banach space $X = ({\mathbb R}^N, \|\cdot\|_X)$, we have
\begin{equation*}
\begin{split}
\mathcal{H}(\varepsilon, B_2^N, X)\leq c_{\rm su} N\big(
\frac{\int_{{\mathbb S}^{N-1}}\|{\mathbf x}\|_Xd\mu_{{\mathbb S}^{N-1}}({\mathbf x})}{\varepsilon}\big)^2,
\end{split}
\end{equation*}
where $c_{\rm su}$ is a universal constant.
\end{proposition}

For $M,N\in {\mathbb N}\cup \{0\}$ where $M<N$, let us denote
\begin{equation}
C_{K,p}(M,N) \triangleq \|\sqrt{\sum_{i=M+1}^{N} \lambda_i\phi_i({\mathbf x})^2}\|_{L_p(\nu)}.
\end{equation}

\begin{theorem}\label{from-sudakov} Let $S = [M+1,N]\cap {\mathbb Z}$ where $M<N$ are natural numbers. Then, for $p\in [1,+\infty)$,
\begin{equation*}
\begin{split}
\mathcal{H}(\varepsilon, B_{\mathcal{H}_K}[S], L_p(\nu))\leq \frac{\tilde{c}p C_{K,p}(M, N)^{2}}{\varepsilon^2},
\end{split}
\end{equation*}
where $\tilde{c}$ is a universal constant. For $p=+\infty$, we have
\begin{equation*}
\begin{split}
\mathcal{H}(\varepsilon, B_{\mathcal{H}_K}[S], L_\infty(\nu))\leq \tilde{c}_1\big(\frac{\int_0^{C_{K,+\infty}(M,N)} \sqrt{\mathcal{H}(\delta,\tilde{\boldsymbol{\Omega}}, \mathcal{H}_{K})}d\delta}{\varepsilon}\big)^2,
\end{split}
\end{equation*}
where $\tilde{\boldsymbol{\Omega}} = \{K({\mathbf x}, \cdot)\mid {\mathbf x}\in \boldsymbol{\Omega}\}\subseteq \mathcal{H}_{K}$ and $\tilde{c}_1$ is a universal constant.
\end{theorem}
\begin{proof}
Recall that
\begin{equation*}
\begin{split}
B_{\mathcal{H}_K}[S] = \{\sum_{j=M+1}^{N}\sigma_j\xi_j\phi_j\mid \sum_{j=M+1}^{N}\xi_j^2\leq 1\}.
\end{split}
\end{equation*}
Let us denote $\|{\mathbf x}\|_p^{S} \triangleq  \|\sum_{j=M+1}^{N} \sigma_jx_j\phi_j\|_{L_p(\nu)}$ for ${\mathbf x} = [x_j]_{j=M+1}^{N}\in {\mathbb R}^{S}$. Consequently,
\begin{equation*}
\begin{split}
\mathcal{H}(\varepsilon, B_{\mathcal{H}_K}[S], L_p(\nu)) = \mathcal{H}(\varepsilon, B^{S}_2, \|\cdot\|_p^{S}),
\end{split}
\end{equation*}
where $B^{S}_2$ is a unit ball in the canonical Euclidean space ${\mathbb R}^{S}$. By Proposition~\ref{panos}, there exists an $\varepsilon$-covering of $B^{S}_2$ in $({\mathbb R}^{S}, \|\cdot\|_p^{S})$ of size ${\rm exp}( \frac{c_{\rm su}(N-M)\mathcal{M}^2_p}{\varepsilon^2})$ where $$\mathcal{M}_p = \int_{{\mathbb S}^{N-M-1}}\|\sum_{j=M+1}^{N}\sigma_j\xi_j\phi_j\|_{L_p(\nu)} d\mu_{{\mathbb S}^{N-M-1}}(\boldsymbol{\xi}).$$ 
Let us first consider the case of finite $p\geq 1$. Using Jensen's inequality and Fubini's theorem, we have
\begin{equation*}
\begin{split}
&\mathcal{M}_p = \int_{{\mathbb S}^{N-M-1}}\|\sum_{j=M+1}^{N}\sigma_j\xi_j\phi_j\|_{L_p(\nu)} d\mu_{{\mathbb S}^{N-M-1}}(\boldsymbol{\xi})\leq \\
&\Big(\int_{\boldsymbol{\Omega}}\int_{{\mathbb S}^{N-M-1}}\big|\sum_{j=M+1}^{N}\sigma_j\xi_j\phi_j({\mathbf x})\big|^p d\mu_{{\mathbb S}^{N-M-1}}(\boldsymbol{\xi})d\nu({\mathbf x})\Big)^{1/p}.
\end{split}
\end{equation*}

Further, let us use the following fact about Gaussian vectors: if $g_{M+1}, \cdots, g_N\sim \mathcal{N}(0,1)$ and ${\mathbf g} = [g_i]_{i=M+1}^N$, then $\|{\mathbf g}\|_2$ and  $\frac{{\mathbf g}}{\|{\mathbf g}\|_2}$ are independent and  $\frac{{\mathbf g}}{\|{\mathbf g}\|_2}$ is distributed according to the rotation invariant probability measure on ${\mathbb S}^{N-M-1}$, i.e.
\begin{equation*}
\begin{split}
\int_{{\mathbb S}^{N-M-1}}\big|\sum_{j=M+1}^{N}\sigma_j\xi_j\phi_j({\mathbf x})\big|^p d\mu_{{\mathbb S}^{N-M-1}}(\boldsymbol{\xi}) = \frac{{\mathbb E}_{g_i}[|\sum_{j=M+1}^{N}\sigma_jg_j\phi_j({\mathbf x})|^p]}{{\mathbb E}_{g_i}[|\sum_{j=M+1}^{N} g_j^2|^{p/2}]}.
\end{split}
\end{equation*}
Then, using ${\mathbb E}\chi^2(N-M)^{\frac{p}{2}}=2^{\frac{p}{2}}\frac{\Gamma(\frac{N-M+p}{2})}{\Gamma(\frac{N-M}{2})} $ where $\chi^2(N-M)$ has the chi-square distribution with $N-M$ degrees of freedom and, we bound $\mathcal{M}_p$ by
\begin{equation*}
\begin{split}
&\Big(\frac{B_{p}}{2^{\frac{p}{2}}\frac{\Gamma(\frac{N-M+p}{2})}{\Gamma(\frac{N-M}{2})}}\int_{\boldsymbol{\Omega}} \big(\sum_{j=M+1}^{N}\sigma_j^2\phi_j({\mathbf x})^2\big)^{\frac{p}{2}}d\nu({\mathbf x})\Big)^{1/p} =\frac{B_{p}^{1/p}\Gamma(\frac{N-M}{2})^{1/p}}{2^{\frac{1}{2}}\Gamma(\frac{N-M+p}{2})^{1/p}}C_{K,p}(M,N).
\end{split}
\end{equation*}
The log-convexity of $\Gamma(x)$ gives us
\begin{equation*}
\begin{split}
\log\Gamma(\frac{x+p}{2})-\log\Gamma(\frac{x}{2})\geq \frac{p}{2}\psi(\frac{x}{2}),
\end{split}
\end{equation*}
where $\psi(x) = \frac{\Gamma'(x)}{\Gamma(x)}$ and $x>0$.
Therefore, for $x\geq \frac{1}{2}$,
\begin{equation*}
\begin{split}
&\frac{\Gamma(\frac{x}{2})^{1/p}}{\Gamma(\frac{x+p}{2})^{1/p}}\leq 
e^{-\frac{1}{2}\psi(\frac{x}{2})}<\sqrt{2e^{\gamma}}\sqrt{\frac{2}{x}}.
\end{split}
\end{equation*}
Above we used $\psi(z) -\log{z} \geq \psi(\frac{1}{2}) -\log{\frac{1}{2}} = -\log 2-\gamma$ for $z\geq \frac{1}{2}$, where $\gamma$ is Euler's constant.

Thus,
\begin{equation*}
\begin{split}
&\mathcal{H}(\varepsilon, B_{\mathcal{H}_K}[S], L_p(\nu))\leq 2e^{\gamma}c_{\rm su} \frac{(N-M) }{\varepsilon^2}\frac{B_{p}^{2/p}C_{K,p}(M, N)^2}{N-M}\leq 2e^{\gamma}c_{\rm su} B_{p}^{2/p} \frac{ C_{K,p}(M, N)^2}{\varepsilon^2}.
\end{split}
\end{equation*}
Finally, from Stirling's formula we get
\begin{equation*}
\begin{split}
&\frac{1}{p}\log B_{p} = \frac{1}{p}\log\frac{2^{p/2}\Gamma(\frac{p+1}{2})}{\sqrt{\pi}} = \frac{1}{2}\log 2-\frac{1}{2p}\log \pi+\\
&\frac{1}{p}(\frac{p+1}{2}\log (\frac{p+1}{2})-\frac{p+1}{2}-\frac{1}{2}\log (\frac{p+1}{2})+\frac{1}{2}\log (2\pi)+O(\frac{1}{p}))=\\
&\frac{1}{2}\log(p+1)-\frac{1}{2}+O(\frac{1}{p}),
\end{split}
\end{equation*}
and
\begin{equation*}
\begin{split}
\mathcal{H}(\varepsilon, B_{\mathcal{H}_K}[S], L_p(\nu))\leq \tilde{c} p \frac{ C_{K,p}(M, N)^2}{\varepsilon^2}.
\end{split}
\end{equation*}

Now let us consider the case of $L_\infty(\nu)$. Let us denote $F({\mathbf x}, {\mathbf g}) \triangleq \sum_{j=M+1}^{N}\sigma_jg_j\phi_j({\mathbf x})$ where $g_{M+1}, \cdots, g_{N}\sim \mathcal{N}(0,1)$ and ${\mathbf g} = [g_j]_{j=M+1}^{N}$. Note that $F({\mathbf x}, {\mathbf g})$ can be treated as a collection of random variables $\{F({\mathbf x}, \cdot)\}_{{\mathbf x}\in \boldsymbol{\Omega}}$, i.e. as a Gaussian Random Field, with a covariance function
\begin{equation*}
\begin{split}
{\mathbb E}[F({\mathbf x}, {\mathbf g})F({\mathbf y}, {\mathbf g})] = \sum_{j=M+1}^{N}\sigma^2_j \phi_j({\mathbf x})\phi_j({\mathbf y}).
\end{split}
\end{equation*}
Let us denote the Mercer kernel $\sum_{j=M+1}^{N}\sigma^2_j \phi_j({\mathbf x})\phi_j({\mathbf y})$ by $K_S ({\mathbf x},{\mathbf y})$. 

Since $\mathcal{H}(\varepsilon, B_{\mathcal{H}_K}[S], L_\infty(\nu))\leq e^{\frac{c_{\rm su} (N-M)\mathcal{M}^2_\infty}{\varepsilon^2}}$, here we need to bound
\begin{equation}\label{to-GRF}
\begin{split}
&\mathcal{M}_\infty = \int_{{\mathbb S}^{N-M-1}}\|\sum_{j=M+1}^{N}\sigma_j\xi_j\phi_j\|_{L_\infty(\nu)} d\mu_{{\mathbb S}^{N-M-1}}(\boldsymbol{\xi})=\\ &\frac{{\mathbb E}_{g_i}[\sup_{{\mathbf x}\in \boldsymbol{\Omega}}|\sum_{j=M+1}^{N}\sigma_jg_j\phi_j({\mathbf x})|]}{{\mathbb E}_{g_i}[|\sum_{j=M+1}^{N} g_j^2|^{1/2}]} = \frac{{\mathbb E}_{g_i}[\sup_{{\mathbf x}\in \boldsymbol{\Omega}}|F({\mathbf x}, {\mathbf g})|]}{{\mathbb E}_{g_i}[|\sum_{j=M+1}^{N} g_j^2|^{1/2}]}.
\end{split}
\end{equation}

Let $\boldsymbol{\Omega}_S = \{K_S({\mathbf x}, \cdot)\mid {\mathbf x}\in \boldsymbol{\Omega}\}\subseteq \mathcal{H}_{K_S}$. 
According to Dudley's integral covering number bound~\cite{10.1007/978-3-319-40519-3_2}, there is a universal constant $c_1>0$ such that
\begin{equation*}
\begin{split}
\sup_{{\mathbf x}\in \boldsymbol{\Omega}} |F({\mathbf x}, {\mathbf g})| \leq c_1\int_0^\infty \sqrt{\mathcal{H}(\delta,\boldsymbol{\Omega}_S, \mathcal{H}_{K_S})}d\delta.
\end{split}
\end{equation*}
Since $\|K_S({\mathbf x}, \cdot)\|_{\mathcal{H}_{K_S}} = \sqrt{K_S({\mathbf x}, {\mathbf x})}\leq C_{K,+\infty}(M,N)$ we obtain
\begin{equation*}
\begin{split}
\int_0^\infty \sqrt{\mathcal{H}(\delta,\boldsymbol{\Omega}_S, \mathcal{H}_{K_S})}d\delta = \int_0^{C_{K,+\infty}(M,N)} \sqrt{\mathcal{H}(\delta,\boldsymbol{\Omega}_S, \mathcal{H}_{K_S})}d\delta.
\end{split}
\end{equation*}
Also, from 
\begin{equation*}
\begin{split}
&\|K_S({\mathbf x}, \cdot)-K_S({\mathbf y}, \cdot)\|_{\mathcal{H}_{K_S}} = \sqrt{\sum_{j=M+1}^{N}\sigma^2_j (\phi_j({\mathbf x})-\phi_j({\mathbf y}))^2}\leq \\ &\sqrt{\sum_{j=1}^{+\infty}\sigma^2_j (\phi_j({\mathbf x})-\phi_j({\mathbf y}))^2} = \|K({\mathbf x}, \cdot)-K({\mathbf y}, \cdot)\|_{\mathcal{H}_{K}},
\end{split}
\end{equation*}
we conclude
\begin{equation*}
\begin{split}
\int_0^{C_{K,+\infty}(M,N)} \sqrt{\mathcal{H}(\delta,\boldsymbol{\Omega}_S, \mathcal{H}_{K_S})}d\delta\leq \int_0^{C_{K,+\infty}(M,N)} \sqrt{\mathcal{H}(\delta,\tilde{\boldsymbol{\Omega}}, \mathcal{H}_{K})}d\delta,
\end{split}
\end{equation*}
where $\tilde{\boldsymbol{\Omega}} = \{K({\mathbf x}, \cdot)\mid {\mathbf x}\in \boldsymbol{\Omega}\}\subseteq \mathcal{H}_{K}$. Finally, from the equation~\eqref{to-GRF} and ${\mathbb E}_{g_i}[|\sum_{j=M+1}^{N} g_j^2|^{1/2}] = \sqrt{N-M}(1+\mathcal{O}(\frac{1}{N-M}))$, we get
\begin{equation*}
\begin{split}
&\mathcal{H}(\varepsilon, B_{\mathcal{H}_K}[S], L_\infty(\nu))\leq \frac{c_{\rm su}(N-M)\mathcal{M}^2_\infty}{\varepsilon^2} \leq \\
&c_{\rm su} c^2_1 \frac{(N-M)}{\varepsilon^2} (\frac{1+\mathcal{O}(\frac{1}{N-M})}{\sqrt{N-M}}\int_0^{C_{K,+\infty}(M,N)} \sqrt{\mathcal{H}(\delta,\tilde{\boldsymbol{\Omega}}, \mathcal{H}_{K})}d\delta)^2,
\end{split}
\end{equation*}
from which the statement of Lemma is straightforward.
\end{proof}

\begin{remark}\label{non-exotic} Unless the kernel $K$ is exotic, the bound of Theorem~\ref{from-sudakov} for $p=+\infty$ behaves similarly to the bound for finite $p$.

It is natural to assume 
\begin{equation}\label{property-k}
\begin{split}
\sqrt{K({\mathbf x}, {\mathbf x})+K({\mathbf y}, {\mathbf y})-2K({\mathbf x}, {\mathbf y})}\leq C\|{\mathbf x}- {\mathbf y}\|^\alpha
\end{split}
\end{equation}
for $\alpha >0$. 
For example, this assumption is satisfied if $K({\mathbf x}, {\mathbf y}) = k({\mathbf x}- {\mathbf y})$ and $k$ is an $\alpha$-H{\"o}lder function. For $\alpha=1$, a sufficient condition for it is twice continuous differentiability of $K$ (see Appendix~\ref{smooth-kern}).

Let $R=\max_{{\mathbf x}\in \boldsymbol{\Omega}}\|{\mathbf x}\|_2$. Then, 
\begin{equation*}
\begin{split}
\mathcal{H}(\delta,\tilde{\boldsymbol{\Omega}}, \mathcal{H}_{K})\leq \mathcal{H}((\frac{\delta}{C})^{1/\alpha},\boldsymbol{\Omega}, \|\cdot\|_2)\leq \frac{n}{\alpha}\log(\frac{R^\alpha C}{\delta})+c\log(n+1) = \frac{n}{\alpha}\log(\frac{C_{n,\alpha, C, R}}{\delta}),
\end{split}
\end{equation*}
where $C_{n,\alpha, C, R} = R^\alpha Ce^{\frac{c \alpha\log(n+1)}{n}}$ and $c$ is a universal constant from Roger's bound.
Therefore, Theorem~\ref{from-sudakov} gives us
\begin{equation*}
\begin{split}
&\mathcal{H}(\varepsilon, B_{\mathcal{H}_K}[S], L_\infty(\nu))\leq \tilde{c}_1\big(\frac{\int_0^{C_{K,+\infty}(M,N)} \sqrt{\frac{n}{\alpha}\log(\frac{C_{n,\alpha, C, R}}{\delta})}d\delta}{\varepsilon}\big)^2= \\
&\frac{\tilde{c}_1 nC_{n,\alpha, C, R}^2}{\alpha \varepsilon^2}\big(\int_0^{C_{K,+\infty}(M,N)/C_{n,\alpha, C, R}} \sqrt{\log(\frac{1}{\delta})}d\delta\big)^2.
\end{split}
\end{equation*}
Since $\int_0^{x} \sqrt{\log(\frac{1}{\delta})}d\delta = \int_{\log(1/x)}^{+\infty}\sqrt{t}e^{-t}dt\lesssim x\sqrt{\log(\frac{1}{x})}$ for $0<x<\frac{1}{2}$, we conclude
\begin{equation}\label{regular-kernel}
\begin{split}
\mathcal{H}(\varepsilon, B_{\mathcal{H}_K}[S], L_\infty(\nu))\lesssim
\frac{n}{ \alpha }\cdot \frac{C_{K,+\infty}(M,N)^2\log(\frac{C_{n,\alpha, C, R} }{C_{K,+\infty}(M,N)})}{\varepsilon^2},
\end{split}
\end{equation}
if $C_{K,+\infty}(M,N)<\frac{1}{2}C_{n,\alpha, C, R} $. As can be seen, the latter bound differs from the bound for a finite $p$ by logarithmic factor $\log(\frac{C_{n,\alpha, C, R} }{C_{K,+\infty}(M,N)})$ and the constant factor that is proportional to the dimension $n$.
\end{remark}

\begin{remark} In Theorem~\ref{from-sudakov} one can set $M=0$, $N=+\infty$. Then, our bound 
$$\mathcal{H}(\varepsilon, B_{\mathcal{H}_K}, L_p(\nu))\leq \frac{\tilde{c}p}{\varepsilon^2}\|\sqrt{K({\mathbf x},{\mathbf x})}\|^2_{L_p(\nu)},$$
for $p=2$, can be considered as an entropic version of an earlier  bound on the Rademacher complexity of $B_{\mathcal{H}_K}$ found by Bartlett-Mendelson~\cite{BartlettMendelson}. Recall that the empirical Rademacher complexity of $B_{\mathcal{H}_K}$ is defined as the following expression:
\begin{equation*}
\begin{split}
R_N({\mathbf x}_1, \cdots, {\mathbf x}_N) =  \sup_{f\in B_{\mathcal{H}_K}}|\frac{2}{N}\sum_{i=1}^N \sigma_i f({\mathbf x}_i)|,
\end{split}
\end{equation*}
where $N\in {\mathbb N}$ and $\sigma_1, \cdots, \sigma_N$ are independent uniform $\{\pm 1\}$-valued random variables. The Rademacher complexity is linked with $\mathcal{H}(\varepsilon, B_{\mathcal{H}_K}, L_2(P_N))$, where $P_N({\mathbf x}) =\frac{1}{N}\sum_{i=1}^N \delta_{{\mathbf x}_i}({\mathbf x})$ and $\delta_{{\mathbf x}_i}({\mathbf x})$ is a distribution concentrated at a point ${\mathbf x}_i$, by a well-known version of Dudley's integral covering number bound~\cite{Rebeschini}:
\begin{equation*}
\begin{split}
R_N({\mathbf x}_1, \cdots, {\mathbf x}_N)\leq \inf_{\varepsilon>0}\big(4\varepsilon+\int_{\varepsilon}^{+\infty}\sqrt{\frac{\mathcal{H}(\tau, B_{\mathcal{H}_K}, L_2(P_N))}{N}}d\tau\big).
\end{split}
\end{equation*}
Our bound implies $\mathcal{H}(\tau, B_{\mathcal{H}_K}, L_2(P_N)) \lesssim \frac{\|\sqrt{K({\mathbf x},{\mathbf x})}\|^2_{L_2(P_N)}}{\tau^2}$. After setting $\varepsilon = \frac{\|\sqrt{K({\mathbf x},{\mathbf x})}\|_{L_2(P_N)}}{\sqrt{N}}$, from Dudley's bound we obtain
\begin{equation*}
\begin{split}
R_N({\mathbf x}_1, \cdots, {\mathbf x}_N)\lesssim \frac{ \|\sqrt{K({\mathbf x},{\mathbf x})}\|_{L_2(P_N)}}{\sqrt{N}}\log(\frac{\sqrt{N}}{\|\sqrt{K({\mathbf x},{\mathbf x})}\|_{L_2(P_N)}}),
\end{split}
\end{equation*}
which is equivalent to Bartlett-Mendelson's bound up to logarithmic factor.
\end{remark}

\section{Proof of the main upper bound}\label{proof-upper}
In our main bound of $\mathcal{H}^{\rm ext}(\varepsilon, B_{\mathcal{H}_K}, L_p(\nu))$ we will use the following statement, which is given as Theorem 2 in~\cite{DumerIlya} (its proof can be extracted from~\cite{Dumer}). In fact, it is a consequence of the slightly more general Remark 5.15 from~\cite{pisier_1989}.
\begin{proposition}[Dumer-Pinsker-Prelov]\label{Dum} For any $\theta\in (0,1-\frac{1}{\sqrt{2}})$ we have
\begin{equation*}
\begin{split}
&\mathcal{E}(1, \{\sigma_i\}_{i=1}^N) \leq \aleph(1, \{\sigma_i\}_{i=1}^N)\leq  \\
&\mathcal{E}(1, \{\sigma_i\}_{i=1}^N) +|\{i\in [N]\mid \sigma_i> 1-\theta\}| \log\frac{3}{\theta(2-\theta)}.
\end{split}
\end{equation*}
\end{proposition}

We proved Theorem~\ref{bound-for-finite}, that bounds  the $\varepsilon$-entropy of the ``truncated'' version of $B_{\mathcal{H}_K}$, $B_{\mathcal{H}_K}[N]$, and Theorem~\ref{from-sudakov}, that bounds the $\varepsilon$-entropy of the ``tail'' part of $B_{\mathcal{H}_K}$. We made all preparations to bound the $\varepsilon$-entropy of the total $B_{\mathcal{H}_K}$.

\begin{proof}[Proof of Theorem~\ref{Bound-main}] 
Let us first consider the case of $1\leq p\leq 2$. 
Let $N\in {\mathbb N}\cup \{0\}$.
We have
\begin{equation*}
\begin{split}
B_{\mathcal{H}_K}\subseteq B_{\mathcal{H}_K}[N]\oplus B_{\mathcal{H}_K}[T],
\end{split}
\end{equation*}
where $T=\{N+1, N+2, \cdots\}$. 

Let $\varrho_N = \sup_{f\in B_{\mathcal{H}_K}[T]}\|f\|_{L_p(\nu)}$ and let us first verify that $\lim_{N\to +\infty}\varrho_N =0$. Since $\|f\|_{L_p(\nu)}\leq \nu(\boldsymbol{\Omega})^{1/p-1/2}\|f\|_{L_2(\nu)}$ it is enough to do that for $p=2$.
Indeed, for $f=\sum_{i=N+1}^\infty \xi_i\phi_i\in B_{\mathcal{H}_K}[T]$, we have
\begin{equation*}
\begin{split}
&\|f\|_{L_2(\nu)} = \|\sum_{i=N+1}^\infty \xi_i\phi_i\|_{L_2(\nu)} = \big(\sum_{i=N+1}^\infty \xi_i^2\big)^{1/2}\leq \\
&\sigma_{N+1}\big(\sum_{i=N+1}^\infty \frac{\xi_i^2}{\sigma_i^2}\big)^{1/2}\leq \sigma_{N+1}\mathop\rightarrow\limits^{N\to +\infty} 0.
\end{split}
\end{equation*}
Note that any $\varepsilon-\varrho_N$-covering of $B_{\mathcal{H}_K}[N]$ in $L_p(\nu)$ is an $\varepsilon$-covering of $B_{\mathcal{H}_K}[N]\oplus B_{\mathcal{H}_K}[T]$ in $L_p(\nu)$. Therefore,
\begin{equation*}
\begin{split}
&\mathcal{H}^{\rm ext}(\varepsilon, B_{\mathcal{H}_K}, L_p(\nu))\leq \mathcal{H}^{\rm ext}(\varepsilon, B_{\mathcal{H}_K}[N]\oplus B_{\mathcal{H}_K}[T], L_p(\nu))\leq \\
&\mathcal{H}^{\rm ext}(\varepsilon-\varrho_N, B_{\mathcal{H}_K}[N], L_p(\nu)).
\end{split}
\end{equation*}
Using Theorem~\ref{bound-for-finite} and Dumer-Pinsker-Prelov's bound we obtain
\begin{equation*}
\begin{split}
&\mathcal{H}^{\rm ext}(\varepsilon-\varrho_N, B_{\mathcal{H}_K}[N], L_p(\nu))\leq 
\aleph( \nu(\boldsymbol{\Omega})^{1/2-1/p}(\varepsilon-\varrho_N), \{\sigma_i\}_{i=1}^N)\leq \\
&\mathcal{E}(\nu(\boldsymbol{\Omega})^{1/2-1/p}(\varepsilon-\varrho_N), \{\sigma_i\}_{i=1}^N)+\\
&|\{i\in [N]\mid \sigma_i > (1-\theta)\nu(\boldsymbol{\Omega})^{1/2-1/p}(\varepsilon-\varrho_N)\}|\log\frac{3}{\theta(2-\theta)}.
\end{split}
\end{equation*}
Finally, sending $N$ to infinity gives us the desired bound.

Let us now turn to the case of $p>2$. Let $M,N\in {\mathbb N}\cup \{0\}$ where $N>M$.
We have
\begin{equation*}
\begin{split}
B_{\mathcal{H}_K}\subseteq B_{\mathcal{H}_K}[M]\oplus B_{\mathcal{H}_K}[S]\oplus B_{\mathcal{H}_K}[T],
\end{split}
\end{equation*}
where $S=\{M+1, M+2, \cdots, N\}$ and $T = \{N+1, N+2, \cdots\}$.

Let $\varrho_N = \sup_{f\in B_{\mathcal{H}_K}[T]}\|f\|_{L_p(\nu)}$ and let us again verify that $\lim_{N\to +\infty}\varrho_N =0$. For $f=\sum_{i=N+1}^\infty \sigma_i\xi_i\phi_i\in B_{\mathcal{H}_K}[T]$, 
using the Cauchy-Schwarz inequality, we have
\begin{equation*}
\begin{split}
&|f({\mathbf x})| =  |\sum_{i=N+1}^\infty \sigma_i\xi_i\phi_i({\mathbf x})| \leq \\
&\big(\sum_{i=N+1}^\infty |\xi_i|^2\big)^{1/2}\big(\sum_{i=N+1}^\infty\sigma^2_i|\phi_i({\mathbf x})|^2\big)^{1/2}\leq \big(\sum_{i=N+1}^\infty\sigma^2_i|\phi_i({\mathbf x})|^2\big)^{1/2}.
\end{split}
\end{equation*}
Since $\sum_{i=N+1}^\infty\sigma^2_i|\phi_i({\mathbf x})|^2 = K({\mathbf x},{\mathbf x})-\sum_{i=1}^N \lambda_i\phi_i({\mathbf x})^2$, we conclude
\begin{equation*}
\begin{split}
&\|f({\mathbf x})\|_{L_p(\nu)}^p  \leq
\int_{\boldsymbol{\Omega}}\big(\sum_{i=N+1}^\infty\sigma^2_i|\phi_i({\mathbf x})|^2\big)^{p/2}d\nu({\mathbf x}) = \\
&\|\sqrt{K({\mathbf x},{\mathbf x})-\sum_{i=1}^N \lambda_i\phi_i({\mathbf x})^2}\|_{L_{p}(\nu)}^{p} = C_{K,p}(N)^{p}.
\end{split}
\end{equation*}
Thus, $\varrho_N \leq C_{K,p}(N)\mathop\rightarrow\limits^{N\to +\infty} 0$.

Note that we have $\mathcal{H}^{\rm ext}(\varrho_N, B_{\mathcal{H}_K}[T], L_p(\nu))=0$.
By construction, for any $\delta\in (0,\varepsilon-\varrho_N)$, any $\varepsilon-\varrho_N-\delta$-covering of $B_{\mathcal{H}_K}[M]$ of size $C_1$ in $L_p(\nu)$ and any $\delta$-covering of $B_{\mathcal{H}_K}[S]$ of size $C_2$ in $L_p(\nu)$, there is an $\varepsilon-\varrho_N$-covering of $B_{\mathcal{H}_K}[M]\oplus B_{\mathcal{H}_K}[S]$ of size $C_1 C_2$ in $L_p(\nu)$. And the latter covering will also be an $\varepsilon$-covering of $B_{\mathcal{H}_K}[M]\oplus B_{\mathcal{H}_K}[S]\oplus B_{\mathcal{H}_K}[T]$.  Therefore,
\begin{equation*}
\begin{split}
&\mathcal{H}^{\rm ext}(\varepsilon, B_{\mathcal{H}_K}, L_p(\nu))\leq \\
&\mathcal{H}^{\rm ext}(\varepsilon-\varrho_N-\delta, B_{\mathcal{H}_K}[M], L_p(\nu))+\mathcal{H}^{\rm ext}(\delta, B_{\mathcal{H}_K}[S], L_p(\nu)).
\end{split}
\end{equation*}
Using Theorem~\ref{from-sudakov} we bound the second term 
\begin{equation*}
\begin{split}
&\mathcal{H}^{\rm ext}(\delta, B_{\mathcal{H}_K}[S], L_p(\nu))\leq \\ 
&\begin{cases}
			\frac{\tilde{c}p C_{K,p}(M, N)^{2}}{\delta^2} & \text{if $2<p<+\infty$}\\
			\tilde{c}_1\big(\frac{\int_0^{C_{K,+\infty}(M,N)} \sqrt{\mathcal{H}(t,\tilde{\boldsymbol{\Omega}}, \mathcal{H}_{K})}dt}{\delta}\big)^2& \text{if $p=+\infty$}\\
		 \end{cases},
\end{split}
\end{equation*}
where $\tilde{\boldsymbol{\Omega}} = \{K({\mathbf x}, \cdot)\mid {\mathbf x}\in \boldsymbol{\Omega}\}\subseteq \mathcal{H}_{K}$. Sending $N$ to infinity gives us
\begin{equation*}
\begin{split}
&\mathcal{H}^{\rm ext}(\varepsilon, B_{\mathcal{H}_K}, L_p(\nu))\leq 
\mathcal{H}^{\rm ext}(\varepsilon-\delta, B_{\mathcal{H}_K}[M], L_p(\nu))+\\
&\begin{cases}
			\frac{\tilde{c}p C_{K,p}(M)^{2}}{\delta^2} & \text{if $2<p<+\infty$}\\
			\tilde{c}_1\big(\frac{\int_0^{C_{K,+\infty}(M)} \sqrt{\mathcal{H}(t,\tilde{\boldsymbol{\Omega}}, \mathcal{H}_{K})}dt}{\delta}\big)^2& \text{if $p=+\infty$}\\
		 \end{cases},
\end{split}
\end{equation*}
for any $\delta\in (0,\varepsilon)$.
Using Theorem~\ref{bound-for-finite} we bound the first term
\begin{equation*}
\begin{split}
&\mathcal{H}^{\rm ext}(\varepsilon-\delta, B_{\mathcal{H}_K}[M], L_p(\nu))\leq \\ 
&\begin{cases}
			\min\limits_{\alpha\in [0,2/p]}\aleph(\frac{\varepsilon-\delta}{C_K^{1-\alpha} \nu(\boldsymbol{\Omega})^{1/p-\alpha/2}M^{\alpha/2}},\{\sigma_i^{\alpha}\}_{i=1}^M)& \text{if $2<p<+\infty$}\\
			\aleph(\frac{\varepsilon-\delta}{C_K}, M)& \text{if $p=+\infty$}\\
		 \end{cases}.
\end{split}
\end{equation*}
For $p=+\infty$, the application of Roger's bound on $\aleph(\frac{\varepsilon-\delta}{C_K}, M)$ leads us to
\begin{equation*}
\begin{split}
&\mathcal{H}^{\rm ext}(\varepsilon, B_{\mathcal{H}_K}, L_\infty(\nu))\leq \\
& M\log(\frac{C_K}{\varepsilon-\delta})+c\log(M+1)+\tilde{c}_1\big(\frac{\int_0^{C_{K,+\infty}(M)} \sqrt{\mathcal{H}(t,\tilde{\boldsymbol{\Omega}}, \mathcal{H}_{K})}dt}{\delta}\big)^2,
\end{split}
\end{equation*}
for any $M\in {\mathbb N}\cup \{0\}$.

For finite $p>2$, application of Dumer-Pinsker-Prelov's bound leads us to 
\begin{equation*}
\begin{split}
&\mathcal{H}^{\rm ext}(\varepsilon, B_{\mathcal{H}_K}, L_p(\nu))\leq \\
&\min_{\alpha\in [0,2/p]} \Big\{\mathcal{E}(\varepsilon-\delta, \{C_K^{1-\alpha} \nu(\boldsymbol{\Omega})^{1/p-\alpha/2}M^{\alpha/2}\sigma_i^\alpha\}_{i=1}^M)+\\
&|\{i\in [M]\mid C_K^{1-\alpha} \nu(\boldsymbol{\Omega})^{1/p-\alpha/2}M^{\alpha/2}\sigma_i^\alpha > (1-\theta)(\varepsilon-\delta)\}|\log\frac{3}{\theta(2-\theta)}+\\
&\frac{\tilde{c}p C_{K,p}(M)^{2}}{\delta^2}
\Big\}.
\end{split}
\end{equation*}
Finally,  and varying over $\delta\in (0,\varepsilon)$ and $M\in {\mathbb N}\cup \{0\}$ gives us the desired result.
\end{proof}

\section{Bounding the decay rate of $C_{K,\infty}(N)$}
In this section we will prove Corollary~\ref{corol}.
\begin{lemma}\label{residual} If $(\boldsymbol{\Omega}, \nu, K, \boldsymbol{\sigma}, \boldsymbol{\phi}, \varkappa)\in \mathcal{K}$ and $\lim_{i\to +\infty}i\sigma_i^2 =0$, then
\begin{equation*}
\begin{split}
C_{K,\infty}(N)\leq \varkappa\sum_{i=N+1}^{\infty}i|\sigma_i^2-\sigma_{i+1}^2|.
\end{split}
\end{equation*}
\end{lemma}
\begin{proof}
Let us denote $S_i({\mathbf x}) = \sum_{j=1}^{i-1} \phi_j({\mathbf x})^2$. Then, $S_{i+1}({\mathbf x})\leq \varkappa i$. Using the summation by parts formula we obtain
\begin{equation*}
\begin{split}
&\sum_{i=N+1}^{\infty} \sigma_i^2\phi_i({\mathbf x})^2 = \sum_{i=N+1}^{\infty} \sigma_i^2(S_{i+1}({\mathbf x})-S_{i}({\mathbf x})) =\\
&\lim_{i\to +\infty}\sigma_i^2 S_i({\mathbf x})-\sigma_{N+1}^2 S_{N+1}({\mathbf x})-\sum_{i=N+1}^{\infty}S_{i+1}({\mathbf x})(\sigma_{i+1}^2-\sigma_i^2)\leq \\
&\varkappa\sum_{i=N+1}^{\infty}i|\sigma_i^2-\sigma_{i+1}^2|.
\end{split}
\end{equation*}
Lemma proved.
\end{proof}

\begin{proof}[Proof of Corollary~\ref{corol}] The previous lemma shows that $(\boldsymbol{\Omega}, \nu, K, \boldsymbol{\sigma}, \boldsymbol{\phi}, \varkappa)\in \mathcal{K}$ and $\lim_{i\to +\infty}i\sigma_i^2 =0$ are sufficient to obtain an upper bound on $\mathcal{H}^{\rm ext}(2\varepsilon, B_{\mathcal{H}_K}, C(\boldsymbol{\Omega}))$ purely in terms of the decay rate of eigenvalues and $\varkappa$. Recall that $m_\varepsilon = |\{i \mid \sigma_i> \varepsilon\}|$. Indeed, let us set $\delta=\varepsilon$ and $N=m_\varepsilon$ for the minimization operator arguments in Theorem~\ref{Bound-main}. We obtain
\begin{equation*}
\begin{split}
&\mathcal{H}^{\rm ext}(2\varepsilon, B_{\mathcal{H}_K}, C(\boldsymbol{\Omega}))\leq \\
&m_\varepsilon\log(\frac{C_K}{\varepsilon})+c\log(m_\varepsilon+1)+\tilde{c}_1\big(\frac{\int_0^{C_{K,+\infty}(m_\varepsilon)} \sqrt{\mathcal{H}(t,\tilde{\boldsymbol{\Omega}}, \mathcal{H}_{K})}dt}{\varepsilon}\big)^2.
\end{split}
\end{equation*}
where $\tilde{\boldsymbol{\Omega}} = \{K({\mathbf x}, \cdot)\mid {\mathbf x}\in \boldsymbol{\Omega}\}\subseteq \mathcal{H}_{K}$ and $c,\tilde{c}_1$ are a universal constants.
From Remark~\ref{non-exotic} we conclude that
\begin{equation*}
\begin{split}
\big(\frac{\int_0^{C_{K,+\infty}(m_\varepsilon)} \sqrt{\mathcal{H}(t,\tilde{\boldsymbol{\Omega}}, \mathcal{H}_{K})}dt}{\varepsilon}\big)^2 \lesssim \frac{n}{ \alpha }\cdot \frac{C_{K,+\infty}(m_\varepsilon)^2\log(\frac{C_{n,\alpha, C, R} }{C_{K,+\infty}(m_\varepsilon)})}{\varepsilon^2}.
\end{split}
\end{equation*}
where $C_{n,\alpha, C, R} = R^\alpha Ce^{\frac{c \alpha\log(n+1)}{n}}$. From Lemma~\ref{residual} we have
\begin{equation*}
\begin{split}
C_{K,\infty}(m_\varepsilon)\leq \varkappa\sum_{i=m_\varepsilon+1}^{\infty}i|\sigma_i^2-\sigma_{i+1}^2|.
\end{split}
\end{equation*}
Therefore,
\begin{equation*}
\begin{split}
&\mathcal{H}^{\rm ext}(2\varepsilon, B_{\mathcal{H}_K}, C(\boldsymbol{\Omega}))\leq
m_\varepsilon\log(\frac{C_K}{\varepsilon})+c\log(m_\varepsilon+1)+\\
&\frac{\vardbtilde{c} n}{ \alpha }\cdot \frac{\varkappa^2(\sum_{i=m_\varepsilon+1}^{\infty}i|\sigma_i^2-\sigma_{i+1}^2|)^2\log(\frac{C_{n,\alpha, C, R} }{\varkappa\sum_{i=m_\varepsilon+1}^{\infty}i|\sigma_i^2-\sigma_{i+1}^2|})}{\varepsilon^2},
\end{split}
\end{equation*}
under condition that $\varkappa\sum_{i=m_\varepsilon+1}^{\infty}i|\sigma_i^2-\sigma_{i+1}^2|\leq \frac{1}{2}C_{n,\alpha, C, R}$. Corollary proved.
\end{proof}

\section{Lower bounds based on volume considerations}\label{lower-section}

For $p\geq 1$, $p'$ denotes $\frac{p}{p-1}$, i.e. $\frac{1}{p}+\frac{1}{p'}=1$. Suppose that $g_1, \cdots, g_L\sim^{\rm iid} \mathcal{N}(0,1)$.
Then, for any $x_1, \cdots, x_L\in {\mathbb R}$, $\sum_{i=1}^L g_i x_i\sim \mathcal{N}(0, \sum_{i=1}^L x_i^2)$. Consequently, the $p$th absolute moment of the latter normal random variable equals $\frac{2^{p/2}\Gamma(\frac{p+1}{2})}{\sqrt{\pi}}(\sum_{i=1}^L x_i^2)^{\frac{p}{2}}$, i.e.
\begin{equation*}
\begin{split}
{\mathbb E}_{g_1, \cdots, g_L\sim^{\rm iid} \mathcal{N}(0,1)}\big[|\sum_{i=1}^L g_i x_i|^{p}\big] = \frac{2^{p/2}\Gamma(\frac{p+1}{2})}{\sqrt{\pi}}(\sum_{i=1}^L x_i^2)^{\frac{p}{2}}.
\end{split}
\end{equation*}
Recall that the constant $\frac{2^{p/2}\Gamma(\frac{p+1}{2})}{\sqrt{\pi}}$ is denoted by $B_{p}$.

Given $\boldsymbol{\xi}=\langle \xi_i\rangle_{i\in S}\in {\mathbb R}^S$ we denote $\sum_{i\in S}\xi_i \phi_i({\mathbf x})$ by $\boldsymbol{\phi}^{\boldsymbol{\xi}}_S({\mathbf x})$. Any $f\in L_p(\nu)$ induces the following function of $\boldsymbol{\xi}$:
\begin{equation*}
\begin{split}
\|\boldsymbol{\xi}\|^f_{p} \triangleq \|f+\sum_{i\in S}\xi_i \phi_i({\mathbf x})\|_{L_p(\nu)}.
\end{split}
\end{equation*}
For $f=0$, the latter function is a norm on ${\mathbb R}^S$. Let us denote that norm by $\|\boldsymbol{\xi}\|^\ast_{p}$.

For $N=|S|$, the expression $\int_{{\mathbb S}^{N-1}}\|\boldsymbol{\xi}\|^\ast_{p}d\mu_{{\mathbb S}^{N-1}}(\boldsymbol{\xi})$ plays an important role in bounding the volume of a unit ball in $({\mathbb R}^S, \|\cdot\|_p^\ast)$, that is why our first lemma is dedicated to a general bound of this integral.
\begin{lemma}\label{pisier} Let $1\leq p< \infty$, $\frac{1}{p}+\frac{1}{p'}=1$, and $N=|S|$. Then,
\begin{equation*}
\begin{split}
\int_{{\mathbb S}^{N-1}}\|\boldsymbol{\chi}\|^\ast_{p}d\mu_{{\mathbb S}^{N-1}}(\boldsymbol{\chi})\leq
\begin{cases}
\nu(\boldsymbol{\Omega})^{\frac{1}{p}-\frac{1}{2}}(B_{p})^{\frac{1}{p}}c_{N,p}	, & \text{if $1\leq p\leq 2$}\\
\frac{C_{K,p}}{\langle \{\sigma_i\}_{i\in S}\rangle_{-2}}, & \text{if $p\geq 2$}\\
\end{cases},
\end{split}
\end{equation*}
where
\begin{equation}
c_{N,p}=\sqrt{\frac{N}{2}} \frac{\Gamma(\frac{N}{2})^{1/p}}{\Gamma(\frac{N+p}{2})^{1/p}}\mathop\to\limits^{N\to+\infty}1
\end{equation}
and $\langle \{\sigma_i\}_{i\in S}\rangle_{s} = \big(\frac{1}{N}\sum_{i\in S}\sigma_i^s\big)^{1/s}$ is a generalized mean of $\{\sigma_i\}_{i\in S}$.

If $(\boldsymbol{\Omega}, \nu, K, \boldsymbol{\sigma}, \boldsymbol{\phi}, c)\in \mathcal{K}$,  $S=\{1, \cdots, N\}$, then
\begin{equation*}
\begin{split}
\int_{{\mathbb S}^{N-1}}\|\boldsymbol{\chi}\|^\ast_{p}d\mu_{{\mathbb S}^{N-1}}(\boldsymbol{\chi})\leq    \nu(\boldsymbol{\Omega})^{1/p} c_{N,p} (B_{p})^{\frac{1}{p}}\varkappa^{\frac{1}{2}},
\end{split}
\end{equation*}
for any $p\geq 1$.
\end{lemma}
\begin{proof}
Using Jensen's inequality, we have
\begin{equation*}
\begin{split}
&\int_{{\mathbb S}^{N-1}}\|\boldsymbol{\chi}\|^\ast_{p}d\mu_{{\mathbb S}^{N-1}}(\boldsymbol{\chi}) = \int_{{\mathbb S}^{N-1}} \|\sum_{i\in S}\chi_i \phi_i({\mathbf x})\|_{L_{p}(\nu)} d\mu_{{\mathbb S}^{N-1}}(\boldsymbol{\chi}) = \\
&\int_{{\mathbb S}^{N-1}} \big(\int_{\boldsymbol{\Omega}}|\sum_{i\in S}\chi_i \phi_i({\mathbf x})|^{p}d\nu({\mathbf x}) \big)^{\frac{1}{p}}d\mu_{{\mathbb S}^{N-1}}(\boldsymbol{\chi})\leq \\
&\big(\int_{{\mathbb S}^{N-1}}\int_{\boldsymbol{\Omega}}|\sum_{i\in S}\chi_i \phi_i({\mathbf x})|^{p}d\nu({\mathbf x})d\mu_{{\mathbb S}^{N-1}}(\boldsymbol{\chi})  \big)^{\frac{1}{p}}.
\end{split}
\end{equation*}
Let us consider the case of $1\leq p\leq 2$.
As in the proof of Theorem~\ref{from-sudakov},
the integration with respect to $\mu_{{\mathbb S}^{N-1}}$ can be reduced to the integration over ${\mathbb R}^{N}$, but with respect to the Gaussian measure. Also, we use the fact that the mean of $|\sum_{j=1}^{N}g_j\phi_j({\mathbf x})\big|^p$, for standard Gaussian random variables $\{g_j\}_{j=1}^{N}$, is equal to $B_p\big(\sum_{j=1}^{N}\phi_j({\mathbf x})^2\big)^{\frac{p}{2}}$. Using concavity of $x^{p/2}$ and Jensen's inequality, we have
\begin{equation*}
\begin{split}
&\int_{\boldsymbol{\Omega}}\Big[\int_{{\mathbb S}^{N-1}}|\sum_{i\in S}\chi_i \phi_i({\mathbf x})|^{p}d\mu_{{\mathbb S}^{N-1}}(\boldsymbol{\chi})\Big]d\nu({\mathbf x}) =\\ &\frac{1}{{\mathbb E}_{g_i\sim^{\rm iid} \mathcal{N}(0,1)}\big[(\sum_{i\in S}g_i^2)^{p/2}\big]} \int_{\boldsymbol{\Omega}} {\mathbb E}_{g_i\sim^{\rm iid} \mathcal{N}(0,1)}\big[|\sum_{i\in S}g_i \phi_i({\mathbf x})|^{p}\big]d\nu({\mathbf x})= \\
&\frac{B_{p}}{2^{\frac{p}{2}}\frac{\Gamma(\frac{N+p}{2})}{\Gamma(\frac{N}{2})}}\int_{\boldsymbol{\Omega}} \big(\sum_{i\in S}\phi_i({\mathbf x})^2\big)^{\frac{p}{2}}d\nu({\mathbf x})\leq
\nu(\boldsymbol{\Omega})\frac{B_{p}N^{p/2}\nu(\boldsymbol{\Omega})^{-p/2}}{2^{\frac{p}{2}}\frac{\Gamma(\frac{N+p}{2})}{\Gamma(\frac{N}{2})}} =c_{N,p}^{p} B_{p}\nu(\boldsymbol{\Omega})^{1-p/2}.
\end{split}
\end{equation*}
where $c_{N,p}=\sqrt{\frac{N}{2}} \frac{\Gamma(\frac{N}{2})^{1/p}}{\Gamma(\frac{N+p}{2})^{1/p}}\mathop\to\limits^{N\to+\infty}1$.
Thus, we obtain the first inequality of Lemma:
\begin{equation*}
\begin{split}
\int_{{\mathbb S}^{N-1}}\|\boldsymbol{\chi}\|^\ast_{p}d\mu_{{\mathbb S}^{N-1}}(\boldsymbol{\chi})\leq c_{N,p}  \nu(\boldsymbol{\Omega})^{\frac{1}{p}-\frac{1}{2}}(B_{p})^{\frac{1}{p}} .
\end{split}
\end{equation*}
Let us consider the case of $p\geq 2$. In that case, we have
\begin{equation*}
\begin{split}
&\int_{{\mathbb S}^{N-1}}\|\boldsymbol{\chi}\|^\ast_{p}d\mu_{{\mathbb S}^{N-1}}(\boldsymbol{\chi}) = {\mathbb E}_{\boldsymbol{\chi}\sim {\mathbb S}^{N-1}} \big[\|\sum_{i\in S}\frac{\chi_i}{\sigma_i} \sigma_i\phi_i({\mathbf x})\|_{L_p(\nu)}\big] \leq \\ &\|(\sum_{i\in S}\sigma_i^2\phi_i^2({\mathbf x}))^{1/2} \|_{L_p(\nu)} {\mathbb E}_{\boldsymbol{\chi}\sim {\mathbb S}^{N-1}}  \Big[\big(\sum_{i\in S}\frac{\chi^2_i}{\sigma^2_i} \big)^{1/2}\Big] \leq \\
&C_{K,p} \Big({\mathbb E}_{\boldsymbol{\chi}\sim {\mathbb S}^{N-1}}\left[\sum_{i\in S}\frac{\chi^2_i}{\sigma^2_i}\right] \Big)^{1/2}=
C_{K,p}   \big(\sum_{i\in S}\frac{1}{N\sigma^2_i} \big)^{1/2}=\frac{C_{K,p}}{\langle \{\sigma_i\}_{i\in S}\rangle_{-2}}.
\end{split}
\end{equation*}

Let us now assume that $p\geq 1$, $(\boldsymbol{\Omega}, \nu, K, \boldsymbol{\sigma}, \boldsymbol{\phi}, c)\in \mathcal{K}$ and $S=[N]$.
As before, reducing the integration w.r.t. $\mu_{{\mathbb S}^{N-1}}$ to the expectation with Gaussian random variables gives us
\begin{equation*}
\begin{split}
&\int_{\boldsymbol{\Omega}}\int_{{\mathbb S}^{N-1}}|\sum_{i\in S}\chi_i \phi_i({\mathbf x})|^{p}d\mu_{{\mathbb S}^{N-1}}(\boldsymbol{\chi})d\nu({\mathbf x}) = \frac{B_{p}}{2^{\frac{p}{2}}\frac{\Gamma(\frac{N+p}{2})}{\Gamma(\frac{N}{2})}}\int_{\boldsymbol{\Omega}}\big(\sum_{i\in S}\phi_i({\mathbf x})^2\big)^{\frac{p}{2}}d\nu({\mathbf x})\leq \\
&\frac{\nu(\boldsymbol{\Omega})B_{p}(\varkappa N)^{p/2}}{2^{\frac{p}{2}}\frac{\Gamma(\frac{N+p}{2})}{\Gamma(\frac{N}{2})}} = \nu(\boldsymbol{\Omega}) c_{N,p}^{p} B_{p}\varkappa^{p/2}.
\end{split}
\end{equation*}
Thus,
\begin{equation*}
\begin{split}
\int_{{\mathbb S}^{N-1}}\|\boldsymbol{\chi}\|^\ast_{p}d\mu_{{\mathbb S}^{N-1}}(\boldsymbol{\chi}) \leq  \nu(\boldsymbol{\Omega})^{1/p} c_{N,p} (B_{p})^{\frac{1}{p}}\varkappa^{\frac{1}{2}}.
\end{split}
\end{equation*}
Lemma proved.
\end{proof}

\begin{lemma}\label{urysohn} Let $1\leq p< \infty$, $\frac{1}{p}+\frac{1}{p'}=1$, $f\in L_2(\nu)\cap L_p(\nu)$, $B = \{\boldsymbol{\xi}\in {\mathbb R}^S\mid \|\boldsymbol{\xi}\|^f_p\leq 1\}$, $B_2 = \{\boldsymbol{\xi}\in {\mathbb R}^S\mid \|\boldsymbol{\xi}\|_2\leq 1\}$ and $N=|S|$. Then,
\begin{equation*}
\begin{split}
\Big(\frac{{\rm vol}(B)}{{\rm vol}(B_{2})}\Big)^{\frac{1}{N}}\leq
\begin{cases}
\nu(\boldsymbol{\Omega})^{\frac{1}{2}-\frac{1}{p}}(B_{p'})^{\frac{1}{p'}}c_{N,p'}, & \text{if $p\geq 2$}\\
\frac{C_{K,p'}}{\langle \{\sigma_i\}_{i\in S}\rangle_{-2}}, & \text{if $1\leq p\leq 2$}\\
\end{cases}.
\end{split}
\end{equation*}

If $(\boldsymbol{\Omega}, \nu, K, \boldsymbol{\sigma}, \boldsymbol{\phi}, c)\in \mathcal{K}$,  $S=\{1, \cdots, N\}$, then
\begin{equation*}
\begin{split}
\Big(\frac{{\rm vol}(B)}{{\rm vol}(B_{2})}\Big)^{\frac{1}{N}}\leq \nu(\boldsymbol{\Omega})^{1/p'} c_{N,p'} (B_{p'})^{\frac{1}{p'}}\varkappa^{\frac{1}{2}},
\end{split}
\end{equation*}
for any $p\geq 1$.
\end{lemma}
\begin{proof} W.l.o.g. we can assume that $f\in {\rm span}(\{\phi_s|s\in S\})^\perp$. If this is not the case, we can always set $f' = f-\sum_{i\in S}\langle f, \phi_i\rangle_{L_2(\nu)}\phi_i$.  Since $B' = \{\boldsymbol{\xi}\in {\mathbb R}^S\mid \|\boldsymbol{\xi}\|^{f'}_p\leq 1\}$ is a translation of $B$, we have ${\rm vol}(B')={\rm vol}(B)$. Thus, the statement of Lemma for $f$ follows from the statement in which we change $f$ to $f'$.

By Urysohn's inequality for the volume ratio of convex bodies, we have
\begin{equation*}
\begin{split}
\Big(\frac{{\rm vol}(B)}{{\rm vol}(B_{2})}\Big)^{\frac{1}{N}}\leq \frac{1}{2}\int_{{\mathbb S}^{N-1}}(\|\boldsymbol{\chi}\|_{B^o}+\|-\boldsymbol{\chi}\|_{B^o})d\mu_{{\mathbb S}^{N-1}}(\boldsymbol{\chi}),
\end{split}
\end{equation*}
where $\|\boldsymbol{\chi}\|_{B^o}=\sup\{\langle \boldsymbol{\xi}, \boldsymbol{\chi}\rangle \mid \boldsymbol{\xi}\in B\}$.

By H{\"o}lder's inequality,
\begin{equation*}
\begin{split}
\sum_{i\in S}\xi_i \chi_i = \int_{\boldsymbol{\Omega}}(f({\mathbf x})+\boldsymbol{\phi}^{\boldsymbol{\xi}}_S({\mathbf x}))\boldsymbol{\phi}^{\boldsymbol{\chi}}_S({\mathbf x}) d\nu \leq \|f+\boldsymbol{\phi}^{\boldsymbol{\xi}}_S\|_{L_p(\nu)} \|\boldsymbol{\phi}^{\boldsymbol{\chi}}_S\|_{L_{p'}(\nu)}=\|\boldsymbol{\xi}\|^f_{p}\|\boldsymbol{\chi}\|^\ast_{p'}.
\end{split}
\end{equation*}
Therefore, we have $\|\boldsymbol{\chi}\|_{B^o}=\sup\{\langle \boldsymbol{\xi}, \boldsymbol{\chi}\rangle \mid \|\boldsymbol{\xi}\|^f_p=1\}\leq \|\boldsymbol{\chi}\|^\ast_{p'}$. Thus,
\begin{equation*}
\begin{split}
\Big(\frac{{\rm vol}(B)}{{\rm vol}(B_{2})}\Big)^{\frac{1}{N}}\leq \int_{{\mathbb S}^{N-1}}\|\boldsymbol{\chi}\|^\ast_{p'}d\mu_{{\mathbb S}^{N-1}}(\boldsymbol{\chi}).
\end{split}
\end{equation*}
Using Lemma~\ref{pisier} we obtain
\begin{equation*}
\begin{split}
&\Big(\frac{{\rm vol}(B)}{{\rm vol}(B_{2})}\Big)^{\frac{1}{N}}\leq
\begin{cases}
\nu(\boldsymbol{\Omega})^{\frac{1}{p'}-\frac{1}{2}}(B_{p'})^{\frac{1}{p'}}c_{N,p'}	, & \text{if $1\leq p'\leq 2$}\\
\frac{C_{K,p'}}{\langle \{\sigma_i\}_{i\in S}\rangle_{-2}}, & \text{if $p'\geq 2$}\\
\end{cases} = \\
&\begin{cases}
\nu(\boldsymbol{\Omega})^{\frac{1}{2}-\frac{1}{p}}(B_{p'})^{\frac{1}{p'}}c_{N,p'}	, & \text{if $p\geq 2$}\\
\frac{C_{K,p'}}{\langle \{\sigma_i\}_{i\in S}\rangle_{-2}}, & \text{if $1\leq p\leq 2$}\\
\end{cases}.
\end{split}
\end{equation*}
Analogously, if $p\geq 1$, $(\boldsymbol{\Omega}, \nu, K, \boldsymbol{\sigma}, \boldsymbol{\phi}, c)\in \mathcal{K}$ and $S=[N]$, then Lemma~\ref{pisier} gives
\begin{equation*}
\begin{split}
\Big(\frac{{\rm vol}(B)}{{\rm vol}(B_{2})}\Big)^{\frac{1}{N}}\leq \nu(\boldsymbol{\Omega})^{1/p'} c_{N,p'} (B_{p'})^{\frac{1}{p'}}\varkappa^{\frac{1}{2}}.
\end{split}
\end{equation*}
Lemma proved.
\end{proof}

Before we start the proof of the main lower bound, let us introduce a mapping $\Psi_S: L_2(\nu)[S]\to {\mathbb R}^S$, for $S\subseteq {\mathbb N}$, by
\begin{equation}
\Psi_S[f]_i = \langle f, \phi_i\rangle_{L_2(\nu)}, i\in S.
\end{equation}

\begin{proof}[Proof of Theorem~\ref{lower-kushpel}] Let $S\subseteq {\mathbb N}$ be finite and $N = |S|$.

Let us first assume that $p\geq 2$. Let $B_{\mathcal{H}_K}\subseteq \bigcup_{i=1}^c (f_i+\varepsilon B_{L_p(\nu)})$, $f_i\in L_p(\nu)$ and $c= e^{\mathcal{H}^{\rm ext}(\varepsilon, B_{\mathcal{H}_K}, L_p(\nu))}$.
We have
\begin{equation*}
\begin{split}
B_{\mathcal{H}_K}[S]\subseteq \big(\bigcup_{i=1}^c (f_i+\varepsilon B_{L_p(\nu)})\big)[S] = \bigcup_{i=1}^c (f_i+\varepsilon B_{L_p(\nu)})[S].
\end{split}
\end{equation*}
From $\Psi_S(B_{\mathcal{H}_K}[S])\subseteq \bigcup_{i=1}^c \Psi_S((f_i+\varepsilon B_{L_p(\nu)})[S])$, we conclude the following lower bound on $c$:
\begin{equation*}
\begin{split}
c\geq \frac{{\rm vol}(\Psi_S(B_{\mathcal{H}_K}[S]))}{\min_i{\rm vol}(\Psi_S((f_i+\varepsilon B_{L_p(\nu)})[S]))}.
\end{split}
\end{equation*}
Note that $\sum_{j\in S}\xi_j\phi_j\in (f_i+\varepsilon B_{L_p(\nu)})[S]$ if and only if $\|-f_i + \sum_{j\in S}\xi_j\phi_j\|_{L_p(\nu)}\leq \varepsilon$, or $\|[\xi_j]_{j\in S}\|_p^{-f_i}\leq \varepsilon$. In other words,
\begin{equation*}
\begin{split}
{\rm vol}(\Psi_S((f_i+\varepsilon B_{L_p(\nu)})[S])) = \varepsilon^N {\rm vol}\big(\{\boldsymbol{\xi}\in {\mathbb R}^S \mid \|\boldsymbol{\xi}\|_p^{-f_i/\varepsilon}\leq 1\}\big) .
\end{split}
\end{equation*}
Since $L_p(\nu)\subseteq L_2(\nu)$, we have $-f_i/\varepsilon\in L_p(\nu)\cap L_2(\nu)$ and conditions of Lemma~\ref{urysohn} are satisfied.
Using Lemma~\ref{urysohn} we conclude
\begin{equation*}
\begin{split}
e^{\mathcal{H}^{\rm ext}(\varepsilon, B_{\mathcal{H}_K}, L_p(\nu))}\geq \frac{\prod_{i\in S}\sigma_i {\rm vol}(B_{2})}{\varepsilon^{N}  (c_{N,p'}\nu(\boldsymbol{\Omega})^{1/2-1/p}B_{p'}^{1/p'})^N{\rm vol}(B_{2})}.
\end{split}
\end{equation*}
Finally, we define $S=\{i\in {\mathbb N}\mid \sigma_i>\varepsilon  \nu(\boldsymbol{\Omega})^{\frac{1}{2}-\frac{1}{p}} (B_{p'})^{\frac{1}{p'}}\}$ and obtain
\begin{equation*}
\begin{split}
&\mathcal{H}^{\rm ext}(\varepsilon, B_{\mathcal{H}_K}, L_p(\nu))\geq \\
&\sum_{i: \sigma_i>\varepsilon\nu(\boldsymbol{\Omega})^{1/2-1/p} (B_{p'})^{1/p'}}\log(\frac{\sigma_i}{\varepsilon  c_{N,p'}\nu(\boldsymbol{\Omega})^{1/2-1/p}(B_{p'})^{1/p'}}) = \\
&\sum_{i=1}^\infty \log_{+} (\frac{\sigma_i}{C\varepsilon})-N\log c_{N,p'},
\end{split}
\end{equation*}
where $C = \nu(\boldsymbol{\Omega})^{\frac{1}{2}-\frac{1}{p}} (B_{p'})^{\frac{1}{p'}}$. It remains to show that $N\log c_{N,p'}<0.7$.

Let $\psi(x) = \frac{\Gamma'(x)}{\Gamma(x)}$ be a digamma function. Then, from log-convexity of $\Gamma(x)$ we obtain
\begin{equation*}
\begin{split}
\log\Gamma(\frac{N+p'}{2})-\log\Gamma(\frac{N}{2})\geq \frac{p'}{2}\frac{\Gamma'(\frac{N}{2})}{\Gamma(\frac{N}{2})}=\frac{p'}{2}\psi(\frac{N}{2}).
\end{split}
\end{equation*}
Therefore,
\begin{equation*}
\begin{split}
&N\log c_{N,p'} = N\big(\frac{1}{2}\log(\frac{N}{2})+ \frac{1}{p'}(\log\Gamma(\frac{N}{2})-\log\Gamma(\frac{N+p'}{2}))\big)\leq \\
&\frac{N}{2}\big(\log(\frac{N}{2})-\psi(\frac{N}{2})\big)<0.7,
\end{split}
\end{equation*}
where we used the asymptotic behaviour $\psi(z) \sim \log{z} - \frac{1}{2z}$ and a plotting of $z(\log{z}-\psi(z))$ for a positive argument $z$.
From this we obtain the first inequality
\begin{equation*}
\begin{split}
\mathcal{H}^{\rm ext}(\varepsilon, B_{\mathcal{H}_K}, L_p(\nu))\geq
\sum_{i=1}^\infty \log_{+} (\frac{\sigma_i}{C\varepsilon})-0.7.
\end{split}
\end{equation*}

The case of $1\leq p\leq 2$ differs from the previous one by the fact that  $L_p(\nu)\not\subseteq L_2(\nu)$ and $f_i$ is not necesarily in $L_2(\nu)$ (i.e. $f_i$ does not satisfy conditions of Lemma~\ref{urysohn}). Therefore, there we need to assume that $B_{\mathcal{H}_K}\subseteq \bigcup_{i=1}^c (f_i+\varepsilon B_{L_p(\nu)})$ is an internal covering, i.e. $f_i\in B_{\mathcal{H}_K}\subseteq L_p(\nu)\cap L_2(\nu)$. Let us now set $S = [N]$ for some $N\in {\mathbb N}$. Then, analogous arguments based on Lemma~\ref{urysohn}   leads us to
\begin{equation*}
\begin{split}
e^{\mathcal{H}^{\rm int}(\varepsilon, B_{\mathcal{H}_K}, L_p(\nu))}\geq \frac{{\rm vol}(\Psi_S(B_{\mathcal{H}_K}[S]))}{\min_i{\rm vol}(\Psi_S((f_i+\varepsilon B_{L_p(\nu)})[S]))} \geq \frac{\prod_{i\in S}\sigma_i {\rm vol}(B_{2})}{\varepsilon^{N}  \big(\frac{C_{K,p'}}{\langle \{\sigma_i\}_{i\in S}\rangle_{-2}}\big)^N{\rm vol}(B_{2})} ,
\end{split}
\end{equation*}
and, therefore, to
\begin{equation*}
\begin{split}
\mathcal{H}^{\rm int}(\varepsilon, B_{\mathcal{H}_K}, L_p(\nu))\geq \sup_{N\in {\mathbb N}}\sum_{i=1}^N \log(\frac{\sigma_i}{C_{K,p'}\varepsilon})+N\log \langle \{\sigma_i\}_{i\in [N]}\rangle_{-2}.
\end{split}
\end{equation*}

If, additionally, $(\boldsymbol{\Omega}, \nu, K, \boldsymbol{\sigma}, \boldsymbol{\phi}, \varkappa)\in \mathcal{K}$, then for $p\geq 2$ by Lemma~\ref{urysohn} we have
\begin{equation*}
\begin{split}
e^{\mathcal{H}^{\rm ext}(\varepsilon, B_{\mathcal{H}_K}, L_p(\nu))}\geq \frac{\prod_{i\in S}\sigma_i {\rm vol}(B_{2})}{\varepsilon^{N} c^N_{N,p'} \big(\nu(\boldsymbol{\Omega})^{1/p'}(B_{p'})^{\frac{1}{p'}}\varkappa^{\frac{1}{2}}\big)^N{\rm vol}(B_{2})}.
\end{split}
\end{equation*}
After defining $S=\{i\in {\mathbb N}\mid \sigma_i>\varepsilon  (B_{p'})^{\frac{1}{p'}}\varkappa^{\frac{1}{2}}\}$, we obtain
\begin{equation*}
\begin{split}
\mathcal{H}^{\rm ext}(\varepsilon, B_{\mathcal{H}_K}, L_p(\nu))\geq \sum_{i=1}^\infty \log_+(\frac{\sigma_i}{C_1\varepsilon})-0.7,
\end{split}
\end{equation*}
where $C_1 = \nu(\boldsymbol{\Omega})^{1/p'}(B_{p'})^{\frac{1}{p'}}\varkappa^{1/2}$.
\end{proof}
For $(\boldsymbol{\Omega}, \nu, K, \boldsymbol{\sigma}, \boldsymbol{\phi}, \varkappa)\in \mathcal{K}$ and $1\leq p\leq 2$, we have
\begin{equation*}
\begin{split}
\mathcal{H}^{\rm int}(\varepsilon, B_{\mathcal{H}_K}, L_p(\nu))\geq \frac{\prod_{i\in S}\sigma_i {\rm vol}(B_{2})}{\varepsilon^{N} c^N_{N,p'} \big(\nu(\boldsymbol{\Omega})^{1/p'}(B_{p'})^{\frac{1}{p'}}\varkappa^{\frac{1}{2}}\big)^N{\rm vol}(B_{2})} ,
\end{split}
\end{equation*}
and, therefore,
\begin{equation*}
\begin{split}
\mathcal{H}^{\rm int}(\varepsilon, B_{\mathcal{H}_K}, L_p(\nu))\geq \sum_{i=1}^\infty \log_+(\frac{\sigma_i}{C_1\varepsilon})-0.7.
\end{split}
\end{equation*}

\begin{appendices}
\ifCOM
\section{Table of asymptotics of $\varepsilon$-entropy}

\begin{tabular}{ |p{3cm}|p{3cm}|p{3cm}|p{3cm}| p{3cm}|  }
\hline
\multicolumn{4}{|c|}{Asymptotics of $\varepsilon$-entropy} \\
\hline
Kernel $K({\mathbf x},{\mathbf y})$ & Domain & $\mathcal{H}(\varepsilon, B_{\mathcal{H}_{K}}, L_2(\mu_{\boldsymbol{\Omega}}))$  & $\mathcal{H}(\varepsilon, B_{\mathcal{H}_{K}}, C(\boldsymbol{\Omega}))$  \\
\hline
$e^{-\sigma^2\|{\mathbf x}-{\mathbf y}\|^2}$ & ${\mathbb S}^{n-1}$ & $\mathcal{O}\big(\frac{(\log\frac{1}{\varepsilon})^{n}}{(\log\log\frac{1}{\varepsilon})^{n-1}}\big)$ &  $\mathcal{O}\big(\frac{(\log\frac{1}{\varepsilon})^{n}}{(\log\log\frac{1}{\varepsilon})^{n-1}}\big)$ \\
\hline
$e^{-\sigma^2\|{\mathbf x}-{\mathbf y}\|^2}$ & $[-1,1]^{n}$ & $\mathcal{O}\big(\frac{(\log\frac{1}{\varepsilon})^{n+1}}{(\log\log\frac{1}{\varepsilon})^{n}}\big)$ &  $\mathcal{O}\big(\frac{(\log\frac{1}{\varepsilon})^{n+1}}{(\log\log\frac{1}{\varepsilon})^{n}}\big)^\ast$ \\
\hline
$e^{-2\pi\sigma\|{\mathbf x}-{\mathbf y}\|}$ & ${\mathbb S}^{n-1}$ & $\mathcal{O}(\frac{1}{\varepsilon^{2(n-2)/n}} \log \frac{1}{\varepsilon})$ &  $\mathcal{O}(\frac{\log\frac{1}{\varepsilon}}{\varepsilon})$ \\
\hline
\end{tabular}
\else
\fi

\section{Growth rates of $\mathcal{E}$ and $m_\varepsilon$}\label{growth-rates}
Let us assume that $\sigma_1\geq \sigma_2\geq \cdots $.
Usually, the asymptotic behavior of the $\varepsilon$-entropy is carefully studied when $\varepsilon\to +0$~\cite{TikhomirovKolmogorov}. 

\begin{lemma}\label{decay} Suppose that $\sigma_i \leq \frac{C}{i^{\gamma}(\log (i+1))^\zeta}$ for $\gamma>0, \zeta\geq 0$. Then, for $\varepsilon>0$, we have
\begin{equation*}
\begin{split}
\mathcal{E}(\varepsilon, \{\sigma_i\}_{i=1}^\infty) \leq \zeta (N_\varepsilon+1)\log(\log(N_\varepsilon+2))+\gamma(N_\varepsilon+1) -\zeta \log\log 2,
\end{split}
\end{equation*}
where $N_\varepsilon>0$ is the solution of the equation $\varepsilon = \frac{C}{x^{\gamma}(\log (x+1))^\zeta}$, $x>0$. Also,
\begin{equation*}
\begin{split}
m_\varepsilon\leq N_\varepsilon = \frac{\gamma^{\zeta/\gamma}(\frac{C}{\varepsilon})^{1/\gamma}}{(\log (\frac{C}{\varepsilon}))^{\zeta/\gamma}}(1+o(1)).
\end{split}
\end{equation*}

\end{lemma}
\begin{proof} Let us define $\varepsilon_N = \frac{C}{N^{\gamma}(\log (N+1))^\zeta}$. Then,
\begin{equation*}
\begin{split}
&\frac{1}{N}\mathcal{E}(\varepsilon_N, \{\sigma_i\}_{i=1}^\infty) = \frac{1}{N}\mathcal{E}(\varepsilon_N, \{\sigma_i\}_{i=1}^N)\leq
\frac{1}{N}\sum_{i=1}^N \log(\frac{C}{\varepsilon_N i^{\gamma}(\log (i+1))^\zeta}) = \\
&\frac{1}{N} \sum_{i=1}^N (\log(\frac{C}{\varepsilon_N})-\gamma \log i-\zeta \log(\log(i+1))) \leq \\
&\frac{1}{N}\int_{0}^N (\log(\frac{C}{\varepsilon_N})-\gamma \log x) dx -\frac{\zeta \log\log 2}{N}= \\
&\frac{1}{N}x(\log(\frac{C}{\varepsilon_N})-\gamma \log x)|_0^N +
\frac{1}{N}\int_{0}^N x \frac{\gamma}{x} dx -\frac{\zeta \log\log 2}{N}=\\
&(\log(\frac{C}{\varepsilon_N})-\gamma \log N)+\gamma -\frac{\zeta \log\log 2}{N} =\zeta\log(\log(N+1))+\gamma -\frac{\zeta \log\log 2}{N}.
\end{split}
\end{equation*}
Thus, we have $$\mathcal{E}( \frac{C}{N^{\gamma}(\log(N+1))^\zeta}, \{\sigma_i\}_{i=1}^\infty) \leq
N(\zeta\log(\log(N+1))+\gamma-\frac{\zeta \log\log 2}{N}),$$
and, therefore, for $\varepsilon>0$, we have
\begin{equation*}
\begin{split}
\mathcal{E}( \varepsilon, \{\sigma_i\}_{i=1}^\infty) =  \mathcal{E}( \frac{C}{N_\varepsilon^{\gamma}(\log(N_\varepsilon+1))^\zeta}, \{\sigma_i\}_{i=1}^\infty)  \leq
\mathcal{E}( \frac{C}{\lceil N_\varepsilon\rceil^{\gamma} (\log(\lceil N_\varepsilon\rceil +1))^\zeta}, \{\sigma_i\}_{i=1}^\infty)\leq \\
\zeta (N_\varepsilon+1)\log(\log(N_\varepsilon+2))+\gamma(N_\varepsilon+1) -\zeta \log\log 2.
\end{split}
\end{equation*}
Let us now estimate $N_\varepsilon$. Let $g(x) = x(\log (x+1))^{\zeta/\gamma}$, then $N_\varepsilon = g^{-1}((\frac{C}{\varepsilon})^{1/\gamma})$. We have $g^{-1}(x)= \frac{x}{(\log x)^{\zeta/\gamma}}(1+o(1))$, due to $g(\frac{x}{(\log x)^{\zeta/\gamma}}) = x+o(1)$. Therefore, $N_\varepsilon= \frac{(\frac{C}{\varepsilon})^{1/\gamma}}{(\log (\frac{C}{\varepsilon})/\gamma)^{\zeta/\gamma}}(1+o(1))$.
Lemma proved.
\end{proof}

\begin{lemma}\label{exp-decay} Suppose that $\sigma_i \leq Ce^{-\gamma i^r}$ for $C>0,\gamma>0, r>0$. Then, for $\varepsilon>0$, we have
\begin{equation*}
\begin{split}
m_\varepsilon \leq \big(\frac{\log(C/\varepsilon)}{\gamma}\big)^{1/r}+1,\\
\mathcal{E}(\varepsilon, \{\sigma_i\}_{i=1}^\infty) \leq \frac{\gamma r (\big(\frac{\log(C/\varepsilon)}{\gamma}\big)^{1/r}+1)^r}{r+1}.
\end{split}
\end{equation*}
\end{lemma}
\begin{proof} Let us define $\varepsilon_N = Ce^{-\gamma N^r}$. Then,
\begin{equation*}
\begin{split}
&\frac{1}{N}\mathcal{E}(\varepsilon_N, \{\sigma_i\}_{i=1}^\infty) = \frac{1}{N}\mathcal{E}(\varepsilon_N, \{\sigma_i\}_{i=1}^N)\leq
\frac{1}{N}\sum_{i=1}^N \log(\frac{Ce^{-\gamma i^r}}{\varepsilon_N}) = \\
&\frac{1}{N} \sum_{i=1}^N (\log(\frac{C}{\varepsilon_N})-\gamma i^r) \leq \frac{1}{N}\int_{0}^N (\log(\frac{C}{\varepsilon_N})-\gamma x^r) dx= \\
&\log(\frac{C}{\varepsilon_N})- \frac{\gamma N^r}{r+1} =  \frac{\gamma r N^r}{r+1}.
\end{split}
\end{equation*}
Thus, we have $$\mathcal{E}( Ce^{-\gamma N^r}, \{\sigma_i\}_{i=1}^\infty) \leq
\frac{\gamma r N^r}{r+1}.$$
Let us set $N_\varepsilon = \big(\frac{\log(C/\varepsilon)}{\gamma}\big)^{1/r}$.
For $\varepsilon>0$, we have
\begin{equation*}
\begin{split}
\mathcal{E}( \varepsilon, \{\sigma_i\}_{i=1}^\infty) =  \mathcal{E}(Ce^{-\gamma N_\varepsilon^r}, \{\sigma_i\}_{i=1}^\infty)  \leq
\mathcal{E}( Ce^{-\gamma \lceil N_\varepsilon\rceil^r}, \{\sigma_i\}_{i=1}^\infty)\leq \\
\frac{\gamma r (\big(\frac{\log(C/\varepsilon)}{\gamma}\big)^{1/r}+1)^r}{r+1}.
\end{split}
\end{equation*}
Lemma proved.
\end{proof}

\section{Twice differentiable kernels and the property~\eqref{property-k}}\label{smooth-kern}
\ifCOM
The following corollary of Theorem~\ref{Bound-main} directly follows from the previous lemma.
\begin{corollary}\label{decay} Suppose that $\lambda_i \leq \frac{C_1}{i^{\gamma}}$ for $\gamma>0$. Then, for $p\in [1,2]$, we have
$$
\mathcal{H}(\varepsilon, B_{\mathcal{H}_K}, L_p(\nu)) =  {\mathcal O}\big(\frac{1}{\varepsilon^{2/\gamma}}\big).
$$
\end{corollary}
\begin{proof} We have $\sigma_i\leq \sqrt{\frac{C_1}{i^{\gamma}}}$ and $m_{\alpha} \leq |\{i\in {\mathbb N}\mid \sqrt{\frac{C_1}{i^{\gamma}}}>\alpha\}|\leq \big(\frac{C_1}{\alpha^2}\big)^{1/\gamma}$.
Then, using Theorem~\ref{Bound-main} and the previous Lemma, we obtain:
\begin{equation*}
\begin{split}
&\mathcal{H}(\varepsilon, B_{\mathcal{H}_K},  L_p(\nu))\leq 
\mathcal{E}(\varepsilon, \{\sigma_i\}_{i=1}^{\infty}) +  m_{(1-\theta)\varepsilon} \log\frac{3}{\theta(2-\theta)}\leq \\
&\mathcal{O}(\frac{1}{\varepsilon^{2/\gamma}})+ \big(\frac{C_1}{(1-\theta)^2\varepsilon^2}\big)^{1/\gamma} \log\frac{3}{\theta(2-\theta)}= \mathcal{O}(\frac{1}{\varepsilon^{2/\gamma}}).
\end{split}
\end{equation*}
\end{proof}
To apply Theorem~\ref{Bound-main}, for $p>2$, we need more information about the kernel. A natural such additional assumption is the differentiability of the kernel $K$.
It is well-known that eigenvalues of such a kernel have a polynomial decay rate, and therefore, $C_{K,+\infty}(N)$ can be bounded.
\begin{theorem} If (a) $\boldsymbol{\Omega}$ is compact and has a piecewise smooth boundary, (b)  for any fixed ${\mathbf y}\in \boldsymbol{\Omega}$, $\frac{\partial^{|\alpha|}K({\mathbf x}, {\mathbf y})}{\partial {\mathbf x}^\alpha}\in C(\boldsymbol{\Omega}^2)$ for any $\alpha\in ({\mathbb N}\cup \{0\})^n$ such that $|\alpha| = \sum_{i=1}^n\alpha_i\leq s$. Then, for any $p\in [0,1]$,
$$
\mathcal{H}(\varepsilon, B_{\mathcal{H}_K},  L_p(\boldsymbol{\Omega})) = {\mathcal O}\big(\frac{1}{\varepsilon^{2n/s}}\big).
$$
\end{theorem}
\begin{proof} Let $W^s_1(\boldsymbol{\Omega})$ be a Sobolev space, i.e. a set of functions whose generalized derivatives up to degree $s$ are in $L_1(\boldsymbol{\Omega})$.
Let us introduce an operator ${\rm O}^1_K: L_1(\boldsymbol{\Omega})\to W^s_1(\boldsymbol{\Omega})$ by
$$
{\rm O}^1_K[f]({\mathbf x}) = \int_{\boldsymbol{\Omega}} K({\mathbf x}, {\mathbf y})f({\mathbf y})d{\mathbf y}
$$
The operator is well-defined and continuous because all conditions of the Leibniz rule are satisfied for the following differentiation:
\begin{equation*}
\begin{split}
|\frac{\partial^{\alpha}}{\partial {\mathbf x}^\alpha}\int K({\mathbf x}, {\mathbf y})f({\mathbf y})d{\mathbf y}| = |\int \frac{\partial^{\alpha}K({\mathbf x}, {\mathbf y})}{\partial {\mathbf x}^\alpha}f({\mathbf y})d{\mathbf y}| \leq 
\sup_{{\mathbf x}, {\mathbf y}}\frac{\partial^{\alpha}K({\mathbf x}, {\mathbf y})}{\partial {\mathbf x}^\alpha} \|f\|_{L_1(\boldsymbol{\Omega})}.
\end{split}
\end{equation*}
Since $\boldsymbol{\Omega}$ is compact, boundedness implies Lebesgue integrability.

Let ${\rm Id}: W^s_1(\boldsymbol{\Omega})\to L_1(\boldsymbol{\Omega})$ be a standard embedding of the Sobolev space $W^s_1(\boldsymbol{\Omega})$ into $L_1(\boldsymbol{\Omega})$. 
From Proposition 3.c.10 of the textbook of~\cite{konig1986eigenvalue} we conclude that ${\rm Id}\circ {\rm O}^1_K: L_1(\boldsymbol{\Omega})\to L_1(\boldsymbol{\Omega})$ is a Riesz operator (in fact, it can be shown that it is compact) and eigenvalues of ${\rm Id}\circ {\rm O}^1_K$ satisfy:
$$
\lambda_i({\rm O}^1_K) = {\mathcal O}(i^{-s/n})
$$
where $\lambda_1({\rm O}^1_K)\geq \lambda_2({\rm O}^1_K)\geq ...$ are ordered eigenvectors of ${\rm O}^1_K$.
By construction, for any $f\in L_1(\boldsymbol{\Omega})$, ${\rm O}^1_K[f]\in C(\boldsymbol{\Omega})$. Therefore, any eigenvector of ${\rm O}^1_K$ is a continuous function and is square integrable. Thus, all eigenvectors of ${\rm O}^1_K$ are also eigenvectors of ${\rm O}_K$. The opposite is also true. Therefore, $\lambda_i=\lambda_i({\rm O}^1_K) = {\mathcal O}(i^{-s/n})$.

Finally, using Lemma~\ref{decay}, we conclude:
$$
\mathcal{H}(\varepsilon,B_{\mathcal{H}_K},  L_2(\boldsymbol{\Omega})) = {\mathcal O}\big(\frac{1}{\varepsilon^{2/(s/n)}}\big).
$$
\end{proof}
\else
\fi

\begin{theorem} Let $\frac{\partial^2 K({\mathbf u}, {\mathbf v})}{\partial x_i \partial y_j}\in C(\boldsymbol{\Omega}^2)$. Then,
$$\sqrt{K({\mathbf x}, {\mathbf x})+K({\mathbf y}, {\mathbf y})-2K({\mathbf x}, {\mathbf y})}\leq C\|{\mathbf x}- {\mathbf y}\|,$$
where $C>0$ is constant.
\end{theorem}
\begin{proof}
First, note that
\begin{equation*}
\begin{split}
&K({\mathbf u}, {\mathbf u})+K({\mathbf v}, {\mathbf v})-2K({\mathbf u}, {\mathbf v}) =(K({\mathbf u}, {\mathbf u}) -K({\mathbf u}, {\mathbf v}) ) -(K({\mathbf v}, {\mathbf u}) -K({\mathbf v}, {\mathbf v}) ) =\\
&\int_0^1 \sum_{i=1}^n \frac{\partial K({\mathbf u}, {\mathbf v}+t({\mathbf u}-{\mathbf v}))}{\partial y_i} (u_i-v_i)dt-\int_0^1 \sum_{i=1}^n \frac{\partial K({\mathbf v}, {\mathbf v}+t({\mathbf u}-{\mathbf v}))}{\partial y_i} (u_i-v_i)dt = \\
&\int_{[0,1]^2} \sum_{i,j=1}^n \frac{\partial^2 K({\mathbf v}+s({\mathbf u}-{\mathbf v}), {\mathbf v}+t({\mathbf u}-{\mathbf v}))}{\partial x_j \partial y_i} (u_i-v_i)(u_j-v_j)dtds = \\
&\sum_{i,j=1}^n K_{ij} (u_i-v_i)(u_j-v_j),
\end{split}
\end{equation*}
where $K_{ij} = \int_{[0,1]^2} \frac{\partial^2 K({\mathbf v}+s({\mathbf u}-{\mathbf v}), {\mathbf v}+t({\mathbf u}-{\mathbf v}))}{\partial x_j \partial y_i}dtds$. Let us denote $M({\mathbf u}, {\mathbf v}) = [K_{ij}]_{i,j=1}^n$. The kernel $K$ satisfies the property~\eqref{property-k} from Remark~\ref{non-exotic} with $\alpha=1$ and $C = \sqrt{\max_{{\mathbf u}, {\mathbf v}\in \boldsymbol{\Omega}} \|M({\mathbf u}, {\mathbf v})\|}$ where $\|\cdot\|$ is the operator norm, due to $C<+\infty$ and
\begin{equation*}
\begin{split}
\sqrt{\sum_{i,j=1}^n K_{ij} (u_i-v_i)(u_j-v_j)}\leq C\|{\mathbf u}- {\mathbf v}\|.
\end{split}
\end{equation*}
\end{proof}

\section{Application I: zonal kernels}\label{zonal-section}
A Mercer kernel $K: {\mathbb S}^{n-1}\times {\mathbb S}^{n-1}\to {\mathbb R}$ is called zonal if $K({\mathbf x}, {\mathbf y}) = t({\mathbf x}^\top {\mathbf y})$. Entropy numbers of zonal kernels were studied in~\cite{Kushpel2012}. Examples of zonal kernels include Laplace kernel $K({\mathbf x},{\mathbf y}) = e^{-2\pi\sigma\|{\mathbf x}-{\mathbf y}\|}$ and the Neural Tangent Kernels on an $n-1$-sphere. It is well-known that zonal kernels satisfy the following expansion:
\begin{equation*}
\begin{split}
K({\mathbf x}, {\mathbf y}) = \sum_{l=0}^\infty a(l)\sum_{i=1}^{\mathcal{N}(n,l)}Y_{l,i}({\mathbf x})Y_{l,i}({\mathbf y})
\end{split}
\end{equation*}
where $\{Y_{l,i}\}_{i=1}^{\mathcal{N}(n,l)}$ is an orthonormal basis in the space of real-valued spherical harmonics of order $l$ on ${\mathbb S}^{n-1}$ and $\mathcal{N}(n,l) = {n+l-1 \choose n-1}-{n+l-3 \choose n-1}$ is the dimension of that space. Note that $\bigcup_{l=0}^\infty\{Y_{l,i}\}_{i=1}^{\mathcal{N}(n,l)}$ form an orthonormal basis in $L_2(\mu_{{\mathbb S}^{n-1}})$ and coefficients $\{a(l)\}_{l=0}^\infty$ are eigenvalues of the integral operator ${\rm O}_K: L_2(\mu_{{\mathbb S}^{n-1}})\to L_2(\mu_{{\mathbb S}^{n-1}})$, ${\rm O}_K\phi({\mathbf x}) = \int_{{\mathbb S}^{n-1}}K({\mathbf x},{\mathbf y})\phi({\mathbf y})d\mu_{{\mathbb S}^{n-1}}({\mathbf y})$. The latter expansion can be viewed as a spectral decomposition of ${\rm O}_K: L_2(\mu_{{\mathbb S}^{n-1}})\to L_2(\mu_{{\mathbb S}^{n-1}})$ in Mercer's theorem.

For a fixed $n$, let us denote
\begin{equation*}
\begin{split}
N_L=\sum_{l=0}^{L}\mathcal{N}(n,l).
\end{split}
\end{equation*}

\ifCOM
\begin{lemma} For any integers $0\leq L_1<L_2$, we have
\begin{equation*}
\begin{split}
C_{K,p}(N_{L_1}, N_{L_2}) = \sqrt{\sum_{l=L_1}^{L_2}a(l)\mathcal{N}(n,l)}.
\end{split}
\end{equation*}
\end{lemma}
\begin{proof}
It is well-known that $\sum_{i=1}^{\mathcal{N}(n,l)}|Y_{l,i}({\mathbf x})|^2 = \mathcal{N}(n,l)$. Therefore,
\begin{equation*}
\begin{split}
\sum_{l=L_1}^{L_2}\sum_{i=1}^{\mathcal{N}(n,l)} a(l) Y_{l,i}({\mathbf x})^2 =  \sum_{l=L_1}^{L_2}a(l)\mathcal{N}(n,l).
\end{split}
\end{equation*}
Therefore, $C_{K,p}(N_{L_1}, N_{L_2}) = \|\sqrt{\sum_{l=L_1}^{L_2}a(l)\mathcal{N}(n,l)}\|_{L_p(\mu_{{\mathbb S}^{n-1}})} = \sqrt{\sum_{l=L_1}^{L_2}a(l)\mathcal{N}(n,l)}$.
\end{proof}

\begin{lemma} We have
\begin{equation*}
\begin{split}
C_{K,p}(N_{L_1}, N_{L_2})\leq \\
\begin{cases}
a& \text{if $a(l) = \frac{1}{l^r}$}\\
\end{cases}
\end{split}
\end{equation*}
\end{lemma}
\begin{proof}
From the previous lemma we obtain
\begin{equation*}
\begin{split}
C_{K,p}(N_{L_1}, N_{L_2}) = \sqrt{\sum_{l=L_1}^{L_2}\frac{\mathcal{N}(n,l)}{l^r}}
\end{split}
\end{equation*}

\end{proof}
\else
\fi

As the following lemma shows, besides kernels with uniformly bounded eigenvectors, zonal kernels satisfy the definition~\ref{kushpel-inspired}.

\begin{lemma}\label{zonal-bound} For any zonal kernel $K$ on ${\mathbb S}^{n-1}$, we have $$({\mathbb S}^{n-1}, \mu_{{\mathbb S}^{n-1}}, K, \bigcup_{l=0}^\infty\{\sqrt{a(l)}\}_{i=1}^{\mathcal{N}(n,l)}, \bigcup_{l=0}^\infty\{Y_{l,i}\}_{i=1}^{\mathcal{N}(n,l)}, \varkappa)\in \mathcal{K},$$ where $\varkappa = \sup_{L\geq 0, L\in {\mathbb Z}}\frac{\sum_{l=0}^{L+1}\mathcal{N}(n,l)}{\sum_{l=0}^{L}\mathcal{N}(n,l)}$ is finite.
\end{lemma}
\begin{proof} It is well-known that $\sum_{i=1}^{\mathcal{N}(n,l)}|Y_{l,i}({\mathbf x})|^2 = \mathcal{N}(n,l)$.
Since eigenfunctions $\bigcup_{l=0}^\infty\{Y_{l,i}\}_{i=1}^{\mathcal{N}(n,l)}$ are listed in the order of increasing $l$, for an integer $L\geq 0$ and $j\in [\mathcal{N}(n,L+1)]$, we have
\begin{equation*}
\begin{split}
\sum_{l=0}^{L}\sum_{i=1}^{\mathcal{N}(n,l)}|Y_{l,i}({\mathbf x})|^2+\sum_{i=1}^{j}|Y_{L+1,i}({\mathbf x})|^2\leq \sum_{l=0}^{L}\mathcal{N}(n,l)+\mathcal{N}(n,L+1)\leq \\
(\sum_{l=0}^{L}\mathcal{N}(n,l)+j)\frac{\sum_{l=0}^{L}\mathcal{N}(n,l)+\mathcal{N}(n,L+1)}{\sum_{l=0}^{L}\mathcal{N}(n,l)}\leq (\sum_{l=0}^{L}\mathcal{N}(n,l)+j)\varkappa,
\end{split}
\end{equation*}
where $\varkappa = \sup_{L\geq 0, L\in {\mathbb Z}}\frac{\sum_{l=0}^{L+1}\mathcal{N}(n,l)}{\sum_{l=0}^{L}\mathcal{N}(n,l)}$. Since $\varkappa-1 = \sup_{L\geq 0, L\in {\mathbb Z}}\frac{{n+L \choose n-1}-{n+L-2 \choose n-1}}{{n+L-1 \choose n-1}+{n+L-2 \choose n-1}-n-1} = \mathcal{O}(\frac{n-1}{L+n})=\mathcal{O}(1)$, we obtain that $\varkappa$ is finite.
\end{proof}

\begin{lemma}\label{wolfram} Let $a>-1$. For $b\to +\infty$, we have $\int_1^b x^a\log(\frac{b}{x})dx \asymp b^{a+1}$.
\end{lemma}
\begin{proof}
We have $\int_1^b x^a\log(\frac{b}{x})dx=\frac{b^{a+1}-(a+1)\log b-1}{(a+1)^2}$ which behaves like $\frac{b^{a+1}}{(a+1)^2}$ as $b\to+\infty$.
\end{proof}

In the following theorem we simply apply our upper and lower bounds, given in Theorems~\ref{Bound-main} and~\ref{lower-kushpel}, to a special case of zonal kernels. Due to Lemma~\ref{zonal-bound}, we can use the second type of lower bounds from Theorem~\ref{lower-kushpel}, which implies that the lower bound behaves like $ \mathcal{E}(C'\varepsilon, \{\sigma_i\}_{i=1}^\infty)$.
\begin{theorem}
Let $K({\mathbf x},{\mathbf y})$ be a zonal kernel on ${\mathbb S}^{n-1}$. Suppose that coefficients $\{a(l)\}_{l=0}^\infty$ satisfy $a(l) \asymp \frac{1}{l^{\gamma}}$ where $\gamma>n-1$.

Then, for $\varepsilon\to +0$, we have
\begin{equation*}
\begin{split}
\mathcal{H}(\varepsilon, B_{\mathcal{H}_{K}},  L_p(\mu_{{\mathbb S}^{n-1}}))\asymp  \frac{1}{\varepsilon^{2(n-1)/\gamma}}, \text{ if $1\leq p\leq 2$,}\\
\frac{1}{\varepsilon^{2(n-1)/\gamma}}\ll \mathcal{H}(\varepsilon, B_{\mathcal{H}_{K}},  L_p(\mu_{{\mathbb S}^{n-1}}))\ll \frac{1}{\varepsilon^{2(n-1)/\gamma}}\log(\frac{1}{\varepsilon}) , \text{ if $p\in [2,+\infty)$}.
\end{split}
\end{equation*}
\end{theorem}
\begin{proof}
Let $\{\sigma_i\}$ be singular values of ${\rm O}_K$ ordered according to the increasing degree $l$. Note that $\sigma_i\leq C_K$ for any $i\in {\mathbb N}$.
Since $a(l)\asymp \frac{1}{l^\gamma}$, we have $\sigma_i\asymp \frac{1}{l^{\gamma/2}}$ for $N_{l-1}< i\leq N_l$. Let $L_\varepsilon$ be the largest integer $l$ such that $\sigma_{i}>\varepsilon$ for some $N_{l-1}< i\leq N_l$. Note that $L_\varepsilon \asymp \varepsilon^{-2/\gamma}$. Since $\mathcal{N}(n,l) = {n+l-1 \choose n-1}-{n+l-3 \choose n-1}$, we have $a_n l^{n-2}\leq \mathcal{N}(n,l) \leq b_n l^{n-2}$ for a sufficiently large $l$. Thus, there are constants $l_0\in {\mathbb N}$ and $C>0$ such that
\begin{equation*}
\begin{split}
&\mathcal{E}(\varepsilon, \{\sigma_i\}_{i=1}^\infty) = \sum_{i=1}^\infty \log_+(\frac{\sigma_i}{\varepsilon})\ll
\sum_{l=0}^{l_0-1}\mathcal{N}(n,l)\log_+(\frac{C_K}{\varepsilon})+\sum_{l=l_0}^{L_\varepsilon} \mathcal{N}(n,l)\log(\frac{C}{\varepsilon l^{\gamma/2}}) \asymp \\
&  \sum_{l=l_0}^{L_\varepsilon} \mathcal{N}(n,l)(\log(\frac{C}{\varepsilon})-\frac{\gamma}{2} \log l) \asymp    \sum_{l=l_0}^{L_\varepsilon} l^{n-2}(\log(\frac{C}{\varepsilon})-\frac{\gamma}{2} \log l)\asymp  \\
& \int_{1}^{(C/\varepsilon)^{2/\gamma}} x^{n-2}(\log(\frac{C}{\varepsilon})-\frac{\gamma}{2} \log x)dx.
\end{split}
\end{equation*}
From Lemma~\ref{wolfram} we obtain that the latter integral behaves like $(C/\varepsilon)^{2(n-1)/\gamma}$ as $\varepsilon\to +0$.
Analogously, there is a constant $C'>$ such that
\begin{equation*}
\begin{split}
&\mathcal{E}(\varepsilon, \{\sigma_i\}_{i=1}^\infty) = \sum_{i=1}^\infty \log_+(\frac{\sigma_i}{\varepsilon})\gg  \int_{1}^{(C'/\varepsilon)^{2/\gamma}} x^{n-2}(\log(\frac{C'}{\varepsilon})-\frac{\gamma}{2} \log x)dx \asymp  (C'/\varepsilon)^{2(n-1)/\gamma}.
\end{split}
\end{equation*}
Therefore,
\begin{equation*}
\begin{split}
\mathcal{E}(\varepsilon, \{\sigma_i\}_{i=1}^\infty)\asymp  \frac{1}{\varepsilon^{2(n-1)/\gamma}}.
\end{split}
\end{equation*}
Analogously, we prove $m_{(1-\theta)\varepsilon}\asymp  \frac{1}{\varepsilon^{2(n-1)/\gamma}}$.

Using Lemma~\ref{zonal-bound} and Theorem~\ref{lower-kushpel}, we obtain $\mathcal{H}(\varepsilon, B_{\mathcal{H}_{K}}, L_p(\mu_{{\mathbb S}^{n-1}}))\gg \frac{1}{\varepsilon^{2(n-1)/\gamma}}$ when $\varepsilon\to 0$ for any $p\geq 1$.

For $1\leq p\leq 2$, from Theorem~\ref{Bound-main} we directly obtain the asymptotics $\mathcal{H}(\varepsilon, B_{\mathcal{H}_{K}}, L_p(\mu_{{\mathbb S}^{n-1}}))\ll \frac{1}{\varepsilon^{2(n-1)/\gamma}}$ when $\varepsilon\to 0$. Therefore, $\mathcal{H}(\varepsilon, B_{\mathcal{H}_{K}}, L_p(\mu_{{\mathbb S}^{n-1}}))\asymp \frac{1}{\varepsilon^{2(n-1)/\gamma}}$, for $1\leq p\leq 2$.

Recall that $\gamma>n-1$. Let us find an upper bound on the asymptotics of $\mathcal{H}(\varepsilon, B_{\mathcal{H}_{K}}, C({\mathbb S}^{n-1}))$ using Theorem~\ref{Bound-main}. First, the asymptotics of $N_{L_x}$ satisfies
\begin{equation*}
\begin{split}
N_{L_x} \asymp \int_0^{L_x} l^{n-2}dl \asymp L_x^{n-1} \asymp x^{-2(n-1)/\gamma},
\end{split}
\end{equation*}
due to $L_x \asymp x^{-2/\gamma}$. Also, we have
\begin{equation*}
\begin{split}
C_{K,p}(N_L)^2 \leq C_{K,\infty}(N_L)^2 \asymp \int_L^\infty \frac{1}{l^\gamma} l^{n-2}dl \asymp \frac{1}{L^{\gamma-n+1}},
\end{split}
\end{equation*}
which gives $C_{K,p}(N_{L_x})^2\ll x^{2(\gamma-n+1)/\gamma}$.

For a finite $p>2$, by Theorem~\ref{Bound-main} we have
\begin{equation*}
\begin{split}
\mathcal{H}(\varepsilon, B_{\mathcal{H}_{K}}, L_p(\mu_{{\mathbb S}^{n-1}}))\ll \min_{0<x<\varepsilon}N_{L_x}\log(\frac{1}{x})+\frac{C_{K,p}(N_{L_x})^2}{ (\varepsilon-x)^2}\asymp\\
\min_{0<x<\varepsilon}-x^{-2(n-1)/\gamma}\log x + \frac{x^{2(\gamma-n+1)/\gamma}}{ (\varepsilon-x)^2}.
\end{split}
\end{equation*}
After we set $x = \frac{\varepsilon}{2}$ we obtain
$\mathcal{H}(\varepsilon, B_{\mathcal{H}_{K}}, L_p(\mu_{{\mathbb S}^{n-1}}))\ll \frac{1}{\varepsilon^{2(n-1)/\gamma}}\log(\frac{1}{\varepsilon})$.
\end{proof}
\begin{remark} With a slightly more accurate analysis one can eliminare the logarithmic factor in the latter upper bound and obtain the general result $\mathcal{H}(\varepsilon, B_{\mathcal{H}_{K}},  L_p(\mu_{{\mathbb S}^{n-1}}))\asymp  \frac{1}{\varepsilon^{2(n-1)/\gamma}}$ for $p\geq 1$. Since we are only interested in demonstrating the efficiency of our general bound given in Theorem~\ref{Bound-main}, we omit this step.
\end{remark}

\section{Application II: the Gaussian kernel on a box}\label{gaussian-kernel}
Let us now study the behaviour of the upper bound in Theorem~\ref{Bound-main} for the kernel $K({\mathbf x},{\mathbf y}) = e^{-\sigma^2\|{\mathbf x}-{\mathbf y}\|^2}$ on the domain $\boldsymbol{\Omega} = [-1,1]^n$.
The main result of this section is the following theorem.
\begin{theorem}\label{gaussian-main} For $p\geq 1$, expressions from the right hand side of our bounds in Theorem~\ref{Bound-main} behave asymptotically like $\mathcal{O}(\frac{ (\log\frac{1}{\varepsilon})^{n+1}}{(\log \log \frac{1}{\varepsilon})^n})$ for $\varepsilon\to +0$.
\end{theorem}
It is well-known that $\mathcal{H}(\varepsilon, B_{\mathcal{H}_K}, C([-1,1]^n))\asymp \frac{ (\log\frac{1}{\varepsilon})^{n+1}}{(\log \log \frac{1}{\varepsilon})^n}$~\cite{KUHN2011489}, so our upper bounds are asymptotically tight. In fact, in Theorem~\ref{Gaussian-bound} below we obtain a non-asymptotical upper bound on $\mathcal{E}(\varepsilon, \{\sigma_i\}_{i=1}^\infty)$ and $m_\varepsilon$, which is much more informative that the asymptotics for $\varepsilon\to +0$.

The basic tool in our analysis is the following lemma.

\begin{lemma}\label{Eigen-Bound} Let $K(x,t) = e^{-\sigma^2(x-y)^2}$ be a Mercer kernel on the domain $\Omega = [-1,1]$ and $T: L_2([-1,1])\to L_2([-1,1])$, $T[\phi](x) = \int_{-1}^1 K(x,t)\phi(t)dt$. Let $|\lambda_1|\geq |\lambda_2|\geq \cdots$ be eigenvalues of $T$. Then, eigenvalues of $T$ satisfy
\begin{equation}\label{eigen-bound}
\begin{split}
\lambda_{k} < 8\cdot  (\frac{2}{e}(k-1))^{-\frac{k-1}{2}}  \sigma^{k-1},
\end{split}
\end{equation}
for $k\in  {\mathbb N}$.
\end{lemma}

The following theorem was proved in~\cite{Little}. Using this theorem one can deduce a bound on eigenvalues of the Gaussian kernel. We give its proof to follow the exact value of a constant which was neglected in the original formulation of the theorem.
\begin{theorem}[Little-Reade]\label{little} Let $K:[-1,1]^2\to {\mathbb R}$ be a continuous function such that $K(x,t)= K(t,x)$ and for any $t\in [-1,1]$, $K_t(x) = K(x,t)$ has an analytical continuation $\tilde{K}_t:E_R\to {\mathbb C}$ inside $E_R = \{z\in {\mathbb C} \mid z=\frac{1}{2}(w+w^{-1}), w\in {\mathbb C}, R^{-1}\leq |w|\leq R\}$ for $R>1$. Also, let $|\lambda_1|\geq |\lambda_2|\geq \cdots$ be eigenvalues of $T: L_2([-1,1])\to L_2([-1,1])$, $T[\phi](x) = \int_{-1}^1 K(x,t)\phi(t)dt$, ordered by their absolute values, counting multiplicities. Then,
$$
|\lambda_k|\leq 4 \sup_{t\in [-1,1], z\in E_R} |\tilde{K}_t(z)| \frac{R^{-k+2}}{R-1}.
$$
\end{theorem}
\begin{proof} Since $\tilde{K}_t$ is analytical inside $E_R$, $2\tilde{K}_t(\frac{1}{2}(w+w^{-1}))$ is analytical in the annulus $\{w\in {\mathbb C}\mid R^{-1}< |w| < R\}$. Therefore, its Laurent series is convergent to its value in that annulus, i.e.
$$
2\tilde{K}_t(\frac{1}{2}(w+w^{-1})) = \sum_{i=-\infty}^\infty a_i(t) w^i,
$$
where
$$
a_k(t) =\frac{1}{\pi {\rm i}}\int_{|w|=r} \tilde{K}_t(\frac{1}{2}(w+w^{-1}))w^{-k-1}dw,
$$
for $R^{-1}<r<R$. If we set $r=1$ and $w=e^{{\rm i}\phi}$, then
$$
a_k(t) =\frac{1}{\pi}\int_{-\pi}^{\pi} \tilde{K}_t(\cos \phi)e^{-{\rm i}k\phi}d\phi.
$$
Since $\tilde{K}_t(\cos \phi) = K_t(\cos \phi)$ is an even function, then $a_k(t)=a_{-k}(t)$. Thus,
\begin{equation*}
\begin{split}
\tilde{K}_t(\frac{1}{2}(w+w^{-1})) = \frac{1}{2}a_0(t)+\sum_{i=1}^\infty a_i(t)\frac{(w^i+w^{-i})}{2}.
\end{split}
\end{equation*}
Let $T_i$ be the $i$th Chebyshev's polynomial, i.e.  $T_i(\cos \phi) = \cos(i\phi)$ and $T_i(\cosh \phi) = \cosh(i\phi)$.
If we set $z=\frac{1}{2}(w+w^{-1})$, $w=\rho e^{{\rm i}\phi}$, then $\frac{(w^i+w^{-i})}{2} = \frac{(e^{i(\log\rho+{\rm i}\phi)}+e^{-i(\log\rho+{\rm i}\phi)})}{2} = \cosh (i(\log\rho+{\rm i}\phi)) = T_i(\cosh (\log\rho+{\rm i}\phi)) =  T_i(z)$. Thus,
\begin{equation*}
\begin{split}
\tilde{K}_t(z) = \frac{1}{2}a_0(t)+\sum_{i=1}^\infty a_i(t)T_i(z),
\end{split}
\end{equation*}
for $z\in E_R$.

Coefficients $a_k(t)$ can be bounded:
\begin{equation*}
\begin{split}
&|a_k(t)| =|\frac{1}{\pi {\rm i}}\int_{|w|=R-\varepsilon} \tilde{K}_t(\frac{1}{2}(w+w^{-1}))w^{-k-1}dw| \leq \\
& \max_{z\in E_R} |\tilde{K}_t(z)| \frac{1}{\pi }2\pi(R-\varepsilon)^{-k-1}R.
\end{split}
\end{equation*}
As $\varepsilon\to 0+$, we obtain
$$
|a_k(t)| \leq
2\sup_{z\in E_R} |\tilde{K}_t(z)|R^{-k}.
$$
Let us now denote by $S_k:L_2([-1,1])\to L_2([-1,1])$ an integral operator $S_k[\phi](t)=\int_{-1}^1 \big(\frac{1}{2}a_0(t)+\sum_{i=1}^k a_i(t)T_i(z)\big)\phi(z)dz$. By construction, ${\rm rank\,}S_k \leq k+1$.
Using $|T_i(z)|\leq 1, z\in [-1,1]$ and the Eckart-Young-Mirsky theorem for the low-rank approximation of compact operators in the operator norm (it follows from Courant-Fischer-Weyl min-max principle in the same way as in a finite dimensional case), we obtain:
\begin{equation*}
\begin{split}
&|\lambda_{k+2}| \leq 
\sup_{\phi\in L_2([-1,1]), \|\phi\|_{L_2}=1}\|(T-S_k)[\phi]\|_{L_2([-1,1])} \leq \\
& \sqrt{2}\sup_{t\in [-1,1]}\sup_{\phi\in L_2([-1,1]), \|\phi\|_{L_2}=1}
\sum_{i=k+1}^\infty \int_{-1}^1|a_i(t)|\cdot |T_i(z)|\cdot |\phi(z)| dz\leq \\
&\sqrt{2}\sup_{t\in [-1,1]}\sup_{\phi\in L_2([-1,1]), \|\phi\|_{L_2}=1}
\sum_{i=k+1}^\infty \int_{-1}^1 2\sup_{t\in [-1,1], z\in E_R} |\tilde{K}_t(z)|R^{-i}\cdot |\phi(z)| dz \leq \\
&4 \sup_{t\in [-1,1], z\in E_R} |\tilde{K}_t(z)| \sum_{i=k+1}^\infty R^{-i} = 
4 \sup_{t\in [-1,1], z\in E_R} |\tilde{K}_t(z)| \frac{R^{-(k+1)}}{1-R^{-1}} = \\
&4 \sup_{t\in [-1,1], z\in E_R} |\tilde{K}_t(z)| \frac{R^{-k}}{R-1}.
\end{split}
\end{equation*}
\end{proof}

\begin{proof}[Proof of Lemma~\ref{Eigen-Bound}] We can define a continuation function as $\tilde{K}_t(z) = e^{-\sigma^2(z-t)^2}$ and it is obviously analytical for any $t$. The set $B_r = \{\frac{1}{2}(w+w^{-1}) \mid |w|=r\}$  can be given as
\begin{equation*}
\begin{split}
\{\frac{1}{2}((r+r^{-1})\cos\phi, (r-r^{-1})\sin\phi) \mid \phi \in [0,2\pi)\},
\end{split}
\end{equation*}
i.e. it is an ellipse with foci $\{+1, -1\}$ (called the Bernstein ellipse). Thus,
\begin{equation*}
\begin{split}
&\sup_{t\in [-1,1], z\in E_R} |\tilde{K}_t(z)| =
\max\{ e^{-{\rm Re\,}\sigma^2(z-t)^2} \mid
z\in E_R, t\in [-1,1]\} = \\
&{\rm exp}\big(-\sigma^2 \min_{ R^{-1}\leq r\leq R}\min_{t\in [-1,1], z\in B_r} {\rm Re\,}(z-t)^2\big).
\end{split}
\end{equation*}
For a fixed $t\in [-1,1]$, the expression
\begin{equation*}
\begin{split}
\min_{z\in B_r} {\rm Re\,}(z-t)^2= \min_{ x+{\rm i}y\in B_r} (x-t)^2-y^2
\end{split}
\end{equation*}
can be calculated using the method of Lagrange multipliers:
\begin{equation*}
\begin{split}
&L(x,y,\lambda) = (x-t)^2-y^2+
\lambda(\frac{4x^2}{(r+r^{-1})^2}+\frac{4y^2}{(r-r^{-1})^2}-1),\\
&\partial_x L = 2(x-t)+\frac{8x\lambda}{(r+r^{-1})^2}=0, 
\partial_y L = -2y+\frac{8y\lambda}{(r-r^{-1})^2}=0.
\end{split}
\end{equation*}
Thus, either 1) $y=0$, or 2) $\lambda = \frac{(r-r^{-1})^2}{4}$. In the first case  the minimum is attained on $x$-axis and $\arg\min_{z\in B_r} {\rm Re\,}(z-t)^2 = +\frac{1}{2}(r+r^{-1})$ if $t>0$ and $\arg\min_{z\in B_r} {\rm Re\,}(z-t)^2 = -\frac{1}{2}(r+r^{-1})$ if $t<0$. In other words, we have
$$
\min_{z\in B_r} {\rm Re\,}(z-t)^2 =( \frac{1}{2}(r+r^{-1})-|t|)^2.
$$
In the second case, $x = \frac{t}{1+\frac{(r-r^{-1})^2}{(r+r^{-1})^2}} = \frac{(r+r^{-1})^2t}{2(r^2+r^{-2})}$ and $y=\pm \frac{(r-r^{-1})}{2} \sqrt{1-(\frac{(r+r^{-1})t}{(r^2+r^{-2})})^2}$ and we have:
\begin{equation*}
\begin{split}
&\min_{ x+{\rm i}y\in B_r} (x-t)^2-y^2 =
(\frac{(r-r^{-1})^2t}{2(r^2+r^{-2})})^2-\frac{(r-r^{-1})^2}{4} (1-(\frac{(r+r^{-1})t}{(r^2+r^{-2})})^2) = \\
&\frac{t^2(r-r^{-1})^2}{2(r^2+r^{-2})}-\frac{(r-r^{-1})^2}{4}.
\end{split}
\end{equation*}
Therefore,
\begin{equation*}
\begin{split}
&\min_{t\in [-1,1], z\in B_r} {\rm Re\,}(z-t)^2=\\
&\min_{t\in [-1,1]} \min\{( \frac{1}{2}(r+r^{-1})-|t|)^2, \frac{t^2(r-r^{-1})^2}{2(r^2+r^{-2})}-\frac{(r-r^{-1})^2}{4} \} =
-\frac{(r-r^{-1})^2}{4}.
\end{split}
\end{equation*}
Thus,
\begin{equation*}
\begin{split}
\sup_{t\in [-1,1], z\in E_R} |\tilde{K}_t(z)| =
{\rm exp}\big( \min_{ R^{-1}\leq r\leq R}\frac{\sigma^2(r-r^{-1})^2}{4}\big) = 
{\rm exp}\big( \frac{\sigma^2(R-R^{-1})^2}{4}\big).
\end{split}
\end{equation*}
Using Theorem~\ref{little} we conclude
\begin{equation*}
\begin{split}
\lambda_k\leq 4 {\rm exp}\big( \frac{\sigma^2(R-R^{-1})^2}{4}\big)  \frac{R^{-k+2}}{R-1},
\end{split}
\end{equation*}
for any $R>1$. Let us set $R$ in such a way that $\sigma(R-R^{-1})=\sqrt{2(k-1)}$, i.e. $R=\frac{1}{2}(\sqrt{\frac{2(k-1)}{\sigma^2}+4}+\frac{\sqrt{2(k-1)}}{\sigma})$. Then, $R-1\geq \frac{\sqrt{2(k-1)}}{2\sigma}$ and $R\geq \frac{\sqrt{2(k-1)}}{\sigma}$. Thus, for $k\geq 2$, we have
\begin{equation*}
\begin{split}
\lambda_k < 4 {\rm exp}\big( \frac{k-1}{2}\big)  (\frac{\sqrt{2(k-1)}}{\sigma})^{-k+2} \frac{2\sigma}{\sqrt{2(k-1)}} = 
8\cdot  (\frac{2}{e}(k-1))^{-\frac{k-1}{2}}  \sigma^{k-1}.
\end{split}
\end{equation*}
Note that $0^0=1$ and the latter inequality is also correct for $k=1$, due to
\begin{equation*}
\begin{split}
\lambda_1 \leq \int_{-1}^1 K(x,x)dx = 2 < 8.
\end{split}
\end{equation*}
\end{proof}

\ifCOM
If $R>\max\{2,\frac{\sigma}{2}\}$, we finally obtain:
\begin{equation*}
\begin{split}
\lambda_k\leq 8e\cdot  {\rm exp}\big( \frac{\sigma^2(R^2-2)}{4}\big)  R^{-k+1}
\end{split}
\end{equation*}
The RHS attains its minimum when $\frac{\sigma^2R}{2}+(-k+1)\frac{1}{R}=0$, i.e. $R=\frac{\sqrt{2(k-1)}}{\sigma}$. If $k\geq \frac{\sigma^4}{8}+1$ and $k\geq  2\sigma^2+1$, then $\frac{\sqrt{2(k-1)}}{\sigma}\geq \max\{2,\frac{\sigma}{2}\}$.
Therefore, we conclude:
$$
\lambda_k\leq 8e\cdot  (\frac{2}{e}(k-1))^{-\frac{k-1}{2}} \sigma^{k-1}e^{-\sigma^2/2}
$$
\else
\fi

Let $\Sigma_1$ be a set of eigenvalues of $G_1(x,t) = e^{-\sigma^2(x-y)^2}$ on the domain $\Omega = [-1,1]$.
Since $K({\mathbf x},{\mathbf y}) = e^{-\sigma^2\|{\mathbf x}-{\mathbf y}\|^2} = \prod_{i=1}^n e^{-\sigma^2(x_i-y_i)^2}$, the set of eigenvalues of $T: L_2(\boldsymbol{\Omega})\to L_2(\boldsymbol{\Omega})$, $T[\phi]({\mathbf x})= \int_{\boldsymbol{\Omega}} K({\mathbf x},{\mathbf y}) \phi({\mathbf y})d{\mathbf y}$, $\boldsymbol{\Omega} = [-1,1]^n$ can be factorized as
$$
\Sigma_n =\{\lambda_{i_1}\cdots \lambda_{i_n} \mid \lambda_{i_j}\in \Sigma_1\}.
$$
Next, let us estimate the value of the function $\mathcal{E}(\varepsilon, \{\sigma_i\}_{i=1}^\infty)$, which participates in all our lower and upper bounds on $\mathcal{H}(\varepsilon, B_{\mathcal{H}_K}, L_p([-1,1]^n))$, i.e. the value
\begin{equation*}
\begin{split}
\mathcal{E}(\varepsilon, \{\sigma_i\}_{i=1}^\infty) = \sum_{\lambda_{i_1}\cdots \lambda_{i_n} > \varepsilon, \lambda_{i_j}\in \Sigma_1}\log \frac{\lambda_{i_1}\cdots \lambda_{i_n}}{\varepsilon}.
\end{split}
\end{equation*}
From the inequality \eqref{eigen-bound} we conclude
\begin{equation*}
\begin{split}
\log \frac{\lambda_{i_1}\cdots \lambda_{i_n}}{\varepsilon} \leq
\sum_{j=1}^n \big(  -\frac{i_j-1}{2}\log (\frac{2}{e}(i_j-1)) +
(i_j-1)\log (\sigma) \big) + \log(\frac{1}{\varepsilon})+n\log(8),
\end{split}
\end{equation*}
and the summation is made over $\{(i_1,\cdots , i_n) \in {\mathbb N}^n\mid \lambda_{i_1}\cdots \lambda_{i_n} > \varepsilon, \lambda_{i_j}\in \Sigma_1\} \subseteq \{(i_1+1,\cdots , i_n+1) \mid (i_1,\cdots , i_n)\in I\}$, where
\begin{equation*}
\begin{split}
&I = \{(i_1,\cdots , i_n)\in  ({\mathbb N}\cup\{0\})^^n \mid
\prod_{j=1}^n 8\cdot  (\frac{2}{e}i_j)^{-\frac{i_j}{2}} \sigma^{i_j}> \varepsilon\} = \\
&\{(i_1,\cdots , i_n)\in ({\mathbb N}\cup\{0\})^n \mid 
\sum_{j=1}^n   \frac{i_j}{2}\log (\frac{2}{e}i_j) - i_j\log (\sigma)  <
\log(\frac{1}{\varepsilon})+n\log(8)\}.
\end{split}
\end{equation*}
Note that
\begin{equation*}
\begin{split}
m_\varepsilon = |I|.
\end{split}
\end{equation*}

Let us denote $u(x) =  \frac{x}{2}\log (\frac{2}{e}x) - x\log (\sigma)$, $r = \lfloor \frac{\sigma^2}{2}\rfloor$ and $\Delta(\sigma) =\min\{u(r), u(r+1)\}$.
\begin{lemma}\label{functionC} The function $f({\mathbf x}) = \sum_{j=1}^n  u(x_j)$ satisfies
$$
f(x_1, \cdots, x_n) \geq n \Delta(\sigma),
$$
if $x_i\in {\mathbb N}\cup \{0\}$, $i\in [n]$.
\end{lemma}
\begin{proof} Note that
$$
u' = (\frac{x}{2}\log (\frac{2}{e}x) - x\log (\sigma))' = \frac{1}{2}\log (x)+ \frac{1}{2}\log 2-\log (\sigma) > 0,
$$
if $x>\frac{\sigma^2}{2}$ and the derivative is negative if $x<\frac{\sigma^2}{2}$. Thus, the minimum of $u(x_i)$ is $-\frac{\sigma^2}{4}$,
if $x_i>0$ is real, and is $\Delta(\sigma)$, if $x_i\in {\mathbb N}\cup \{0\}$.
Thus,
$$
u(x_i)\geq \Delta(\sigma),
$$
and the summation over $i\in [n]$ gives the needed inequality.
\end{proof}
Using Lemma~\ref{functionC}, we conclude
\begin{equation}\label{thirt}
\begin{split}
&\sum_{\lambda_{i_1}\cdots \lambda_{i_n} > \varepsilon, \lambda_{i_j}\in \Sigma_1}\log \frac{\lambda_{i_1}\cdots \lambda_{i_n}}{\varepsilon}\leq \\
&|I| \max_{{\mathbf x}\in ({\mathbb N}\cup \{0\})^n} 
\{ \sum_{j=1}^n -\frac{x_j}{2}\log (\frac{2}{e}x_j) + x_j\log (\sigma)+\log(\frac{1}{\varepsilon})+n\log(8)\}\leq \\
&(-n \Delta(\sigma)+ \log(\frac{1}{\varepsilon})+n\log(8))|I|.
\end{split}
\end{equation}
Thus, the problem is reduced to bounding $|I|$, i.e. the number of integer points in a convex body $X$, where $X=\{{\mathbf x}\in [0,+\infty)^n \mid \sum_{j=1}^n   u(x_j) <\log(\frac{1}{\varepsilon})+n\log(8)\}$. The cardinality of $I$ cannot be simply bounded by ${\rm Vol\,}(X)$, due to the fact that $u(x)$ is not monotonically increasing on the whole $ [0,+\infty)$.
\begin{lemma}\label{sum-bound} Let $\sigma>0, D=\log(\frac{1}{\varepsilon})+n\log(8)$ and
$$I=\{{\mathbf x}\in ({\mathbb N}\cup \{0\})^n \mid \sum_{j=1}^n   u(x_j) <D\}.$$
Then,
$$|I|\leq \sum_{q=0}^n 2^{n-q} {n \choose n-q}{\rm Vol}(X_q),$$
where $X_q = \{{\mathbf x}\in [0,+\infty)^{q} \mid \sum_{j=1}^q  u(x_j) <D-(n-q)\Delta(\sigma)\}$.
\end{lemma}
\begin{proof} Recall that the function $ u(x) = \frac{x}{2}\log (\frac{2}{e}x) - x\log (\sigma)$ is a decreasing function on $[0, \frac{\sigma^2}{2}]$ and is an increasing function on $[ \frac{\sigma^2}{2}, \infty)$.
Let $(i_1, \cdots, i_n)\in I$ be an integer point. Let us define $(\phi(i_1), \cdots, \phi(i_n))$ in the following way: if $i_j > \frac{\sigma^2}{2}$, then $\phi(i_j)=i_j-1$ and, if $i_j \leq\frac{\sigma^2}{2}$, then $\phi(i_j)=i_j+1$. The function $u(x)$  on $[i_j, \phi(i_j)]$ satisfies
\begin{equation*}
\begin{split}
u(x)\leq u(i_j),
\end{split}
\end{equation*}
if $i_j\ne r$ and $i_j\ne r+1$.

Let us define $R(i_1, \cdots, i_n) = \{j \mid i_j =r\}\cup \{j \mid i_j =r+1\}$ and $Q(i_1, \cdots, i_n) = [n]\setminus R(i_1, \cdots, i_n)$.
Then, the set $B(i_1, \cdots, i_n) =\prod_{j\in Q(i_1, \cdots, i_n)} [i_j, \phi(i_j)]$ is a unit cube in ${\mathbb R}^{q }$ where $q=|Q(i_1, \cdots, i_n)|$.
Since $u(x)\leq u(i_j), x\in [i_j, \phi(i_j)], j\in Q(i_1, \cdots, i_n) $, the function $f_q (x_1, \cdots, x_q)=\sum_{j=1}^q  u(x_j)$ on that cube can be bounded in the following way
\begin{equation*}
\begin{split}
f_q(\langle x_{j}\rangle_{j\in Q_{(i_1, \cdots, i_n)}}) \leq f_q (\langle i_{j}\rangle_{j\in Q_{(i_1, \cdots, i_n)}})\leq 
D-(n-q)\Delta(\sigma).
\end{split}
\end{equation*}

Let us denote $X_q = \{{\mathbf x}\in [0,+\infty)^{q} \mid \sum_{j=1}^q  u(x_j) <D-(n-q)\Delta(\sigma)\}$.
Thus, to any integer point $(i_1, \cdots, i_n)\in I$ we associated a cube $B(i_1, \cdots, i_n)\subseteq X_{q}$. This mapping is not surjective. By construction, the cube $B=\prod_{j=1}^q [a_j, a_j+1]\subseteq X_q$ can be an image of some $(i_1, \cdots, i_n)$ only if $|\{ j \mid i_j = r\vee i_j = r+1\}|=n-q$ and $i_j$ is defined uniquely by $B$ if $i_j\ne r$ and $i_j\ne r+1$.
Therefore, $B$ can be an image of no more than $2^{n-q}{n \choose n-q}$ integer points. Therefore, the total number of preimages of all cubes $\{B_{(i_1, \cdots, i_n)}\}_{(i_1, \cdots, i_n)\in I}$ does not exceed
\begin{equation*}
\begin{split}
\sum_{q=0}^n 2^{n-q} {n \choose n-q}{\rm Vol}(X_q).
\end{split}
\end{equation*}
\end{proof}
To bound ${\rm Vol\,}(X_q)$ we need the following lemma.
\begin{lemma}\label{kee} Let $R_{b}({\mathbf x}) = \sum_{i=1}^n x_i\log x_i +b\sum_{i=1}^n x_i$. Then,
$$
{\rm Vol}(\{{\mathbf x}\in [0,+\infty)^n \mid R_{b}({\mathbf x})<c\}) \leq \frac{1}{n!}\frac{2^n c^n}{(b+\log c -\log n)^n},
$$
if $b+\log c -\log n>1$.
\end{lemma}
\begin{proof} Let $S({\mathbf x}) = \sum_{j=1}^n x_j$. Then,
\begin{equation*}
\begin{split}
&R_{b}({\mathbf x}) = \sum_{j=1}^n x_j \cdot \big(\sum_{i=1}^n \frac{x_i}{\sum_{j=1}^n x_j}\log \frac{x_i}{\sum_{j=1}^n x_j}+\log(\sum_{j=1}^n x_j)\big) +b\sum_{i=1}^n x_i = \\
&S({\mathbf x}) \sum_{i=1}^n p_i\log p_i +S({\mathbf x})\log S({\mathbf x})+b S({\mathbf x}) =
-S({\mathbf x}) E( {\mathbf p}) +S({\mathbf x})\log S({\mathbf x})+b S({\mathbf x}),
\end{split}
\end{equation*}
where $p_i({\mathbf x}) = \frac{x_i}{\sum_{j=1}^n x_j}$ and $E( p_1,\cdots, p_n) = -\sum_{i}p_i\log p_i$. Let us make the change of variables in the integral from $$x_1, \cdots, x_n$$ to
$$
S({\mathbf x}), p_1({\mathbf x}), \cdots, p_{n-1}({\mathbf x}).
$$
Since $x_i = S p_i$, $i\in [n-1]$, and $x_n = S(1-\sum_{i=1}^{n-1}p_i)$, the Jacobian of the transformation is
\begin{equation*}
\begin{split}
&\big|\frac{\partial (x_1, \cdots, x_n)}{\partial (S, p_1, \cdots, p_{n-1})}\big| =\\
&|{\rm det\,}\begin{bmatrix} p_1 & S & 0 & \cdots & 0 & 0\\ p_2 & 0 & S & \cdots & 0 & 0 \\p_3 & 0 & 0 & \cdots & 0 & 0\\ \cdots \\ p_{n-1} & 0 & 0 & \cdots & 0 & S \\ (1-\sum_{i=1}^{n-1}p_i) & -S & -S & -S  & -S & -S  \end{bmatrix} | = \\
&|{\rm det\,}\begin{bmatrix} p_1 & S & 0 & \cdots & 0 & 0\\ p_2 & 0 & S & \cdots & 0 & 0 \\p_3 & 0 & 0 & \cdots & 0 & 0\\ \cdots \\ p_{n-1} & 0 & 0 & \cdots & 0 & S \\ 1 & 0 & 0 & 0  & 0 & 0  \end{bmatrix}| = S^{n-1}.
\end{split}
\end{equation*}
From
\begin{equation*}
\begin{split}
&\{{\mathbf x}\in [0,+\infty)^n \mid R_{b}({\mathbf x})<c, S({\mathbf x})=s \} \subseteq \\
&\{{\mathbf x}\in [0,+\infty)^n \mid p_i({\mathbf x}) \geq 0, \sum_{i=1}^n p_i =1,
E( {\mathbf p}) > \frac{s\log s+bs-c}{s}\},
\end{split}
\end{equation*}
we conclude
\begin{equation*}
\begin{split}
&{\rm Vol}(\{{\mathbf x}\in [0,+\infty)^n \mid R_{b}({\mathbf x})<c\}) \leq \\
&\int_{0}^{s^\ast} {\rm Vol}(\{(p_1,\cdots, p_{n-1})\mid p_i \geq 0, i\in [n-1], \sum_{i=1}^{n-1} p_i \leq 1,\\
& E( \{p_i\}_1^n) > \frac{s\log s+bs-c}{s}\}) s^{n-1} ds \leq
\frac{1}{(n-1)!}\int_{0}^{s^\ast} s^{n-1} ds = \frac{(s^\ast)^n}{n!},
\end{split}
\end{equation*}
where $\log n = \frac{s^\ast\log s^\ast+bs^\ast-c}{s^\ast} = \log s^\ast+b-\frac{c}{s^\ast}$, or $\frac{s^\ast}{ne^{-b}}\log \frac{s^\ast}{ne^{-b}} = \frac{c}{ne^{-b}} $.

Let us find the growth rate of $s^\ast$. With this purpose let us introduce the function
$$
g(x) = x\log x.
$$
This function is increasing for $x>e^{-1}$ due to $f' = 1+\log x>0$. The inverse of $f|_{(e^{-1}, \infty)}$ satisfies
$$
f^{-1}(x) \leq \frac{2x}{\log x},
$$
for $x>e$ due to $f(\frac{2x}{\log x}) =  \frac{2x}{\log x}(\log 2+\log x - \log \log x)  = x(2+\frac{2\log 2-2\log \log x}{\log x}) > x$.
Therefore,
$$
\frac{s^\ast}{ne^{-b}} = f^{-1}(\frac{c}{ne^{-b}})\leq \frac{2c}{ne^{-b}(b+\log c -\log n)},
$$
i.e. $s^\ast \leq \frac{2c}{b+\log c -\log n}$ if $\frac{c}{ne^{-b}}>e$. Thus, the volume of $\{{\mathbf x}\in (0,+\infty)^n \mid R_{b}({\mathbf x})<c\}$ is bounded by
$$
\frac{1}{n!}\frac{2^n c^n}{(b+\log c -\log n)^n},
$$
if $b+\log c -\log n>1$.
\end{proof}
\begin{theorem}\label{Gaussian-bound} For the Gaussian kernel $K({\mathbf x},{\mathbf y}) = e^{-\sigma^2\|{\mathbf x}-{\mathbf y}\|^2}$ on the domain $\boldsymbol{\Omega} = [-1,1]^n$ we have
\begin{equation*}
\begin{split}
&\mathcal{E}(\varepsilon, \{\sigma_i\}_{i=1}^\infty) \leq n2^n  (-n \Delta(\sigma)+ \log(\frac{1}{\varepsilon})+n\log(8)) \\
&\max_{q\in [n]\cup \{0\}}\frac{{n \choose q}  2^q (D-(n-q)\Delta(\sigma))^q}{q! (-\log (\sigma^2 e/4)+\log (D-(n-q)\Delta(\sigma)) -\log q)^q},
\end{split}
\end{equation*}
where $D=\log(\frac{1}{\varepsilon})+n\log(8)$ and under the condition that
\begin{equation*}
\begin{split}
D>(n-q)\Delta(\sigma)+\frac{e^2q \sigma^2}{4}, \forall q\in [n]\cup \{0\}.
\end{split}
\end{equation*}
Also, under the same condition, we have
\begin{equation*}
\begin{split}
m_{\varepsilon }\leq n2^n
\max_{q\in [n]\cup \{0\}}\frac{{n \choose q}  2^q (D-(n-q)\Delta(\sigma))^q}{q! (-\log (\sigma^2 e/4)+\log (D-(n-q)\Delta(\sigma)) -\log q)^q}.
\end{split}
\end{equation*}
\end{theorem}
\begin{proof}
Using Lemma~\ref{sum-bound} and Lemma~\ref{kee} we bound ${\rm Vol}(X_q)$ and obtain:
\begin{equation*}
\begin{split}
&|I| \leq  \sum_{q=0}^n 2^{n-q} {n \choose n-q}\frac{1}{q!}\cdot 
\frac{2^q (2D-2(n-q)\Delta(\sigma))^q}{(-\log (\sigma^2 e/2)+\log (2D-2(n-q)\Delta(\sigma)) -\log q)^q} \leq \\
&n2^n \cdot
\max_{q\in [n]\cup \{0\}}\frac{{n \choose q}  2^q (D-(n-q)\Delta(\sigma))^q}{q! (-\log (\sigma^2 e/4)+\log (D-(n-q)\Delta(\sigma)) -\log q)^q},
\end{split}
\end{equation*}
under the condition that $$-\log (\sigma^2 e/4)+\log (D-(n-q)\Delta(\sigma)) -\log q>1, \forall q\in [n]\cup \{0\}.$$
Thus, using the inequality~\eqref{thirt} and the latter bound on $|I|$, we conclude
\begin{equation*}
\begin{split}
&\mathcal{E}(\varepsilon, \{\sigma_i\}_{i=1}^\infty) \leq  n2^n  (-n \Delta(\sigma)+ \log(\frac{1}{\varepsilon})+n\log(8)) \cdot \\
&\max_{q\in [n]\cup \{0\}}\frac{{n \choose q}  2^q (D-(n-q)\Delta(\sigma))^q}{q! (-\log (\sigma^2 e/4)+\log (D-(n-q)\Delta(\sigma)) -\log q)^q}.
\end{split}
\end{equation*}
Also,
\begin{equation*}
\begin{split}
m_{\varepsilon }\leq |I|\leq n2^n
\max_{q\in [n]\cup \{0\}}\frac{{n \choose q}  2^q (D-(n-q)\Delta(\sigma))^q}{q! (-\log (\sigma^2 e/4)+\log (D-(n-q)\Delta(\sigma)) -\log q)^q}.
\end{split}
\end{equation*}
\end{proof}

\begin{corollary} If $\varepsilon\to 0$, then
\begin{equation*}
\begin{split}
&\mathcal{E}(\varepsilon, \{\sigma_i\}_{i=1}^\infty) \ll \frac{ (\log\frac{1}{\varepsilon})^{n+1}}{(\log \log \frac{1}{\varepsilon})^n},\\
&m_{\varepsilon }\ll \frac{ (\log\frac{1}{\varepsilon})^{n}}{(\log \log\frac{1}{\varepsilon})^n}.
\end{split}
\end{equation*}
\end{corollary}
\begin{proof}
For $\varepsilon\to 0$, the maximum over $q\in [n]\cup \{0\}$ in Theorem~\ref{Gaussian-bound} is attained for $q=n$ and we obtain
\begin{equation*}
\begin{split}
&\mathcal{E}(\varepsilon, \{\sigma_i\}_{i=1}^\infty) \ll \frac{ (\log\frac{1}{\varepsilon})^{n+1}}{(\log \log \frac{1}{\varepsilon})^n}.
\end{split}
\end{equation*}
Analogously, the bound on $m_{\varepsilon }$ can be obtained.
\end{proof}

\begin{proof}[Proof of Theorem~\ref{gaussian-main}] Using the previous corollary, we have
\begin{equation*}
\begin{split}
& \mathcal{E}(\nu(\boldsymbol{\Omega})^{1/2-1/p}\varepsilon, \{\sigma_i\}_{i=1}^\infty)+m_{(1-\theta)\nu(\boldsymbol{\Omega})^{1/2-1/p}\varepsilon}\log\frac{3}{\theta(2-\theta)} \ll \frac{ (\log\frac{1}{\varepsilon})^{n+1}}{(\log \log \frac{1}{\varepsilon})^n}.
\end{split}
\end{equation*}
Thus, the RHS expression for $p\in [1,2]$ indeed behaves like $\mathcal{O}(\frac{ (\log\frac{1}{\varepsilon})^{n+1}}{(\log \log \frac{1}{\varepsilon})^n})$.

It remains to check that the asymptotics $\frac{ (\log\frac{1}{\varepsilon})^{n+1}}{(\log \log \frac{1}{\varepsilon})^n}$ holds for the case of $p=+\infty$. After we set $\delta=\frac{\varepsilon}{2}$ and $N=m_{\varepsilon/2}$ in the RHS expression
\begin{equation*}
\begin{split}
\min\limits_{\substack{\delta\in (0,\varepsilon),\\ N\in {\mathbb N}\cup \{0\}}} \big\{ N\log(\frac{C_K}{\varepsilon-\delta})+c\log(N+1)+\tilde{c}_1\big(\frac{\int_0^{C_{K,+\infty}(N)} \sqrt{\mathcal{H}(t,\tilde{\boldsymbol{\Omega}}, \mathcal{H}_{K})}dt}{\delta}\big)^2\big\},
\end{split}
\end{equation*}
we obtain $N\log(\frac{C_K}{\varepsilon-\delta})+c\log(N+1)\ll m_{\varepsilon/2}\log(\frac{1}{\varepsilon})\ll \frac{ (\log\frac{1}{\varepsilon})^{n+1}}{(\log \log \frac{1}{\varepsilon})^n}$.  Since the Gaussian kernel satisfies the property~\eqref{property-k} from Remark~\ref{non-exotic}, it is clear that
\begin{equation*}
\begin{split}
\big(\frac{\int_0^{C_{K,+\infty}(m_{\varepsilon/2})} \sqrt{\mathcal{H}(t,\tilde{\boldsymbol{\Omega}}, \mathcal{H}_{K})}dt}{\delta}\big)^2 \ll \frac{C_{K,+\infty}(m_{\varepsilon/2})^2}{\varepsilon^2}\log(\frac{1}{C_{K,+\infty}(m_{\varepsilon/2})}).
\end{split}
\end{equation*}
Since the Gaussian kernel is infinitely differentiable, then from Theorem 1.1 of~\cite{ARXIV2205} we have $C_{K,+\infty}(N) = \mathcal{O}((\sum_{i=N+1}^\infty \sigma_i^{2})^{1/2-\eta})$ for any $\eta\in (0,\frac{1}{2})$. Since singular values $\{\sigma_i\}_{i=1}^\infty$ of the Gaussian kernel on $[-1,1]^n$ decay faster than $Ce^{-\gamma i^r}$ for some $C, \gamma, r>0$, we have $C_{K,+\infty}(m_{\varepsilon/2}) = \mathcal{O}(\varepsilon)$ and, therefore, the integral term of on the right hand side of the bound contributes no more than $\mathcal{O}(\log(\frac{1}{\varepsilon}))$. Thus,
\begin{equation*}
\begin{split}
& \min\limits_{\substack{\delta\in (0,\varepsilon),\\ N\in {\mathbb N}\cup \{0\}}} \big\{ N\log(\frac{C_K}{\varepsilon-\delta})+c\log(N+1)+\tilde{c}_1\big(\frac{\int_0^{C_{K,+\infty}(N)} \sqrt{\mathcal{H}(t,\tilde{\boldsymbol{\Omega}}, \mathcal{H}_{K})}dt}{\delta}\big)^2\big\} \ll \\
& \frac{ (\log\frac{1}{\varepsilon})^{n+1}}{(\log \log \frac{1}{\varepsilon})^n}.
\end{split}
\end{equation*}
\end{proof}



\end{appendices}


\newcommand{\noopsort}[1]{}

\end{document}